\documentclass[11pt]{article}

\usepackage{amsmath,amsfonts,amsthm,amssymb,bbm}
\usepackage{algorithm}	    % for algorithms
\usepackage{algorithmic}	    % for algorithms
\usepackage{graphicx,subcaption}
\usepackage{url}
\usepackage{textcomp}
\usepackage{soul}
\usepackage{upgreek}

\usepackage[utf8]{inputenc} % allow utf-8 input
\usepackage[T1]{fontenc}    % use 8-bit T1 fonts
\usepackage{booktabs}       % professional-quality tables
\usepackage{multirow}	    % import for multirow tables
\usepackage{nicefrac}       % compact symbols for 1/2, etc.
\usepackage{todonotes}      % for making better todos
\usepackage{pgfplots}
\pgfplotsset{compat=1.17}
\usepackage{amsmath}
\usepackage{wrapfig}
\usetikzlibrary{positioning}
\usetikzlibrary{arrows}
\usetikzlibrary{arrows.meta}
\usetikzlibrary{hobby}

\usepackage{longtable}

\usepackage{minitoc}

\usepackage{float}

\newtheorem*{rep@theorem}{\rep@title}
\newcommand{\newreptheorem}[2]{%
\newenvironment{rep#1}[1]{%
 \def\rep@title{#2 \ref{##1}}%
 \begin{rep@theorem}}%
 {\end{rep@theorem}}}
\makeatother

\newtheorem*{rep@lemma}{\rep@title}
\newcommand{\newreplemma}[2]{%
\newenvironment{rep#1}[1]{%
 \def\rep@title{#2 \ref{##1}}%
 \begin{rep@lemma}}%
 {\end{rep@lemma}}}
\makeatother

\newtheorem*{rep@corollary}{\rep@title}
\newcommand{\newrepcorollary}[2]{%
\newenvironment{rep#1}[1]{%
 \def\rep@title{#2 \ref{##1}}%
 \begin{rep@corollary}}%
 {\end{rep@corollary}}}
\makeatother

\theoremstyle{plain}
\newtheorem{theorem}{{\bf Theorem}}[section]
\newreptheorem{theorem}{{\bf Theorem}}
\newtheorem{lemma}[theorem]{{\bf Lemma}}
\newreptheorem{lemma}{{\bf Lemma}}

\newreptheorem{proposition}{{\bf Proposition}}

\newreptheorem{corollary}{{\bf Corollary}}
\newtheorem{definition}{{\bf Definition}}[section]

\newtheorem{remark}[theorem]{{\bf Remark}}

\newcommand\citenumber[2]{[#1, \citep{#2}]}

\newcommand{\dif}[1]{\mathrm{d} #1}
\newcommand{\der}[2]{\frac{\dif{#1}}{\dif{#2}}}
\newcommand{\doh}[2]{\frac{\partial #1}{\partial #2}}
\newcommand{\graD}[1]{\nabla #1}
\newcommand{\divR}[1]{\mathrm{div} \left(#1\right)}
\newcommand{\lapL}[1]{\Delta #1}
\newcommand{\hesS}[1]{\nabla^2 #1}

\newcommand{\tra}[1]{\textsf{Tr}\left(#1\right)}
\newcommand{\dotP}[2]{\left\langle#1,#2\right\rangle}

\newcommand{\norm}[1]{\left\Vert #1 \right\Vert_2}
\newcommand{\ve}{\ensuremath{\mathbf v}}
\newcommand{\V}{V}

\newcommand{\prob}[2]{\underset{#1}{\mathbb{P}}\left[#2\right]}

\newcommand{\lap}[1]{\text{Lap}\left(#1\right)}
\newcommand{\expec}[2]{\underset{#1}{\mathbb{E}}\left[#2\right]}
\newcommand{\Gaus}[2]{\mathcal{N}\left(#1, #2\right)}
\newcommand{\Id}{\mathbb{I}_\dime}
\newcommand{\noise}{\sigma}
\newcommand{\Z}{\mathbf{Z}}
\newcommand{\Out}{O}
\newcommand{\Outb}{\bar{\Out}}

\newcommand{\indic}[1]{\mathbbm{1}\left\{#1\right\}}

\newcommand{\pI}{\uppi}					% Asmyptotic distribution
\newcommand{\psI}{\psi}

\newcommand{\Gen}{\mathcal G}				% Generator of Langevin diffusion
\newcommand{\carre}{\Gamma}
\newcommand{\Ent}{\mathrm{Ent}}
\newcommand{\entropy}{\textrm H}
\newcommand{\p}{p}

\newcommand{\nU}{\upnu}
\newcommand{\mU}{\upmu}
\newcommand{\rhO}{\uprho}

\newcommand{\X}{\mathcal{X}}				% Data space
\newcommand{\C}{\mathcal{O}}				% Output/parameter space
\newcommand{\dime}{d}					% Dimension of parameters
\newcommand{\domain}{\R^\dime}				% Parameter space for deep learning
\newcommand{\R}{\mathbb{R}}				% Real numbers
\newcommand{\N}{\mathbb{N}}				% Natural numbers
\newcommand{\Pub}{\Phi}					% space of publishable outcomes
\newcommand{\pub}{\phi}					% publishable outcomes
\newcommand{\D}{\mathcal{D}}				% Database
\newcommand{\n}{n}					% Number of records in database
\newcommand{\Lrn}{\mathrm{A}}
\newcommand{\Unlrn}{{\bar{\Lrn}}}			% Unlearning algorithm
\newcommand{\Nsgd}{\text{Noisy-GD}}

\newcommand{\x}{\mathbf{x}}				% Single datapoint from data universe
\newcommand{\thet}{\theta}				% Parameter value
\newcommand{\Thet}{\Theta}				% Parameter random variable
\newcommand{\loss}{\boldsymbol{\ell}}			% Loss function on a datapoint
\newcommand{\rloss}{\bar\loss}
\newcommand{\reg}{\mathbf{r}}				% Regularizer
\newcommand{\ltwo}{L2}					% L2 norm
\newcommand{\Loss}{\mathcal{L}}				% Loss function on a database
\newcommand{\rLoss}{\Loss}				% Regularized loss function on a database
\newcommand{\err}{\mathrm{err}}				% Empirical risk
\newcommand{\risk}{\alpha}
\newcommand{\q}{q}					% Renyi divergence order
\newcommand{\eps}{\varepsilon}				% Privacy/indistinguishability budget parameter
\newcommand{\epsdp}{\eps_{\mathrm{dp}}}			% RDP budget parameter 
\newcommand{\epsdd}{\eps_{\mathrm{dd}}}			% Data deletion budget parameter
\newcommand{\del}{\delta}				% Probability of breach in Privacy/indistinguishability
\newcommand{\step}{\eta}				% Step size
\newcommand{\K}{K}					% Number of steps
\newcommand{\batch}{\mathcal{B}}			% subsampled batch
\newcommand{\cvx}{\lambda}
\newcommand{\lip}{L}
\newcommand{\Clip}{\mathrm{Clip}}

\newcommand{\smh}{\beta}
\newcommand{\B}{B}					% bound on loss
\newcommand{\conv}{\circledast}
\newcommand{\grad}{\mathbf{g}}				% stochastic gradient

\newcommand{\y}{\mathbf{y}}				% Single datapoint from data universe
\newcommand{\up}{u}					% Database update request
\newcommand{\Up}{U}					% Database update request random variable
\newcommand{\ind}{\mathrm{ind}}
\newcommand{\U}{\mathcal{U}}				% Database update reqeust space
\newcommand{\reps}{r}					% Max number of records edited

\newcommand{\prox}{\mathrm{prox}}
\newcommand{\publish}{f_{\mathrm{pub}}}
\newcommand{\updreq}{\mathcal{Q}}

\newcommand{\pubs}{p}

\newcommand{\query}{h}					% Notation for count query
\newcommand{\median}{\mathrm{med}}			% Notation for median

\newcommand{\LS}{\mathrm{LS}}
\newcommand{\lsi}{c}					% LSI constant
\newcommand{\T}{T}					% bijective map

\newcommand{\eqdef}{\mathrel{\mathop:}=}

\newcommand{\approxDP}[2]{\overset{#1,#2}{\approx}}		% (\eps, \del)-indistinguishability
\newcommand{\KL}[2]{\mathrm{KL}\left(#1\middle\Vert#2\right)}	% KL divergence
\newcommand{\Fis}[2]{\mathrm{I}\left(#1\middle\Vert#2\right)}	% Fisher information
\newcommand{\Was}[2]{\mathrm{W}_2\left(#1,#2\right)}
\newcommand{\Ren}[3]{\mathrm{R}_{#1}\left(#2\middle\Vert#3\right)}	% Renyi divergence
\newcommand{\Eren}[3]{\mathrm{E}_{#1}\left(#2\middle\Vert#3\right)}	% qth moment of likelihood ratio
\newcommand{\Gren}[3]{\mathrm{I}_{#1}\left(#2\middle\Vert#3\right)}	% qth moment of likelihood ratio
\newcommand{\TV}[1]{\mathbf{TV}\left(#1\right)}

\newcommand{\Test}{\text{Test}}
\newcommand{\Testb}{\overline{\text{Test}}}

\usepackage{xcolor}
\usepackage[numbers,square]{natbib}
\usepackage[top=1.2in,bottom=1.2in,left=1.2in,right=1.2in,centering,letterpaper]{geometry}
\usepackage{hyperref}
\usepackage{bibentry}
\nobibliography*

\title{Forget Unlearning: Towards True Data-Deletion in \\ Machine Learning}

\author{Rishav Chourasia \& Neil Shah \\
	Department of Computer Science\\
	National University of Singapore\\
	\texttt{\{rishav1,neilshah\}@comp.nus.edu.sg} \\
}
\date{\today}

\begin{document}

\maketitle

% for adding table of contents for the appendix
\doparttoc % Tell to minitoc to generate a toc for the parts
\faketableofcontents % Run a fake tableofcontents command for the partocs

% !TEX root = main.tex

\begin{abstract}
	    Unlearning algorithms aim to remove deleted data's influence from trained models at a cost lower than full retraining. However, prior guarantees of unlearning in literature are flawed and don't protect the privacy of deleted records. We show that when users delete their data as a function of published models, records in a database become interdependent. So, even retraining a fresh model after deletion of a record doesn't ensure its privacy. Secondly, unlearning algorithms that cache partial computations to speed up the processing can leak deleted information over a series of releases, violating the privacy of deleted records in the long run. To address these, we propose a sound deletion guarantee and show that the privacy of existing records is necessary for the privacy of deleted records. Under this notion, we propose an accurate, computationally efficient, and secure machine unlearning algorithm based on noisy gradient descent.
\end{abstract}

% !TEX root = ../main.tex

\section{Introduction}
\label{sec:intro}
Corporations today collect their customers' private information to train Machine Learning models that power a variety of services like recommendations, searches, targeted ads, etc. To prevent any unintended use of personal data, privacy policies, such as the \citet{GDPR16} (GDPR) and the \citet{CCPA18} (CCPA), require that these corporations provide the ``\emph{Right to be Forgotten}'' (RTBF) to their users---if a user wishes to revoke access to their data, an organization must comply by erasing all information about her without undue delay (typically a month). Critically, models trained in standard ways are susceptible to model inversion~\citep{fredrikson2015model} and membership inference attacks~\citep{shokri2017membership}, demonstrating that training data can be exfiltrated from these models.

Periodic retraining of models after excluding deleted users can be computationally expensive. Consequently, there is a growing interest in designing computationally cheap \emph{Machine Unlearning} algorithms as an alternative to retraining for erasing the influence of deleted data from trained models. Since it is generally difficult to tell how a specific data point affects a model, \citet{ginart2019making} propose quantifying the worst-case information leakage from an unlearned model through an \emph{unlearning guarantee} on the mechanism, defined as a \emph{differential privacy} (DP) like $(\eps, \del)$-indistinguishability between its output and that of retraining on the updated database. With some minor variations in this definition, several mechanisms have been proposed and certified as unlearning algorithms in literature~\citep{ginart2019making,izzo2021approximate,sekhari2021remember,neel2021descent,guo2019certified,ullah2021machine}.

However, \emph{is indistinguishability to retraining a sufficient guarantee of deletion privacy?} We argue that it is not. In the real world, a user's decision to remove his information is often affected by what a deployed model reveals about him. Unfortunately, the same revealed information may also affect other users' decisions. Such \emph{adaptive} requests make the records in a database interdependent. For instance, if an individual is identified to be part of the training set, many members of his group may request deletion. Therefore, the representation of different groups in the unlearned model can reveal the individual's affiliation, even when he is no longer part of the dataset. We demonstrate on mechanism certified under existing unlearning guarantees, including~\citet{gupta2021adaptive}'s adaptive unlearning, that the identity of the target record can be inferred from the unlearned model when requests are adaptive. Since it is possible that adaptive deletion requests can encode patterns specific to a target record in the curator's database, we argue that any deletion privacy certification via indistinguishability to retraining, as done in all prior unlearning definitions, is fundamentally flawed.

% We argue that under adaptive requests, measuring data-deletion via indistinguishability to retraining (as proposed by~\citet{gupta2021adaptive}) is fundamentally flawed because it does not capture the influence a record might have previously had on the rest of the database. Our example shows a clear violation of the RTBF since even after retraining on the database with the original record removed, a model can reveal substantial information about the deleted record due to the possibility of re-encodings.

\emph{Is an unlearning guarantee a sound and complete measure of deletion privacy when requests are non-adaptive?} Again, we argue that it is neither. A sound deletion privacy guarantee must ensure the non-recovery of deleted records from an \emph{infinite} number of model releases after deletion. However, approximate indistinguishability to retraining implies an inability to accurately recover deleted data from a \emph{singular} unlearned model only, which we argue is not sufficient. We show that certain algorithms can satisfy an unlearning guarantee yet blatantly reveal the deleted data eventually over multiple releases. The vulnerability arises in algorithms that maintain partial computations in internal data-dependent states for speeding up subsequent deletions. These internal states can retain information even after record deletion and influence multiple future releases, making the myopic unlearning guarantee unreliable in sequential deletion settings. Several proposed unlearning algorithms in literature~\citep{neel2021descent, gupta2021adaptive} are stateful (rely on internal states) and, therefore, cannot be trusted. Secondly, existing unlearning definitions are incomplete notions of deletion privacy they exclude valid deletion mechanisms that do not imitate retraining. For instance, a (useless) mechanism that outputs a fixed untrained model on any request is a valid deletion algorithm. However, since its output is easily distinguishable from retraining, it fails to satisfy these unlearning guarantees.

This paper proposes a \emph{sound definition of data-deletion} that does not suffer from the aforementioned shortcomings. Under our paradigm, a data-deletion mechanism is reliable if {\bf A)} it is stateless, i.e., does not rely on any secret states that may be influenced by previously deleted records; and {\bf B)} it generates models that are indistinguishable from some deleted record independent random variable. Statelessness thwarts the danger of sustained information leakage through internal data structures after deletion. Moreover, by measuring its deletion privacy via indistinguishability with any deleted-record independent random variable, as opposed to the output of retraining, we ensure reliability in presence of adaptive requests that can create dependence between current and deleted records in the database. 

In general, we show that \emph{under adaptive requests, any data-deletion mechanism must be privacy-preserving with respect to existing records to ensure the privacy of deleted records.} Privacy with respect to existing records is necessary to prevent adaptive requests from creating any unwanted correlations among present and absent database entries that prevents deletion of records in an information theoretic sense. We also prove that if a mechanism is differentially private with respect to the existing records and satisfies our data-deletion guarantee under non-adaptive edit requests, then it also satisfies a data-deletion guarantee under adaptive requests. That is, we prove \emph{a general reduction for our sound data-deletion guarantee under non-adaptive requests to adaptive requests when the unlearning mechanism is differentially private with respect to records not being deleted}. We emphasize that we are not advocating for doing data deletion through differentially-private mechanisms simply because it caps the information content of all records equally, deleted or otherwise. Instead, a data-deletion mechanism should provide two differing information reattainment bounds; one for records currently in the database in the form of a differential privacy guarantee, and the other for records previously deleted in the form of a non-adaptive data-deletion guarantee, as these two information bounds together ensure deletion privacy under adaptive requests as well. 

Based on our findings, we redefine the problem of data-deletion as designing a mechanism that {\bf (1.)} satisfies a data-deletion guarantee against non-adaptive deletion requests, {\bf (2.)} is differentially private for remaining records, and {\bf (3.)} has the same utility guarantee as retraining under identical differential privacy constraints. On top of these objectives, a data-deletion mechanism must also be computationally cheaper than retraining for being useful. We propose a data-deletion solution based on Noisy Gradient Descent (Noisy-GD), a popular differentially private learning algorithm~\citep{bassily2014private,abadi2016deep}, and show that our solution satisfies all the three objectives while providing substantial computational savings for both convex and non-convex losses. Our solution demonstrates a powerful synergy between data deletion and differential privacy as the same noise needed for the privacy of records present in the database also rapidly erases information regarding records deleted from the database. For convex and smooth losses, we certify that under a $(\q, \epsdd)$-R\'enyi data-deletion and $(\q,\epsdp)$-R\'enyi DP constraint, our Noisy-GD based deletion mechanism for $\dime$-dimensional models over $\n$-sized databases with requests that modify no more than $\reps$ records can maintain a tight optimal excess empirical risk of the order $O \big(\frac{\q\dime}{\epsdp\n^2}\big)$ while being $\Omega(\n\log(\min\{\frac{\n}{\reps}, \n\sqrt{\frac{\epsdd}{\q\dime}}\})$ cheaper than retraining in gradient complexity. For non-convex, bounded and smooth losses we show a computational saving of $\Omega(\dime\n\log\frac{\n}{\reps})$ in gradient complexity while satisfying the three objectives with an excess risk bound of $\tilde O\big(\frac{\q\dime}{\epsdp\n^2} + \frac{1}{\n} \sqrt{\frac{\q}{\epsdp}}\big)$. Compared to our results, prior works have a worse computation saving under same utility constraints~\citep{izzo2021approximate, guo2019certified, sekhari2021remember, ullah2021machine}, do not satisfy differential privacy~\citep{bourtoule2021machine,gupta2021adaptive}, or require unsafe internal data structures for matching our utility bounds~\citep{neel2021descent}.

\section{Preliminaries}
\label{sec:prelim}

\subsection{Indistinguishability and Differential Privacy}
\label{ssec:indistinguishability_notions}

We provide the basics of indistinguishability of random variables (with more details in Appendix~\ref{ssec:appendix_indistinguishability}). Let $\Thet, \Thet'$ be two random variables in space $\C$ with densities $\nU, \nU'$ respectively. 
\begin{definition}[$(\eps, \del)$-indistinguishability~\citep{dwork2014algorithmic}]
	We say $\Thet$ and $\Thet'$ are $(\eps,\del)$-\emph{indistinguishable} (denoted by $\Thet \approxDP{\eps}{\del} \Thet'$) if, for all $\Out \subset \C$, 
	\begin{align}
			\prob{}{\Thet \in \Out} \leq e^\eps \prob{}{\Thet' \in \Out} + \del \quad \text{and} \quad \prob{}{\Thet' \in \Out} \leq e^\eps \prob{}{\Thet \in \Out} + \del.
	\end{align}
\end{definition}
\begin{definition}[R\'enyi divergence~\citep{renyi1961measures}]
    \label{dfn:renyi}
    \emph{R\'enyi divergence} of $\nU$ w.r.t. $\nU'$ of order $\q > 1$ is defined as
    \begin{equation}
        \label{eqn:renyi_dfn_main}
        \Ren{\q}{\nU}{\nU'} = \frac{1}{\q - 1} \log \Eren{\q}{\nU}{\nU'}, \quad
        \text{where} \quad \Eren{\q}{\nU}{\nU'} = \expec{\thet \sim \nU'}{\left(\frac{\nU(\thet)}{\nU'(\thet)}\right)^\q},
    \end{equation}
    when $\nU$ is \emph{absolutely continuous} w.r.t. $\nU'$ (denoted as $\nU \ll \nU'$). If $\nU \not\ll \nU'$, we'll say $\Ren{\q}{\nU}{\nU'} = \infty$. 
\end{definition}
\begin{remark}
	\label{rem:renyi_onesided}
	R\'enyi divergence is asymmetric $({\Ren{\q}{\nU}{\nU'} \neq \Ren{\q}{\nU'}{\nU}})$. \citet[Proposition 3]{mironov2017renyi} show that ${\Ren{\q}{\nU}{\nU'} \leq \eps_0}$ implies indistinguishability only in one direction, i.e., for any $\Out \subset \C$, we have ${\prob{}{\Thet \in \Out} \leq e^\eps \prob{}{\Thet' \in \Out} + \del}$, where $\eps = \eps_0 + \frac{\log 1/\del}{\q - 1}$ for any $0 < \del < 1$.
\end{remark}
\begin{definition}[(R\'enyi) Differential Privacy~\citep{dwork2014algorithmic, mironov2017renyi}]
	A randomized mechanism ${\mathcal{M}:\X^\n\rightarrow\C}$ is said to be $(\eps, \del)$-\emph{differentially private} if ${\mathcal{M}(\D) \approxDP{\eps}{\del} \mathcal{M}(\D')}$ for all \emph{neighbouring} databases ${\D, \D' \in \X^\n}$. Similarly, $\mathcal{M}$ is $(\q, \eps)$-\emph{R\'enyi differentially private} if $\Ren{\q}{\mathcal{M}(\D)}{\mathcal{M}(\D')} \leq \eps$.
\end{definition}

\subsection{(Adaptive) Machine Unlearning}
\label{ssec:unlrn}

Let $\X$ be the data domain. A database $\D$ is an ordered set of $\n$ records from $\X$. We use $\C$ to denote the space of models. A \emph{learning algorithm} ${\Lrn: \X^\n \rightarrow \C}$ inputs a database ${\D \in \X^\n}$ and returns a (possibly random) model in $\C$. \citet{ginart2019making} defines a \emph{data deletion operation} for a \emph{machine learning} algorithm $\Lrn$ as follows. 
\begin{definition}[Data deletion operation~{\citep{ginart2019making}}]
	\label{def:ginart}
	Algorithm $\Unlrn:\X^\n \times \X^\reps \times \C \rightarrow \C$ is a \emph{data deletion operation} for a learning algorithm $\Lrn: \X^\n \rightarrow \C$ if for any database $\D \in \X^\n$, $\Unlrn(\D, S, \Lrn(\D)) \approxDP{0}{0} \Lrn(\D \setminus S)$ for all subset $S \subset \D$ of size $\reps$ selected \emph{independently} of $\Lrn(\D)$.
\end{definition}
%
% \begin{remark}
% 	\citet{guo2019certified} and \citet{sekhari2021remember} propose similar unlearning notions for batch deletions. We describe them in Appendix~\ref{app:two_stage_unlearning}.
% \end{remark}
%

In our paper, we mostly consider batched replacement\footnote{\footnotesize We consider replacements instead of deletion operations to ensure that database size doesn't change.} \emph{edit requests} as stated below.

% Suppose that any database $\D \in \X^\n$ can be modified by \emph{edit requests} that replaces $\reps$ distinct records~\footnote{\footnotesize We consider replacement instead of separate addition/deletion to ensure that database size doesn't change.}.
%\footnote{\footnotesize Prior works define edits as addition or deletion of a record to the database\citep{neel2021descent,ginart2019making} which alters the database size. We only consider replacement edits for analytical simplicity.}
%
\begin{definition}[Edit request]
	A \emph{replacement operation} $\langle \ind, \y \rangle \in [\n] \times \X$ on a database ${\D = (\x_1, \cdots, \x_n) \in \X^\n}$ performs the following modification:
	\begin{equation}
		\D \circ \langle \ind, \y \rangle = (\x_1, \cdots, \x_{\ind-1}, \y, \x_{\ind+1}, \cdots, \x_\n).
	\end{equation}
	Let $\reps \leq \n$ and $\U^\reps = [\n]^\reps_{\neq} \times \X^\reps$. An \emph{edit request} $\up = \{\langle \ind_1, \y_1\rangle, \cdots, \langle \ind_\reps, \y_\reps\rangle \} \in \U^\reps$ on $\D$ is defined as batch of $\reps$ replacement operations modifying distinct indices atomically, i.e. 
	\begin{equation}
		\D \circ \up = \D \circ \langle \ind_1, \y_1\rangle \circ \cdots \circ \langle \ind_\reps, \y_\reps\rangle,
	\end{equation}
	where $\ind_i \neq \ind_j$ for all $i \neq j$.
\end{definition}

Similar to~\citet{ginart2019making}, we define a \emph{deletion} or an \emph{unlearning algorithm} as a (possibly stochastic) mapping ${\Unlrn: \X^\n \times \U^\reps \times \C \rightarrow \C}$ that takes in a database $\D \in \X^\n$, an edit request $\up \in \U^\reps$ and the current model in $\C$, and outputs an updated model in $\C$. We adopt the online setting of~\citet{neel2021descent} in which a stream of edit requests $(\up_i)_{i\geq1} \stackrel{\mathrm{def}}{=} (\up_1, \up_2, \cdots)$, with $\up_i \in \U^\reps$, arrive in sequence. In this formulation, the \emph{data curator} (characterized by $(\Lrn, \Unlrn)$) executes the learning algorithm $\Lrn$ on the initial database $\D_0 \in \X^\n$ during the setup stage before arrival of the first edit request to generate the initial model $\hat\Thet_0 \in \C$, i.e. $\hat\Thet_0 = \Lrn(\D_0)$. Thereafter at any edit step $i \geq 1$, to reflect an incoming edit request $\up_i \in \U^\reps$ that transforms $\D_{i-1} \circ \up_i \rightarrow \D_i$, the curator executes the unlearning algorithm $\Unlrn$ on current database $\D_{i-1}$, the edit request $\up_i$, and the current model $\hat\Thet_{i-1}$ for generating the next model $\hat\Thet_i \in \C$, i.e. $\Unlrn(\D_{i-1}, \up_i, \hat\Thet_{i-1}) = \hat\Thet_{i-1}$. Furthermore, the curator keeps secret the sequence $(\hat\Thet_i)_{i\geq0} \stackrel{\mathrm{def}}{=} (\hat\Thet_0, \hat\Thet_1, \cdots)$ of (un)learned models, only releasing publishable objects $\pub_i = \publish(\hat\Thet_i)$, for all $i \geq 0$, generated using a publish function $\publish: \C \rightarrow \Pub$. Here $\Pub$ is the space of publishable objects like model predictions on an external dataset or noisy model releases.

\citet{ginart2019making} note that the assumption of independence between a deletion request and the preceding model in Definition~\ref{def:ginart} might not always hold. In real world, deletion requests could often be \emph{adaptive}, i.e., may depend on the prior published objects. For instance, security researchers may demonstrate privacy attacks targeting minority subpopulation on publicly available models, causing people in that subpopulation to request deletion of their information from training data. \citet{gupta2021adaptive} model such an interactive environment through an \emph{adaptive update requester}. We provide the following generalized definition of an \citet{gupta2021adaptive}'s update requester and describe its interaction with a curator in Algorithm~\ref{alg:interact}.

% \citet{gupta2021adaptive} note that edit requests in real world could often be \emph{adaptive}, i.e., a request $\up_i$ may depend on (a subset of) the history of prior published outcomes ${\pub_{\leq i}  = (\pub_1, \cdots, \pub_i)}$. For instance, a voter may decide to change his inclination after seeing pre-election results. They model such an intearctive environment through an \emph{adaptive requester} defined as follows.
%
\begin{definition}[Update requester~\citep{gupta2021adaptive}]
	\label{dfn:updreq}
	The update sequence $(\up_i)_{i\geq1}$ is generated by an update requester $\updreq$ that inputs a subset of interaction history between herself and the curator $(\Lrn, \Unlrn)$, and outputs a new edit request for the current round. We quantify the strength of $\updreq$ with two integers $(\pubs, \reps)$. Here $\pubs$ is the maximum number of prior published objects that the requester $\updreq$ has access to for generating the subsequent request and $\reps$ is the number of records that can be edited per request. More formally, a \emph{$\pubs$-adaptive $\reps$-requester} is a mapping ${\updreq: \Pub^{\leq\pubs} \times \U^{\reps *} \rightarrow \U^\reps}$. Given a sorted list of observable indices $\vec{s} = (s_1, \cdots, s_\pubs) \in \N^\pubs$ the $i^{\text{th}}$ edit request $\up_i$ generated by $\updreq$ on interaction with $(\Lrn, \Unlrn)$ is defined as
	\vspace{-0.1cm}
	\begin{equation}
		\up_i = \updreq(\underbrace{\pub_{s_1}, \pub_{s_2}, \cdots, \pub_{s_j}}_{\stackrel{\mathrm{def}}{=}\pub_{\vec{s} < i}}; \underbrace{\up_1, \up_2, \cdots, \up_{i-1}}_{\stackrel{\mathrm{def}}{=} \up_{< i}}), 
		% \stackrel{\mathrm{def}}{=} \updreq(\pub_{\vec{s} < i}; \up_{< i})
	\end{equation}
	where $s_j$ is the largest index in $\vec{s}$ that is less than $i$.
	% A \emph{$\pubs$-adaptive $\reps$-requester} is a mapping ${\updreq: \Pub^{\leq\pubs} \times \U^{\reps *} \rightarrow \U^\reps}$ that takes as input a maximum of $\pubs$ of the published outcomes generated by the curator at arbitrary edit steps $s_1 < s_2 < \cdots < s_\pubs$, and the entire history of previously generated edit requests to generate the next edit request. For a $\pubs$-adaptive $\reps$-requester $\updreq$, the edit request $\up_i$ at any step $i\geq1$ can be written as
	% %
	% \begin{equation}
	% 	\up_i = \updreq(\pub_{s_1}, \pub_{s_2}, \cdots, \pub_{s_j}; \up_1, \up_2, \cdots, \up_{i-1}),
	% \end{equation}
	% %
	% such that $s_j < i$. We refer to $0$-adaptive requesters as non-adaptive. And, by $\infty$-adaptive requesters, we mean requesters that have access to the entire history of interaction transcript $(\pub_{<i};\up_{<i})$.

%	, i.e., edit request $\up_i$ is a function of $(\pub_0, \up_1, \pub_1, \cdots, \up_{i-1}, \pub_{i-1})$.
\end{definition}
\vspace{-0.4cm}
\begin{algorithm}[htbp!]
	\caption{Interaction between $(\Lrn, \Unlrn)$ and requester $\updreq$.}
        \label{alg:interact}
        \begin{algorithmic}[1]
		\REQUIRE Database $\D_0 \in \X^\n$, observable indices $\vec{s} \in \N^{\pubs}$.
		\STATE {Initialize $\hat\Thet_0 \leftarrow \Lrn(\D_0)$}
		\STATE {Publish $\pub_0 \leftarrow \publish(\hat\Thet_0)$}
		\FOR{$i = 1, 2, \cdots$}
			\STATE {Get next request $\up_i \leftarrow \updreq(\pub_{\vec{s} < i}; \up_{<i})$}
			\STATE {Update model $\hat\Thet_i \leftarrow \Unlrn(\D_{i-1}, \up_i, \hat\Thet_{i-1})$}
			\STATE {Publish $\pub_i \leftarrow \publish(\hat\Thet_i)$}
			\STATE {Update database $\D_i \leftarrow \D_{i-1} \circ \up_i$}
		\ENDFOR
	\end{algorithmic}
\end{algorithm}
\vspace{-0.3cm}
We denote $0$-adaptive requesters as non-adaptive and by $\infty$-adaptivity we mean requesters that have access to the entire history of interaction transcript $(\pub_{<i};\up_{<i})$ at step $i$. \citet{neel2021descent} and \citet{gupta2021adaptive} define non-adaptive and adaptive unlearning\footnote{\footnotesize Definition~\ref{dfn:unlearning} of adaptive unlearning is stronger than \citet{gupta2021adaptive}'s since theirs require only one-sided indistinguishability with $(1 - \gamma)$ probability over generated edit requests $\up_{\leq i}$.} as follows.

\begin{definition}[(Adaptive) machine unlearning~\citep{neel2021descent,gupta2021adaptive}]
\label{dfn:unlearning} We say that $\Unlrn$ is a $(\eps, \del)$-\emph{non-adaptive-unlearning} algorithm for $\Lrn$ under a publish function $\publish$, if for all initial databases $\D_0 \in \X^\n$ and all non-adaptive $1$-requesters $\updreq$, the following condition holds. For every edit step $i \geq 1$, and for all generated edit sequences ${\up_{\leq i} \stackrel{\mathrm{def}}{=} (\up_1, \cdots, \up_i)}$,
\begin{equation}
	\label{eqn:unlearn_neel}
	\publish(\Unlrn(\D_{i-1}, \up_i, \hat\Thet_{i-1})) \big|_{\up_{\leq i}} \approxDP{\eps}{\del} \publish(\Lrn(\D_i)).
\end{equation}
If \eqref{eqn:unlearn_neel} holds for all $\infty$-adaptive $1$-requesters $\updreq$, we say that $\Unlrn$ is an $(\eps, \del)$-\emph{adaptive-unlearning} algorithm for $\Lrn$.
\end{definition}

\section{Existing Deletion Definitions are Unsound and Incomplete}
\label{sec:deletion_vs_unlearn}

% Data deletion under the law of ``right to be forgotten'' (RTBF) is an obligation to \emph{permanently erase} all information about an individual upon a verified request. In order to comply, a corporation's actions must not reveal any information identifiable or linkable to a deleted user in the future. In this section, we argue that unlearning guarantee in Definition~\ref{dfn:unlearning} is \emph{neither a sound nor a complete} measure of data-deletion from ML models that RTBF enforces. 

In this section we investigate the failure of prior definitions of certified data-deletion proposed for ensuring the "Right to be Forgotten" (RTBF) guideline. In particular, we shed light on multiple reasons why both adaptive and non-adaptive machine unlearning as described in Definition~\ref{dfn:unlearning} and several other definitions in literature are flawed notions of data deletion for enforcing RTBF. 

{\bf Threat model.} Suppose, for an arbitrary step $i \geq 1$, an adversary is interested in finding out the identity of a record in the database $\D_{i-1}$ that is being replaced by edit request $\up_i$. Our adversary only has access to releases by the curator \emph{after deletion}\footnote{\footnotesize Our threat model doesn't include attacks which involve comparing releases before and after deletion (such as~\citet{chen2021machine}'s). These types of attacks are not covered by RTBF as they rely on information published before deletion request.}, i.e. the infinite post-deletion sequence $\pub_{\geq i} \stackrel{\mathrm{def}}{=} (\pub_i, \pub_{i+1}, \cdots)$. The adversary in our threat model is assumed to also have some understanding of how users may react to published data, but does not have any access to the published data or user requests. That is, our adversary has some knowledge about the \emph{dependence relationship} between the published objects random variables $\pub_{<i} \stackrel{\mathrm{def}}{=} (\pub_0, \cdots, \pub_{i-1})$ and the corresponding edit requests random variables $\up_{<i} \stackrel{\mathrm{def}}{=} (\up_1, \cdots, \up_{i-1})$, but does not observe these random variables. For example, adversary may know that if a certain individual is identified in the published object $\pub_0$, many members of his group will request for deletion $\up_1$. Our adversary can use her knowledge about this dependence to infer the affiliation of a deleted individual by evaluating the representations of different groups in subsequent release $\pub_1$. To capture the worst-case scenario, we model the adversary as having the ability to design an adaptive requester $\updreq$ (as defined in Definition~\ref{dfn:updreq}) that interacts with the data curator in previous $i-1$ steps, but the adversary does not have access to the interaction transcript $(\pub_{<i}; \up_{<i})$ of $\updreq$. 

{\bf Unsoundness due to adaptivity.} We highlight the problem with unlearning on adaptive requests with a simple example. For a data domain $\X = \{-2, -1, 1, 2\}$, consider the following algorithm pair $(\Lrn, \Unlrn)$ for any database $\D \subset \X$ and deletion $S \subset \D$.
\begin{equation}
	\Lrn(\D) = \sum_{\x \in \D} \x, \quad \text{and} \quad \Unlrn(\D, S, \Lrn(\D)) = \sum_{\x \in \D \setminus S} \x.
\end{equation}
Note that pair $(\Lrn, \Unlrn)$ satisfies \citet{ginart2019making}'s Definition~\ref{def:ginart} of a data deletion operation as for any $\D \subset \X$ and any $S \subset \D$, we have $\Unlrn(\D, S, \Lrn(\D)) = \Lrn(\D \setminus S)$. Now consider two neighbouring databases $\D = \{-2, -1, 2\}$, $\D' = \{-2, 1, 2\}$ and the following dependence between the learned model $\Lrn(\bar\D)$ and deletion request $S$:
\begin{equation}
	S = \begin{cases} \{\x \in \X | \x < 0\} &\text{if}\ \Lrn(\bar\D) < 0, \\ \{\x \in \X | \x \geq 0\} &\text{otherwise.}\end{cases}
\end{equation}
Knowing this dependence, an adversary can distinguish whether $\bar\D$ is $\D$ or $\D'$ by looking only at $\Unlrn(\bar\D, \x, \Lrn(\bar\D))$. This is because if $\bar\D = \D$, then the output after deletion is $2$, and if $\bar\D = \D'$ the output is $-2$. Note that even though $\Unlrn$ perfectly imitates retraining via $\Lrn$ and the adversary does not observe either the model $\Lrn(\bar\D)$ or the request $S$, she can still ascertain the identity ($-1$ or $1$) of a deleted record.

\citet{ginart2019making}'s Definition~\ref{def:ginart} isn't suited for adaptive deletion as it explicitly requires deletion requests to be selected independently of the learned model. However, \citet{gupta2021adaptive}'s Definition~\ref{dfn:unlearning} of an adaptive-unlearning algorithm seeks to ensure RTBF specifically when requests could be adaptive. In the following theorem, we show using a similar construction that algorithms certified to be $(0,0)$-adaptive-unlearning can still blatantly violate privacy of deleted records under adaptivity.

% Our example shows that adaptive requests can cause the curator's database to have patterns specific to the identity of a target record being deleted. An adversary knowing the possible patterns and their causes (i.e., the dependence relationship between unobserved releases and edit requests) can therefore infer the target records's identity from post-deletion releases. 

%
\begin{theorem}
	\label{thm:adaptive_violation}
	There exists an algorithm pair $(\Lrn, \Unlrn)$ satisfying $(0, 0)$-adaptive-unlearning under publish function $\publish(\thet) = \thet$ such that by designing a $1$-adaptive $1$-requester $\updreq$, an adversary can infer the identity of a record deleted by edit $\up_i$, at any arbitrary step $i > 3$, with probability at-least $1 - (1/2)^{i-3}$ from a single post-edit release $\pub_i$, even with no access to $\updreq$'s transcript $(\pub_{<i}; \up_{<i})$.
\end{theorem}
In Appendix~\ref{app:two_stage_unlearning} we show that other unlearning definitions in literature, like that of~\citet{sekhari2021remember} and~\citet{guo2019certified}, are also unsound under adaptivity.

{\bf Unsoundness due to secret states.} Both adaptive and non-adaptive unlearning guarantees in Definition~\ref{dfn:unlearning} are bounds on information leakage about a deleted record through a \emph{single released output}. However, our adversary can observe multiple (potentially infinite) releases after deletion. We identify a yet another reason for violation of RTBF under Definition~\ref{dfn:unlearning}, even when edit requests are non-adaptive. This vulnerability arises because Definition~\ref{dfn:unlearning} permits the curator to store secret models while requiring indistinguishability only over the output of a publishing function $\publish$. These secret models may propagate encoded information about records even after their deletion from the database. So, every subsequent release by an unlearning algorithm can reveal new information about a record that was purportedly erased multiple edits earlier. We demonstrate in the following theorem that a certified unlearning algorithm can reveal a limited amount of information about a deleted record per release so as not to break the unlearning certification, yet eventually reveal everything about the record to an adversary that observes enough future releases.
\begin{theorem}
\label{thm:unlrn_vs_forget}
For every $\eps > 0$, there exists a pair $(\Lrn, \Unlrn)$ of algorithms that satisfy $(\eps, 0)$-non-adaptive-unlearning under some publish function $\publish$ such that for all non-adaptive $1$-requesters $\updreq$, their exists an adversary that can correctly infer the identity of a record deleted at any arbitrary edit step $i \geq 1$ by observing only the post-edit releases $\pub_{\geq i}$.
\end{theorem}
Although \citet{ginart2019making}'s Definition~\ref{def:ginart} (and those of \citet{guo2019certified} and \citet{sekhari2021remember}) directly release the (un)learned models without applying any explicit publish function $\publish$, the above vulnerability might still arise in the online setting if a certified deletion operation relies on data-dependent states that aren't updated with database. We remark that \citet{ginart2019making}, in their online formulation, permits a deletion operation to maintain "arbitrary metadata like data structures or partial computations that can be leveraged to help with subsequent deletions", and so could be susceptible to the vulnerability we describe. 
%
% We remark that lot of data-deletion algorithms in literature that claim to ensure RTBF via unlearning guarantees~\citep{neel2021descent,bourtoule2021machine,gupta2021adaptive,ginart2019making,ullah2021machine} have at least one of the above vulnerabilities.

{\bf Incompleteness.} 
Another shortcoming with existing unlearning definitions is that many valid deletion algorithms may fail to satisfy them. For instance, consider a (useless) mechanism $\Unlrn$ that outputs a fixed untrained model in $\thet \in \C$ regardless of its inputs. It is easy to see that $\Unlrn$ is a valid deletion algorithm for any learning algorithm as $\Unlrn(\cdot)$ (and therefore $\publish(\Unlrn(\cdot))$ for any $\publish$) does not depend on the input database or the learned model. However, this $\Unlrn$ does not satisfy Definition~\ref{def:ginart} or Definition~\ref{dfn:unlearning} for any reasonable learning algorithm $\Lrn$. In Appendix~\ref{app:two_stage_unlearning} we show that unlearning definitions of~\citet{guo2019certified} and~\citet{sekhari2021remember} are also incomplete.

\section{Redefining Deletion in Machine Learning}
\label{sec:deletion}

In this section, we redefine \emph{data deletion in Machine Learning} to address the problems with existing notions of unlearning that we demonstrate in the preceding section. The first change we propose is to rule out the possibility of information leakage through internal data structures (as shown in Theorem~\ref{thm:unlrn_vs_forget}) by requiring deletion algorithms to be \emph{stateless}. That is, a data-deletion algorithm $\Unlrn$ cannot depend on any secret data-dependent states and the generated (un)learned models are released without applying any publish function $\publish$ (or equivalently, by allowing only an identity publish function $\publish(\thet) = \thet$ in Algorithm~\ref{alg:interact}). 

Secondly, we propose a definition of \emph{data-deletion} that fixes the security blindspot of existing unlearning guarantees. As demonstrated in Section~\ref{sec:deletion_vs_unlearn}, adaptive requests can encode patterns specific to a target record in the database which persists even after deletion of the target record, making any indistinguishable-to-retraining based data-deletion guarantee unreliable.
% Therefore, being indistinguishable from retraining on the edited database does not guarantee data-deletion, as the target's information remains extractable even after deletion, potentially revealing its identity. 
In the following definition, we account for the worst-case influence of adaptive requests by measuring the indistinguishability of a data-deletion mechanism's output from that of some mechanism that never sees the deleted record or edit requests influenced by it.

% As demonstrated in Theorem~\ref{thm:adaptive_violation}, an adaptive requester can encode patterns specific to a target record in the database by making edit decisions in response to the observed outcomes. Thus, being indistinguishable from retraining on the edited database does not guarantee data deletion, as the target's information remains extractable even after the target record's deletion, potentially revealing its identity. In our definition, we account for an adaptive adversary's influence by measuring the indistinguishability of a data-deletion mechanism's output from some random variable that is independent of the deleted record.
%
\begin{definition}[$(\q,\eps)$-data-deletion under $\pubs$-adaptive $\reps$-requesters]
\label{dfn:deletion}
Let $\q > 1$, $\eps \geq 0$, and $\pubs, \reps \in \N$. We say that an algorithm pair $(\Lrn,\Unlrn)$ satisfies \emph{$(\q, \eps)$-data-deletion under $\pubs$-adaptive $\reps$-requesters} if the following condition holds for all $\pubs$-adaptive $\reps$-requester $\updreq$. For every step $i\geq1$, there exists a randomized mapping $\pI^\updreq_i: \X^\n \rightarrow \C$ such that for all initial databases $\D_0 \in \X^\n$,
\begin{equation}
	\label{eqn:deletion}
	\Ren{\q}{\Unlrn(\D_{i-1}, \up_i, \hat\Thet_{i-1})}{\pI^\updreq_i(\D_0 \circ \langle \ind, \y \rangle)} \leq \eps, 
\end{equation}
for all $\up_i \in \U^\reps$ and all $\langle \ind, \y \rangle \in \up_i$.
\end{definition}
We prove that the above definition is a sound guarantee of RTBF. Suppose that an adversary is interested in identifying a record at index `$\ind$' in $\D_0$ that is being replaced with record $\y \in \X$ by one of the replacement operations in edit request $\up_i \in \U^\reps$. Note that random variable $\pI^\updreq_i(\D_0 \circ \langle\ind,\y\rangle)$ contains no information about the deleted record $\D_0[\ind]$. Inequality~\eqref{eqn:deletion} above implies that even with the power of designing an adaptive requester $\updreq$, no adversary observing the unlearned model $\Unlrn(\D_{i-1}, \up_i, \hat\Thet_{i-1})$ can be too confident that the observation was \emph{not} $\pI^\updreq_i(\D_0 \circ \langle\ind,\y\rangle)$ instead (cf. Remark~\ref{rem:renyi_onesided}). We support this argument with the following guarantee.

%We provide the following guarantee on soundness of Definition~\ref{dfn:deletion}.
%
\begin{theorem}[Data-deletion Definition~\ref{dfn:deletion} is sound]
	\label{thm:soundness}
	If the algorithm pair $(\Lrn, \Unlrn)$ satisfies $(\q, \eps)$-data-deletion guarantee under all $\pubs$-adaptive $\reps$-requesters, then even with the power of designing an $\pubs$-adaptive $\reps$-requester $\updreq$ that interacts with the curator before deletion of a target record at any step $i\geq1$, any adversary observing only the post-deletion models $(\hat\Thet_i, \hat\Thet_{i+1}, \cdots)$ has its membership inference advantage for inferring a deleted target bounded as 
\begin{equation}
	\label{eqn:adv_sound}
	\text{Adv}(\text{MI}) \leq \min\left\{\sqrt{2\eps},\frac{\q e^{\eps(\q-1)/\q}}{\q - 1}[2(\q - 1)]^{1/\q} - 1\right\}.
\end{equation}
\end{theorem}
As $\q \rightarrow \infty$, the bound in \eqref{eqn:adv_sound} approaches $\min\left\{\sqrt{2\eps}, e^\eps - 1\right\}$. 
Note that the bound in \eqref{eqn:adv_sound} approaches $0$ as $\eps \rightarrow 0$, so Definition~\ref{dfn:deletion} is sound. 
% %
% \begin{remark}
% 	\label{rem:adap_independence}
% 	Under adaptive requests, the construction $\pI^\up_i = \Lrn(\D_{i-1} \circ \up)$ does not qualify as being independent of the record that request $\up$ deletes in $\D_{i-1}$. 
% 	%If $\D_{i-1}[\ind_\up]$ was a different record originally, the interaction behaviour of an adaptive requester $\updreq$ would have been different (i.e. a different sequence of edit requests might have been generated), and hence the database $\D_{i-1}$ would also not have been the same. 
% 	Although the edit request $\up$ removes the record $\D_{i-1}[\ind_\up]$ from $\D_{i-1}$, we highlight that it does not remove the record's influences from the rest of the database due to its original presence, as an adaptive requester could have successfully duplicated/encoded information of the target record. So, arguing data-deletion via indistinguishability to $\Lrn(\D_{i-1} \circ \up)$ is flawed under adaptivity.
% \end{remark}

\begin{remark}[Data-deletion generalizes prior unlearning definitions under non-adaptivity]
	\label{rem:non_adap_independence}
	A non-adaptive requester $\updreq$ is equivalent to fixing the request sequence $(\up_i)_{i\geq1}$ a-priori. Since for any $\langle \ind, \y \rangle \in \up_i$, database $\D_{i-1} \circ \up_i = (\D_0 \circ \langle \ind, \y \rangle) \circ \up_1 \circ \cdots \circ \up_i$, note that database ${\D_{i-1}\circ\up_i}$ is a function of $\D_0 \circ \langle \ind, \y \rangle$ given a non-adaptive $\updreq$. So if $\updreq$ is non-adaptive, we can set ${\pI^\updreq_i(\D_0 \circ \langle \ind, \y \rangle) = \pI(\D_{i-1}\circ\up_i)}$ in \eqref{eqn:deletion} for any randomized map ${\pI: \X^\n \rightarrow \C}$, including the learning algorithm $\Lrn$.
	%Note that a non-adaptive requester $\updreq$ is equivalent to fixing the request sequence $(\up_i)_{i\geq1}$ a-priori. Hence, given a non-adaptive $\updreq$, the database $\D_i \circ \langle \ind, \y \rangle$ is a deterministic function of the database $\D_0 \circ \langle \ind, \y \rangle$ for any $i\geq1$, thanks to the commutativity of `$\circ$'. Since $\langle \ind, \y \rangle \in \up_i$, we remark that for a non-adaptive requester $\updreq$, the random variable $\pI^\updreq_i(\D_0 \circ \langle \ind, \y \rangle)$ in \eqref{eqn:deletion} can be the output $\pI(\D_i)$ of any randomized map $\pI: \X^\n \rightarrow \C$, including the learning algorithm $\Lrn$.
\end{remark}

% \begin{remark}
% 	Under non-adaptive requests, an adversary cannot introduce any target-dependent influence on the database, as it is equivalent to fixing the request sequence apriori. Hence, the original presence of $\D_{i-1}[\ind_\up]$ does not influence other records in $\D_{i-1}$, so $\D_{i-1} \circ \up$ is completely independent of it. Consequently, under non-adaptive requests, the distribution $\pI^\up_i$ in \eqref{eqn:deletion} can be the output $\pI(\D_{i-1} \circ \up)$ of any randomized map $\pI: \X^\n \rightarrow \C$, including the learning algorithm $\Lrn$.
% \end{remark}
% %

\subsection{Link to Differential Privacy}
A DP guarantee on $\Lrn$ and $\Unlrn$ is a bound on the information contained in an (un)learned model about individual records present in a database. From the post-processing and composition property, a DP guarantee also bounds the worst case dependence that adaptive requests can introduce between a target record and rest of the database. By virtue of this property, we show a reduction from adaptive to non-adaptive data-deletion in the following theorem when $\Lrn$ and $\Unlrn$ also satisfy R\'enyi DP.

\begin{theorem}[From adaptive to non-adaptive deletion]
	\label{thm:reduction}
	If an algorithm pair $(\Lrn, \Unlrn)$ satisfies $(\q, \epsdd)$-data-deletion under all non-adaptive $\reps$-requesters and is also $(\q, \epsdp)$-R\'enyi DP with respect to records not being deleted, then it also satisfies $(\q, \epsdd + \pubs \epsdp)$-data-deletion under all $\pubs$-adaptive $\reps$-requesters.
\end{theorem}
\begin{remark} 
	\citet{gupta2021adaptive} also prove a reduction from adaptive to non-adaptive unlearning (Definition~\ref{dfn:unlearning}) under differential privacy. We remark that our reduction is fundamentally different from theirs as they require DP to hold with with regard to a change of description of internal randomness as opposed to standard data item replacement in ours. We discuss the key differences in Appendix~\ref{subsec:gupta_compare}.
\end{remark}

To complete the picture, we show in the following theorem that to satisfy data-deletion under adaptive requests, both $\Lrn$ and $\Unlrn$ must preserve the privacy of existing records.
\begin{theorem}[Privacy of remaining records is necessary for adaptive deletion]
	\label{thm:dp_necessary}
	Let $\Test: \C \rightarrow \{0, 1\}$ be a membership inference test for $\Lrn$ to distinguish between neighbouring databases $\D, \D' \in \X^\n$. Similarly, let $\Testb: \C \rightarrow \{0, 1\}$ be a membership inference test for $\Unlrn$ to distinguish between $\bar\D, \bar\D' \in \X^\n$ that are neighbouring after applying edit $\bar\up \in \U^1$. If $\text{Adv}(\Test) > \del$ and $\text{Adv}(\Testb) > \del$, then the pair $(\Lrn, \Unlrn)$ cannot satisfy $(\q, \eps)$-data-deletion under $1$-adaptive $1$-requester for any
	\begin{equation}
		\label{eqn:dp_necessary}
		\hspace{-0.18cm}\eps < \max\left\{\frac{\del^4}{2}, \log(\q - 1)  + \frac{\q}{\q-1}\log\left(\frac{1+\del^2}{\q2^{1/\q}} \right)\right\}.
	\end{equation}

	%If learning algorithm $\Lrn: \X^\n \rightarrow \C$ is not $(0, \del)$-DP with respect to the replacement of a single record and deletion algorithm $\Unlrn: \X^\n \times \U \times \C \rightarrow \C$ is not $(0,\del)$-DP with respect to the replacement of a single record that is not being deleted, then the pair $(\Lrn, \Unlrn)$ cannot satisfy $(\q, \del^4/2)$-data-deletion under $1$-adaptive $1$-requester for any $\q > 1$.
\end{theorem}
%
% Additionally, if $\Lrn$ and $\Unlrn$ does satisfy DP with respect to existing records then a data-deletion bound under non-adaptive requesters reduces to a data-deletion bound under adaptive requesters.
% Note that our definition of data deletion introduces a gradation on the power of an update requester $\updreq$ based on the count of previous releases it is allowed to observe. We argue that this is necessary because the amount of dependence that an adaptive request can encode back in a database is proportional to the information it can extract about a target record from past observations. We formalize this idea in the following theorem.
%
% \begin{theorem}[Non-adaptive data-deletion with R\'enyi DP implies adaptive data-deletion]
% 	\label{thm:reduction}
% 	If an algorithm pair $(\Lrn, \Unlrn)$ satisfies $(\q, \epsdd)$-data-deletion under all non-adaptive $\reps$-requesters and $(\q, \epsdp)$-R\'enyi DP, then it also satisfies $(\q, \epsdd + \pubs \epsdp)$-data-deletion under all $\pubs$-adaptive $\reps$-requesters.
% \end{theorem}
% %
% We provide additional discussion on this reduction theorem in Appendix~\ref{subsec:gupta_compare}.

% We reformulate the data deletion problem in ML as follows.

\subsection{(Un)Learning Framework: ERM}
\label{ssec:erm}

Let space of model parameters be $\domain$ and $\loss(\thet; \x): \domain \times \X \rightarrow \R$ be a loss function of a parameter $\thet \in \domain$ for a record $\x \in \X$. We consider the problem of \emph{empirical risk minimization} (ERM) of the average $\loss(\thet;\x)$ over records in the database $\D \in \X^\n$ under $\ltwo$ regularization, that is, the minimization objective is
\begin{equation}
	\label{eqn:objective}
	\hspace{-0.1cm}\rLoss_{\D}(\thet) =  \frac{1}{\n} \sum_{\x \in \D} \loss(\thet;\x) + \reg(\thet), \ \text{with}\ \reg(\thet) = \frac{\cvx\norm{\thet}^2}{2}.
\end{equation}
The \emph{excess empirical risk} of a model $\Thet$ on $\D$ is defined as 
\begin{equation}
	\err(\Thet;\D) = \expec{}{\rLoss_{\D}(\Thet) - \rLoss_{\D}(\thet_\D^*)},  
\end{equation}
where ${\thet_\D^* = \underset{\thet \in \domain}{\arg\min}\ \rLoss_{\D}(\thet)}$, and expectation is over $\Thet$. 

{\bf Problem Definition.} Let constants $\q > 1$, $0 < \epsdd \leq \epsdp$, and $\risk > 0$. Our goal in this paper is to design a learning mechanism $\Lrn: \X^\n \rightarrow \domain$ and a deletion mechanisms $\Unlrn: \X^\n \times \U^\reps \times \domain \rightarrow \domain$ for ERM such that 
\begingroup
\renewcommand\labelenumi{\bf(\theenumi.)}
\begin{enumerate}
	\item both $\Lrn$ and $\Unlrn$ satisfy $(\q, \epsdp)$-R\'enyi DP with respect to records in the input database,
	\item pair $(\Lrn, \Unlrn)$ satisfies $(\q, \epsdd)$-data-deletion guarantee for all non-adaptive $\reps$-requesters $\updreq$,
	\item and, all models $(\hat\Thet_i)_{i \geq 0}$ produced by $(\Lrn, \Unlrn, \updreq)$ on any ${\D_0 \in \X^\n}$ have $\err(\hat\Thet_i; \D_i) \leq \risk$. 
\end{enumerate}
\endgroup
Objectives {\bf(1.)} and {\bf(2.)} together ensure that $(\Lrn, \Unlrn)$ satisfy data-deletion for adaptive requests, and objective {\bf(3.)} ensures (un)learned models are useful\footnote{\footnotesize Constraint $\risk$ in {\bf(3.)} should be close to the optimal excess risk attainable by ERM on $\D_i$ under $(\q, \epsdp)$-R\'enyi DP.}.  

A deletion algorithm $\Unlrn$ is only useful if it's computationally cheaper than retraining with $\Lrn$. We judge the benefit of $\Unlrn$ over $\Lrn$ for $i^{\text{th}}$ request $\up_i$ by difference in retraining $\text{Cost}(\Lrn; \D_{i-1}\circ \up_i)$ and deletion ${\text{Cost}(\Unlrn; \D_{i-1}, \up_i, \hat\Thet_{i-1})}$.

% !TEX root = ../main.tex

\section{Deletion Using Noisy Gradient Descent}
\label{sec:noisggd}

% We build on a popular DP-ERM algorithm called Noisy-GD~\citep{abadi2016deep}, described in Algorithm~\ref{alg:noisygd} below, and provide a detailed discussion on its R\'enyi DP guarantees in Appendix~\ref{app:dp_guarantees}.
%
This section proposes a simple and effective data-deletion solution based on Noisy-GD~\citep{abadi2016deep,bassily2014private,chaudhuri2011differentially}, a popular privacy-preserving ERM mechanism described in Algorithm~\ref{alg:noisygd}. Appendix~\ref{app:dp_guarantees} provides its R\'enyi DP guarantees. 

\begin{algorithm}[htbp!]
	\caption{Noisy-GD: Noisy Gradient Descent}
        \label{alg:noisygd}
        \begin{algorithmic}[1]
                \REQUIRE Database $\D \in \X^\n$, model $\Thet \in \domain$, number of iterations $\K \in \N$.
		\STATE {Initialize $\Thet_0 = \Thet$}
		\FOR{$k = 0, 1, \cdots, \K - 1$}
			% \STATE {Sample batch $\batch_k \sim \Sam_{\bsize}(\D)$ of size $\bsize$ without replacement}
			\STATE {$\graD{\Loss_\D(\Thet_{\step k})} = \frac{1}{\n} \sum_{\x \in \D} \graD{\loss(\Thet_{\step k};\x)} + \graD{\reg(\Thet_{\step k}}) $}
			\STATE {$\Thet_{\step(k+1)} = \Thet_{\step k} - \step \graD{\Loss_\D(\Thet_{\step k})} + \sqrt{2\step}\Gaus{0}{\noise^2\Id}$} \label{alg:ngd:updatestep}
		\ENDFOR
		\STATE {\bf{return} $\Thet_{\step \K}$}
	\end{algorithmic}
\end{algorithm}

Our proposed approach falls under the Descent-to-Delete framework proposed by~\citet{neel2021descent}, wherein, after each deletion request $\up_i$, we run Noisy-GD starting from the previous model $\hat\Thet_{i-1}$ and perform a small number of gradient descent steps over records in the modified database $\D_i = \D_{i-1} \circ \up_i$; sufficient to erase information regarding deleted records in the subsequent model $\hat\Thet_i$. Our algorithms $(\Lrn_{\Nsgd}, \Unlrn_{\Nsgd})$ is defined as follows.
\begin{definition}[$\Nsgd$ based data-deletion solution]
	\label{dfn:noisy_gd_unlrn}
	Let $\K_\Lrn, \K_\Unlrn \in \N$ and $\rhO$ be a Gaussian weight initialization distribution in $\domain$. For any $\D \in \X^\n$, our learning algorithm $\Lrn_\Nsgd: \X^\n \rightarrow \domain$ is defined as
	\begin{equation}
		\Lrn_\Nsgd(\D) = \Nsgd(\D, \Thet, \K_\Lrn),
	\end{equation}
	where $\Thet \sim \rhO$. And, for any edit request $\up \in \U^\reps$ on database $\D \in \X^\n$ and any model $\Thet \in \domain$, our deletion algorithm $\Unlrn_\Nsgd: \X^\n \times \U^\reps \times \domain \rightarrow \domain$ is defined as
	\begin{equation}
		\Unlrn_\Nsgd(\D, \up, \Thet) = \Nsgd(\D \circ \up_i, \Thet, \K_\Unlrn).
	\end{equation}
	Our curator $(\Lrn_\Nsgd, \Unlrn_\Nsgd)$ with any initial database $\D_0 \in \X^\n$ interacts with any update requester $\updreq$ as described in Algorithm~\ref{alg:ngd:updatestep} with publish function $\publish(\thet) = \thet$.
\end{definition}
For this setup, our objective is to provide conditions under which the algorithm pair $(\Lrn_\Nsgd, \Unlrn_\Nsgd)$ satisfies objectives {\bf (1.)}, {\bf (2.)}, and {\bf (3.)} as stated in the problem definition and analyze the computational savings of using $\Unlrn_\Nsgd$ over $\Lrn_\Nsgd$ in terms of gradient complexity.

\subsection{Deletion and Utility Under Convexity}
\label{sec:deletion_convex}

We give the following set of guarantees for algorithm pair $(\Lrn_\Nsgd, \Unlrn_\Nsgd)$ when loss function $\loss(\thet;\x)$ is convex.
\begin{theorem}[Utility, privacy, deletion, and computation tradeoffs]
	\label{thm:unlearning_accuracy_convex}
	Let constants ${\cvx, \smh, \lip > 0}$, ${\q > 1}$, and ${0 < \epsdd \leq \epsdp}$. Define constant $\kappa = \frac{\cvx + \smh}{\cvx}$. Let the loss function $\loss(\thet;\x)$ be twice differentiable, convex, $\lip$-Lipschitz, and $\smh$-smooth, the regularizer be $\reg(\thet) = \frac{\cvx}{2}\norm{\thet}^2$. If the learning rate be $\step = \frac{1}{2(\cvx + \smh)}$, the gradient noise variance is ${\noise^2 = \frac{4\q\lip^2}{\cvx \epsdp\n^2}}$, and the weight initialization distribution is ${\rhO = \Gaus{0}{\frac{\noise^2}{\cvx(1 - \step\cvx/2)\Id}}}$, then 
% the number of iterations in $\Lrn_\Nsgd$ is ${\K_\Lrn \geq 2\kappa\log \left( \frac{\epsdp\n^2}{4\q\dime} \right)}$, and the number of iterations in $\Unlrn_\Nsgd$ is ${\K_\Unlrn \geq 2\kappa \log \max\{5\kappa, \frac{8\reps^2}{\q\dime}, \left(\frac{\epsdp}{\epsdd}\right)^2\}}$, 
%
\begingroup
\renewcommand\labelenumi{\bf(\theenumi.)}
\begin{enumerate}
	\item both $\Lrn_\Nsgd$ and $\Unlrn_\Nsgd$ are $(\q, \epsdp)$-R\'enyi DP for any $\K_\Lrn, \K_\Unlrn \geq 0$,
	\item pair $(\Lrn_\Nsgd, \Unlrn_\Nsgd)$ satisfies $(\q, \epsdd)$-data-deletion all non-adaptive $\reps$-requesters
		\begin{equation}
			\text{if} \quad \K_\Unlrn \geq 4\kappa \log \frac{\epsdp}{\epsdd},
		\end{equation}
	\item and all models in sequence $(\hat\Thet_i)_{i \geq 0}$ produced by interaction between $(\Lrn_\Nsgd, \Unlrn_\Nsgd)$ and $\updreq$ on any ${\D_0 \in \X^\n}$, where $\updreq$ is any $\reps$-requester, have an excess empirical risk $\err(\hat\Thet_i; \D_i) = O\left(\frac{\q\dime}{\epsdp\n^2}\right)$ if
		\begin{align}
				\K_\Lrn \geq 4\kappa\log \left( \frac{\epsdp\n^2}{4\q\dime} \right), \quad \text{and} \quad \K_\Unlrn \geq 4\kappa \log \max\left\{5\kappa, \frac{8\epsdp\reps^2}{\q\dime}\right\}.
		\end{align}
\end{enumerate}
\endgroup
\end{theorem}
%
% Proof of Theorem~\ref{thm:unlearning_accuracy_convex} can be found in Appendix~\ref{app:deletion_convex}. 

Our excess empirical risk upper bound in Theorem~\ref{thm:unlearning_accuracy_convex} matches the theoretical lower bound of $\Omega(\min\left\{1, \frac{\dime}{\eps^2\n^2}\right\})$ in \citet{bassily2014private} for the best attainable empirical risk of any $(\eps, \del)$-DP algorithms on Lipschitz, smooth, strongly-convex loss functions\footnote{\footnotesize Recall from Remark~\ref{rem:renyi_onesided} that $(\q,\epsdp)$-R\'enyi DP implies $(\eps, \del)$-DP for ${\q = 1 + \frac{2}{\eps} \log(1/\del)}$ and $\epsdp = \eps/2$. When $\eps = \Theta(\log(1/\del))$, one can evaluate that $\frac{\q}{\epsdp} = \Theta(\frac{\log(1/\del)}{\eps^2})$.}. Thus, our deletion algorithm $\Unlrn_\Nsgd$ incurs no additional loss in utility, yet saves substantial computation costs. Our deletion algorithm is stateless and offers a computation saving of $\Omega(\n\log \min\{\frac{\n}{\reps}, \n \sqrt{\frac{\epsdd}{\q\dime}}\})$ in gradient complexity per-request (i.e., $\n (\K_\Lrn - \K_\Unlrn)$) while guaranteeing privacy, adaptive deletion, and optimal utility. This saving is better than all existing unlearning algorithms in literature that we know of, and we present a detailed comparison in Table~\ref{tab:comparision}.

Also, observe that for satisfying $(\q, \epsdp)$-R\'enyi DP and $(\q, \epsdd)$-data-deletion for non-adaptive $\reps$-requesters, the number of iterations $\K_\Unlrn$ needed is independent of the size, $\reps$, of the deletion batch, depending solely on the ratio $\frac{\epsdd}{\epsdp}$. However, the number of iterations required for ensuring optimal utility with differential privacy grows with $\reps$. We highlight that when deletion batches are sufficiently small, i.e., $\reps \leq \sqrt{\frac{\q\dime}{\epsdd}}$, doing enough unlearning iterations for satisfying $(\q, \epsdd)$-data-deletion guarantee is also sufficient for ensuring optimal utility of unlearned model under $(\q, \epsdp)$-R\'enyi DP constraint. 

\begin{table*}[h!]
	\begin{center}
		% \centering
		% \small
		% \vspace{0.15cm}
		\begin{tabular}{llllll}
		% \hline
			\multicolumn{1}{c}{\textbf{Unlearning Algorithm}}    				& \multicolumn{1}{c}{\textbf{\begin{tabular}[c]{@{}c@{}}Requires secret \\ states?\end{tabular}}} 	& \multicolumn{1}{c}{\textbf{\begin{tabular}[c]{@{}c@{}}Compute savings\\ for $i$th edit\end{tabular}}} 			\\ \hline \\
			% Algorithm                                    & Require secret states?      & Computation savings for $i$th edit                                          \\[0.14cm] \hline
			Noisy-m-A-SGD~\citenumber{Thm. 1}{ullah2021machine}			& No			                          		& $\Omega\left(\sqrt{\dime}\left(1 - \frac{\sqrt{\dime}}{\n}\right)\right)$  		\\
			Perturbed-GD~\citenumber{Thm. 9}{neel2021descent}  			& Yes                         					& $\Omega\left(\n \log \left(\frac{\eps\n}{\sqrt\dime}\right)\right)$                 	\\  
			Perturbed-GD~\citenumber{Thm. 28}{neel2021descent}			& No                          					& $\Omega\left(\n \log \left(\frac{\eps\n}{\log^2(i \dime)\sqrt\dime}\right)\right)$  	\\
		Noisy-GD [Thm.~\ref{thm:unlearning_accuracy_convex}, Ours]			& No                        					& $\Omega\left(\n \log \min\left\{\n, \frac{\eps\n}{\sqrt\dime} \right\}\right)$    	\\ 
		\end{tabular}
		\vspace{-0.28cm}
	\end{center}
	\caption{\label{tab:comparision}Comparison of the computation savings in gradient complexity per edit request along with requirement of secret states with prior unlearning algorithms. Edit requests are non-adaptive and modify $\reps=1$ record in $\n$-sized databases. We assume the loss $\loss(\thet;\x)$ of models in $\domain$ to be convex, $1$-Lipschitz, and $O(1)$-smooth, and $\ltwo$ regularization constant to be $O(1)$. For a fair comparison, we require that each of them satisfy $(1 + \frac{2}{\eps}\log(1/\del), \frac{\eps}{2})$-data-deletion guarantee (which implies one-sided $(\eps, \del)$-unlearning (cf. Remark~\ref{rem:renyi_onesided} \&~\ref{rem:non_adap_independence})) and have the same excess empirical risk bound $\alpha = O(1)$.}
\end{table*}
\subsection{Deletion and Utility under Non-Convexity}
\label{sec:unlrn_nonconvex}

For non-convex loss function $\loss(\thet;\x)$, we provide the following set of guarantees for pair $(\Lrn_\Nsgd, \Unlrn_\Nsgd)$.
\begin{theorem}[Accuracy, privacy, deletion, and computation tradeoffs]
	\label{thm:deletion_accuracy_nonconvex}
	Let constants $\cvx, \smh, \lip$, $\noise^2, \step > 0$, constants $\q, \B > 1$, and constants $0 < \epsdd \leq \epsdp < \dime$. Let the loss function $\loss(\thet;\x)$ be $\frac{\noise^2\log(\B)}{4}$-bounded, $\lip$-Lipschitz and $\smh$-smooth, the regularizer be $\reg(\thet) = \frac{\cvx}{2}\norm{\thet}^2$, and the weight initialization distribution be $\rhO = \Gaus{0}{\frac{\noise^2}{\cvx}\Id}$. Then, 
	\begingroup
	\renewcommand\labelenumi{\bf(\theenumi.)}
	\begin{enumerate}
		\item both $\Lrn_\Nsgd$ and $\Unlrn_\Nsgd$ are $(\q, \epsdp)$-R\'enyi DP for any $\step \geq 0$ and any $\K_\Lrn, \K_\Unlrn \geq 0$
			\begin{equation}
				\label{eqn:ref_noise_main}
				\hspace{-0.2cm}
				\text{if} \quad \noise^2 \geq \frac{\q\lip^2}{\epsdp\n^2} \cdot \step \max\{\K_\Lrn, \K_\Unlrn\},
			\end{equation}
		\item pair $(\Lrn_\Nsgd, \Unlrn_\Nsgd)$ satisfy $(\q, \epsdd)$-data-deletion under all non-adaptive $\reps$-requesters for any $\noise^2 > 0$, if learning rate is $\step \leq \frac{\cvx \epsdd}{64\dime\q\B(\smh + \cvx)^2}$ and number of iterations satisfy
			\begin{align}
				\label{eqn:ref_iter_main}
				\hspace{-0.2cm}
					\K_\Lrn \geq \frac{2\B}{\cvx\step} \log \left(\frac{\q \log(\B)}{\epsdd}\right), \ \text{and} \ \K_\Unlrn \geq \K_\Lrn - \frac{2\B}{\cvx\step} \log \left(\frac{\log(\B)}{2\left(\epsdd + \frac{\reps}{\n} \log(\B)\right)}\right),
			\end{align}
		\item and all models in sequence $(\hat\Thet_i)_{i\geq0}$ output by $(\Lrn_\Nsgd, \Unlrn_\Nsgd, \updreq)$ on any $\D_0 \in \X^\n$, where $\updreq$ is an $\reps$-requester, satisfy $\err(\hat\Thet_i; \D_i) = \tilde O\left(\frac{\dime\q}{\epsdp\n^2} + \frac{1}{\n}\sqrt{\frac{\q\epsdd}{\epsdp}}\right)$ when inequalities in~\eqref{eqn:ref_noise_main} and~\eqref{eqn:ref_iter_main} are equalities.
	\end{enumerate}
	\endgroup
\end{theorem}
The R\'enyi DP result in {\bf(1.)} is a restatement of~\citet[Theorem 1]{abadi2016deep} (discussed further in Appendix~\ref{app:dp_guarantees}). Our deletion and utility results in {\bf(2.)} and {\bf(3.)} build on recent breakthroughs in rapid convergence guarantees of Noisy-GD under isoperimetry~\citep{vempala2019rapid,chewi2021analysis}. 

% To prove data deletion and utility, we establish a R\'enyi indistinguishability of generated models $\hat\Thet_i$ with the Gibbs distribution $\pI(\D) \propto \exp(-\Loss_\D/\noise^2)$, i.e., we show $\Ren{\q}{\hat\Thet_i}{\pI(\D_i)} \leq \epsdd$ for all $i\geq 0$. 
% We highlight that the techniques we develop for analyzing data deletion differ tremendously from those used to argue differential privacy; satisfying RDP requires a bounded gradient sensitivity, while we need bounded loss sensitivity for our data-deletion guarantee.

Under non-convexity, all prior works on deletion have focused on empirical analysis for utility. As far as we know, we are the first to provide utility guarantees in this setting. Moreover, our non-convex utility bound exceeds the optimal privacy-preserving utility under convexity by only a factor of $\tilde O\big(\frac{1}{\n}\sqrt{\frac{\q\epsdd}{\epsdp}}\big)$, which becomes small for large databases or small deletion to privacy budget ratio. 

Our result offers a strict computational benefit in using $\Unlrn_\Nsgd$ whenever the fraction of edited records in a single update request satisfies $\frac{\reps}{\n} \leq \frac{1}{2} - \frac{\epsdd}{\log \B}$. For instance, in the deletion regime where we want ${\epsdd = \log(\B)/4}$, relying on $\Unlrn_\Nsgd$ rather than retraining with $\Lrn_\Nsgd$ is $\Omega(\dime\n \log \frac{\n}{\reps})$ cheaper.

\begin{remark}
	Both Theorems~\ref{thm:unlearning_accuracy_convex} and~\ref{thm:deletion_accuracy_nonconvex} also hold when gradients $\graD{\loss(\thet;\x)}$ are clipped to $\lip$ instead of assuming $\lip$-Lipschitzness. Appendix~\ref{ssec:clipping} discusses how gradient clipping is compatible with other assumptions we make.
\end{remark}

\section{Conclusions}
\label{sec:conclusion}

We showed that current data deletion methods in literature are inadequate under both adaptive and non-adaptive requests, and proposed a new notion of data deletion that aligns with the "Right to be Forgotten." We also showed the importance of protecting the privacy of existing records in order to ensure privacy of deleted records for adaptive deletion requests, and provide a general reduction from adaptive to non-adaptive deletion guarantees under DP. Our results on Noisy-GD based deletion algorithm, for both convex and non-convex losses, show significant computation savings compared to retraining at no loss in utility.

\section*{Acknowledgements}

We would like to thank Martin Strobel and Hannah Brown for their feedback on earlier versions of this paper. We would also like to thank Reza Shokri for constructive remarks on presentation of ideas in the paper.

\bibliography{references}
\bibliographystyle{plainnat}

%%%%%%%%%%%%%%%%%%%%%%%%%%%%%%%%%%%%%%%%%%%%%%%%%%%%%%%%%%%%
\newpage
\appendix
\noptcrule
\part{Appendix}
\parttoc
\newpage

% !TEX root = ../main.tex

% !TEX root = ../main.tex

\section{Table of Notations}

\begin{table}[H]
	% \caption{Symbol reference.}
	% \label{tab:symbol}
	\begin{minipage}{\columnwidth}
	\small
	\begin{center}
	\begin{tabular}{ll}
		\toprule
		\textbf{Symbol}					& \textbf{Meaning}						\\
		\hline
		$\C$						& Arbitrary model parameter space.				\\
		$\Pub$						& Space of publishable objects.				\\
		$\dime, \R^\dime$				& Dimension of model parameters and $\dime$-dimensional Euclidean space. \\
		$\n$						& Database size.						\\
		$\X,\X^\n$				  	& Data universe and Domain of all datasets of size $\n$.	\\
		$\nU, \nU', \pI, \mU$				& Arbitrary distributions on $\C$ or on $\domain$.		\\
		$\updreq$					& An edit requester.		\\
		$\reps, \pubs$					& Integers representing the power of an adaptive requester. 	\\
		$\U, \U^\reps$					& Space of singular and batched replacement edits in $[\n] \times \X$.		\\
		$\up, \up_i, \Up_i$				& Arbitrary edit request, $i^{th}$ edit request in $\U^\reps$ and its random variable.		\\
		$\D, \D_i$					& An example database and database after $i^{th}$ update.	\\
		$\x, \y$					& Singular data records from universe $\X$.			\\
		$\step$						& Step size or learning rate in Noisy-GD.			\\
		$\noise^2$					& Variance scaling used in weight initialization distribution or gradient noise.\\
		$\loss(\thet;\x)$				& Twice continuously differentiable loss function on models in $\domain$.  	\\
		$\reg(\thet)$					& $\ltwo$ regularizer $\cvx\norm{\thet}^2/2$.			\\
		$\Loss(\thet), \Loss_\D(\thet)$			& Arbitrary optimization objective and an $\reg(\thet)$ regularized objective on $\D$ over $\loss(\thet;\x)$.  	\\
		% $\rLoss_{\D}(\thet)$	  			& Empirical surrogate risk that Noisy-SGD optimizes.		\\
		$\err(\Thet;\D)$				& Excess empirical risk of random model $\Thet$ over objective $\rLoss_\D$.	\\
		$\pI(\D)$					& An mapping from $\X^\n$ to distributions on $\domain$; sometimes distributions are Gibbs.	\\
		$\Lambda_\D$					& Normalization constant of the Gibbs distribution $\pI(\D)$.	\\
		$\pI^\up_i$					& A distribution independent of record deleted by request $\up$ on database $\D_{i-1}$.		\\
		% $\bsize, \Sam_{\bsize}$			& Batch size and Subsampling mechanism that samples $\bsize$ records with replacement.	\\
		% $\batch_k, \B_k, \vec\B_k$			& Batch random variable, $k^{th}$ step observation, and ovservation sequence.	\\
		% $g(\thet; \D)$				   	& Clipped average gradient for loss $\loss(\thet;\x)$ over database $\D$.	\\
		$\T_k$						& A map over $\domain$.						\\
		$\rhO$						& Weight initialization distribution for Noisy-GD.		\\
 		$\ve, \ve'$					& Vector fields on $\mathbb R^d$.				\\
 		$\thet^*_\D, \thet^*_{\D_i}$			& Risk minimizer for $\Loss_\D$ and $\Loss_{\D_i}$.		\\
 		$\q$						& Order of R\'enyi divergence.					\\
 		$\epsdp, \epsdd$				& Differential privacy budget and data-deletion budget in $\q$-R\'enyi divergence.		\\
		$\eps, \del$					& Parameters for DP-like indistinguishability. 			\\
		$\Lrn, \Lrn_\Nsgd$				& Learning algorithm and Noisy-GD based learning algorithm respectively.	\\
		$\Unlrn, \Unlrn_\Nsgd$				& Data-deletion algorithm and Noisy-GD based data-deletion algorithm respectively.		\\
		$\K_\Lrn, \K_\Unlrn$				& Number of learning and data-deletion iterations in Noisy-GD.	\\
		$k, t$						& Index of a Noisy-GD iteration and continuous time variable for tracing diffusions.		\\
		$\Thet_{\step k}, \Thet_{\step k}'$		& Parameters at iteration $k$ of Noisy-GD.		\\
		$\Thet_t, \Thet_t'$				& Parameters at time $t$ of tracing diffusion for Noisy-GD.	\\
		$\mU_t, \mU_t'$					& Probability density for $\Thet_t, \Thet_t'$.			\\
 		% $\mathbb{I}_d$					& $d$-dimensional identity matrix.				\\
 		$\Z, \Z_k, \Z_k'$				& Random variables taken from $\Gaus{0}{\Id}$.			\\
		$\dif{\Z_t}, \dif{\Z_t'}$			& Two independent Weiner process.				\\
 		$\cvx, \smh, \B, \lip$				& $\ltwo$ regularizer constant and smoothness, boundedness, and Lipschitzness constants.	\\
		$\Clip_\lip(\cdot)$				& Operator that clips vectors in $\domain$ to a magnitude of $\lip$.		\\
		$\Ren{\q}{\nU}{\nU'}, \Eren{\q}{\nU}{\nU'}$	& R\'enyi divergence and $\q^{\text{th}}$ moment of likelihood ratio r.v. between $\nU$ and $\nU'$.			\\
%		$\Eren{\q}{\nU}{\nU'}$				& $\q^{\text{th}}$ moment of likelihood ratio r.v. between $\nU$ and $\nU'$.	\\
		$\Fis{\nU}{\nU'}, \Gren{\q}{\nU}{\nU'}$		& Fisher and $\q$-R\'enyi Information of distribution of $\nU$ w.r.t $\nU'$.	\\
		$\Was{\nU}{\nU'}$				& Wasserstein distance between distribution $\nU$ and $\nU'$.	\\
		$\KL{\nU}{\nU'}$				& Kullback-Leibler divergence of distribution $\nU$ w.r.t. $\nU'$.		\\
		$P_t, \Gen, \Gen^*$				& Markov semigroup, its infinitesimal generator, and its Fokker-Planck operator. \\
		$\Ent_\pI(f^2)$			 		& Entropy of function $f^2$ under any arbitrary distribution $\pI$.		\\
		$\entropy(\cdot)$				& Differential entropy of a distribution. 			\\
		$\LS(\lsi)$					& Log-sobolev inequality with constant $\lsi$.			\\
		% $\prox_\Loss$					& Proxmial mapping for objective $\Loss$.			\\
		\bottomrule
	\end{tabular}
	\end{center}
	%\bigskip
	\end{minipage}
\end{table}

% !TEX root = ../main.tex

\section{Divergence Measures and Their Properties}
\label{ssec:appendix_indistinguishability}

Let $\Thet, \Thet' \in \C$ be two random variables with probability measures $\nU, \nU'$ respectively. We abuse the notations to denote respective probability densities with $\nU, \nU'$ as well.
We say that $\nU$ is absolutely continuous with respect to $\nU'$ (denoted by $\nU \ll \nU'$) if for all measurable sets $\Out \subset \C$, $\nU(\Out) = 0$ whenever $\nU'(\Out) = 0$. 
\begin{definition}[$(\eps, \del)$-indistinguishability~\citep{dwork2014algorithmic}]
	We say $\nU$ and $\nU'$ are $(\eps,\del)$-\emph{indistinguishable} if for all $\Out \subset \C$,
	\begin{equation}
		\prob{\Thet \sim \nU}{\Thet \in \Out} \leq e^\eps \prob{\Thet' \sim \nU'}{\Thet' \in \Out} + \del \quad \text{and} \quad \prob{\Thet' \sim \nU'}{\Thet' \in \Out} \leq e^\eps \prob{\Thet \sim \nU}{\Thet \in \Out} + \del.
	\end{equation}
\end{definition}
In this paper, we measure indistinguishability in terms of R\'enyi divergence.
\begin{definition}[R\'enyi divergence~\citep{renyi1961measures}]
	\emph{R\'enyi divergence} of $\nU$ w.r.t. $\nU'$ of order $\q > 1$ is defined as
	\begin{equation}
		\label{eqn:renyi_dfn}
		\Ren{\q}{\nU}{\nU'} = \frac{1}{\q - 1} \log \Eren{\q}{\nU}{\nU'}, \quad
		\text{where} \quad \Eren{\q}{\nU}{\nU'} = \expec{\thet \sim \nU'}{\left(\frac{\nU(\thet)}{\nU'(\thet)}\right)^\q},
	\end{equation}
	when $\nU$ is \emph{absolutely continuous} w.r.t. $\nU'$ (denoted as $\nU \ll \nU'$). If $\nU \not\ll \nU'$, we'll say $\Ren{\q}{\nU}{\nU'} = \infty$. We abuse the notation $\Ren{\q}{\Thet}{\Thet'}$ to denote divergence $\Ren{\q}{\nU}{\nU'}$ between the measures of $\Thet, \Thet'$.
\end{definition}
A bound on R\'enyi divergence implies a one-directional $(\eps,\del)$-indistinguishability as described below.
\begin{theorem}[Conversion theorem of R\'enyi divergence~{\citep[Proposition 3]{mironov2017renyi}}]
	Let $\q > 1$ and $\eps > 0$. If distributions $\nU, \nU'$ satisfy $\Ren{\q}{\nU}{\nU'} < \eps_0$, then for any $\Out \subset \C$, 
\begin{equation}
	\prob{\Thet \sim \nU}{\Thet \in \Out} \leq e^\eps \prob{\Thet' \sim \nU'}{\Thet' \in \Out} + \del,
\end{equation}
for $\eps = \eps_0 + \frac{\log 1/\del}{\q - 1}$ and any $0 < \del < 1$.
\end{theorem}
We use the following properties of R\'enyi divergence in some of our proofs.
\begin{theorem}[Mononicity of R\'enyi divergence~{\citep[Proposition 9]{mironov2017renyi}}]
	\label{thm:monotonicity}
	For $1 \leq \q_0 < \q$, and arbitrary probability measures $\nU$ and $\nU'$ over $\C$, $\Ren{\q_0}{\nU}{\nU'} \leq \Ren{\q}{\nU}{\nU'}$.
\end{theorem}
\begin{theorem}[R\'enyi composition~{\citep[Proposition 1]{mironov2017renyi}}]
	If $\Lrn_1, \cdots, \Lrn_k$ are randomized algorithms satisfying, respectively, $(\q, \eps_1)\text{-R\'enyi DP}, \cdots, (\q, \eps_k)\text{-R\'enyi DP}$ then their composed mechanism defined as $(\Lrn_1(\D), \cdots, \Lrn_k(\D))$ is $(\q, \eps_1 + \cdots + \eps_k)$-R\'enyi DP. Moreover, $i^{th}$ algorithm can be chosen on the basis of the outputs of algorithms $\Lrn_1, \cdots, \Lrn_{i-1}$.
\end{theorem}
\begin{theorem}[Weak triangle inequality of R\'enyi divergence~{\citep[Proposition 12]{mironov2017renyi}}]
	\label{thm:triangle_inequality}
	For any distribution $\rhO$ on $\C$, the R\'enyi divergence of $\nU$ w.r.t. $\nU'$ satisfies the following weak triangle inequality:
	\begin{equation}
		\label{eqn:triangle_inequality}
    		\Ren{\q}{\nU}{\nU'} \leq \Ren{\q}{\nU}{\rhO} +  \Ren{\infty}{\rhO}{\nU'}.
    	\end{equation}
\end{theorem}

Another popular notion of information divergence is the Kullback-Leibler divergence.
\begin{definition}[Kullback-Leibler divergence~\citep{kullback1951information}]
	\emph{Kullback-Leibler} (KL) divergence $\KL{\nU}{\nU'}$ of $\nU$ w.r.t. $\nU'$ is defined as
	\begin{equation}
		\KL{\nU}{\nU'} = \expec{\thet \sim \nU}{\log \frac{\nU(\thet)}{\nU'(\thet)}}.
	\end{equation}
\end{definition}
R\'enyi divergence generalizes Kullback-Leibler divergence as ${\lim_{\q \rightarrow 1} \Ren{\q}{\nU}{\nU'} = \KL{\nU}{\nU'}}$ ~\citep{van2014renyi}. Some other divergence notions that we rely on are the following.
\begin{definition}[Wasserstein distance~\citep{vaserstein1969markov}]
	\emph{Wasserstein distance} between $\nU$ and $\nU'$ is
	\begin{equation}
		\Was{\nU}{\nU'} = \underset{\Pi}{\inf} \expec{\Thet, \Thet' \sim \Pi}{\norm{\Thet - \Thet'}^2}^{\frac{1}{2}},
	\end{equation}
	where $\Pi$ is any joint distribution on $\C \times \C$ with $\nU$ and $\nU'$ as its marginal distributions.
\end{definition}

\begin{definition}[Relative Fisher information~\citep{otto2000generalization}]
	If $\nU \ll \nU'$ and $\frac{\nU}{\nU'}$ is differentiable, then \emph{relative Fisher information} of $\nU$ with respect to $\nU'$ is defined as
	\begin{equation}
		\label{eqn:fisher_info}
		\Fis{\nU}{\nU'} =  \expec{\thet \sim \nU}{\norm{\graD{\log \frac{\nU(\thet)}{\nU'(\thet)}}}^2}.
	\end{equation}
\end{definition}
\begin{definition}[Relative R\'enyi information~\citep{vempala2019rapid}]
	Let $\q > 1$. If $\nU \ll \nU'$ and $\frac{\nU}{\nU'}$ is differentiable, then \emph{relative R\'enyi information} of $\nU$ with respect to $\nU'$ is defined as
	\begin{equation}
		\label{eqn:renyi_info}
		\Gren{\q}{\nU}{\nU'} = \frac{4}{\q^2} \expec{\thet \sim \nU'}{\norm{\graD{\left(\frac{\nU(\thet)}{\nU'(\thet)}\right)^{\q/2}}}^2} =  \expec{\thet \sim \nU'}{\left(\frac{\nU(\thet)}{\nU'(\thet)}\right)^{\q - 2}\norm{\graD{\left(\frac{\nU(\thet)}{\nU'(\thet)}\right)}}^2}.
	\end{equation}
\end{definition}

\section{Proofs for Section~\ref{sec:deletion_vs_unlearn}}
\label{app:proofs_deletion_vs_unlearn}

\begin{reptheorem}{thm:adaptive_violation}
	There exists an algorithm pair $(\Lrn, \Unlrn)$ satisfying $(0, 0)$-adaptive-unlearning under publish function $\publish(\thet) = \thet$ such that by designing a $1$-adaptive $1$-requester $\updreq$, an adversary can infer the identity of a record deleted by edit $\up_i$, at any arbitrary step $i > 3$, with probability at-least $1 - (1/2)^{i-3}$ from a single post-edit release $\pub_i$, even with no access to $\updreq$'s transcript $(\pub_{<i}; \up_{<i})$.
\end{reptheorem}
\begin{proof}
Let data universe be $\X$, the internal state space $\C$, and publishable outcome space $\Pub$ be $\R$. Consider the task of releasing a sequence of medians using function $\median : \R^* \rightarrow \R$ in the online setting when the initial database $\D_0 \in \X^\n$ is being modified by some adaptive requester $\updreq$. Given a database $\D \in \X^\n$, our learning algorithm is defined as $\Lrn(\D) = \median(\D)$. For an arbitrary edit request $\up \in \U^\reps$, our unlearning algorithm is defined as $\Unlrn(\D, \up, \bullet) = \median(\D \circ \up)$. Let the publish function $\publish: \C \rightarrow \Pub$ be an identity function, i.e. $\publish(\thet) = \thet$. 

For any initial database $\D_0 \in \X^\n$ and an adaptive sequence $(\up_i)_{i\geq1}$ generated by any $\infty$-adaptive $1$-requester $\updreq$, note that
\begin{equation}
	\publish(\Unlrn(\D_{i-1}, \up_i, \bullet)) = \publish(\Lrn(\D_i)), \quad \text{for all} \ i \geq 1 \ \text{and any} \ \bullet \in \C.
\end{equation}
Therefore, $\Unlrn$ is a $(0,0)$-adaptive unlearning algorithm for $\Lrn$ under $\publish$.

Now suppose that $\n$ is odd and $\D_0$ consists of unique entries. W.L.O.G assume that the median record $\median(\D_0)$ is at index $\ind^{m}$ and its owner will be deleting it at step $i$ by sending a non-adaptive edit request $\up_i = \{\langle \ind^{m}, \y \rangle\}$ such that $\y \neq \median(\D_0)$. We design the following $1$-adaptive $1$-requester $\updreq$ that sends edit requests in the first $i-1$ steps to ensure with high probability that the published outcome at step $i$ remains the deleted record, i.e., $\median(\D_i) = \median(\D_0)$:
\begin{equation}
	\updreq(\pub_0, \up_1, \up_2, \cdots, \up_{j-1}) = \{\langle \ind_j, \pub_0 \rangle\} \quad \forall \ 1 \leq j < i,
\end{equation}
where $\ind_j$ is randomly sampled from ${[\n] \setminus \{\ind_1, \cdots, \ind_{j-1}\}}$ without replacement. Note that by the end of interaction, $\updreq$ replaces at-least $i-2$ unique records in $\D_0$ with $\pub_0 = \median(\D_0)$. If one of those original records was larger than $\median(\D_0)$ and another was smaller than $\median(\D_0)$, then it is guaranteed that $\median(\D_i) = \median(\D_0)$. Therefore, $\prob{}{\median(\D_i) = \median(\D_0)}$ is at-least
\begin{align*}
     &\prob{}{\exists \ind^l, \ind^u \in \{ \ind_1, \cdots, \ind_{i-1}\} \ \text{s.t.} \ \D_0[\ind^l] < \D_0[\ind^m] < \D_0[\ind^u]}  \\ 
     &\quad\quad\quad\quad\quad\quad \geq 1 - 2 \times {\lfloor \n \rfloor/2 \choose i-2} \bigg/ {\n \choose i - 2} \geq 1 - \left(\frac{1}{2}\right)^{i-3}.
\end{align*}
\end{proof}

\begin{reptheorem}{thm:unlrn_vs_forget}
For every $\eps > 0$, there exists a pair $(\Lrn, \Unlrn)$ of algorithms that satisfy $(\eps, 0)$-non-adaptive-unlearning under some publish function $\publish$ such that for all non-adaptive $1$-requesters $\updreq$, their exists an adversary that can correctly infer the identity of a record deleted at any arbitrary edit step $i \geq 1$ by observing only the post-edit releases $\pub_{\geq i}$.
\end{reptheorem}

\begin{proof}
	For a query $\query: \X \rightarrow \{0,1\}$, consider the task of learning the count over a database that is being edited online by a non-adaptive $1$-requester $\updreq$. Since $\updreq$ is non-adaptive by assumption, it is equivalent to the entire edit sequence $\{\up_i\}_{i \geq 1}$ being fixed before interaction. We design an algorithm pair $(\Lrn, \Unlrn)$ for this task with secret model space being $\C = \N^3$ and published outcome space being $\Pub = \R$, with the publish function being $\publish(\langle a, b, c \rangle) = a + b/c + \lap{\frac{1}{\eps}}$ (with the convention that $b/c = 0$ if $b=c=0$). At any step $i \geq 0$, our internal model $\hat\Thet_i = \langle \mathrm{cnt}_i, \mathrm{del}_i, i \rangle$ encodes the current count of $\query$ on database $\D_i$, the count of $\query$ on records previously deleted by $\up_{\leq i}$, and the current step index $i$. Our learning algorithm initializes the secret model as $\hat\Thet_0 = \Lrn(\D_0) = \langle \sum_{\x \in \D_0} \query(\x), 0, 0 \rangle$, and, for an edit request $\up_i = \{\langle \ind_i, \y_i \rangle\}$, our algorithm $\Unlrn$ updates the secret model $\hat\Thet_{i-1} \rightarrow \hat\Thet_i$ following the rule
\begin{align*}
	\hat\Thet_i = \Unlrn(\D_{i-1}, \up_i, \hat\Thet_{i-1}) = \langle\mathrm{cnt}_i, \mathrm{del}_i, i\rangle \ \text{where}\ 
	\begin{cases} 
		\mathrm{cnt}_i = \mathrm{cnt}_{i-1} + \query(\y_i) - \query(\D_{i-1}[\ind_i]), \\
		\mathrm{del}_i = \mathrm{del}_{i-1} + \query(\D_{i-1}[\ind_i]).
	\end{cases}
\end{align*}
Note that $\forall i \geq 1$, $\Delta_i \stackrel{\text{def}}{=} \mathrm{del}_i / i \in [0, 1]$. Therefore, from properties of Laplace mechanism~\citep{dwork2014algorithmic}, it is straightforward to see that for all $i \geq 1$,
\begin{align*}
	\publish(\Unlrn(\D_{i-1}, \up_i, \hat\Thet_{i-1})) \big| \up_{\leq i} &= \sum_{\x \in \D_i} \query(\x) + \Delta_i + \lap{\frac{1}{\eps}} \\
	&\approxDP{\eps}{0} \sum_{\x \in \D_i} \query(\x) + \lap{\frac{1}{\eps}} = \publish(\Lrn(\D_i)).
\end{align*}
Hence, $\Unlrn$ is an $(\eps, 0)$-unlearning algorithm for $\Lrn$ under $\publish$.

To show that an adversary can still infer the identity of record deleted by edit request $\up_i = (\ind_i, \bullet)$, consider a database $\D_{i-1}'$ that differs from $\D_{i-1}$ only at index $\ind_i$ such that $\query(\D_{i-1}'[\ind_i]) \neq \query(\D_{i-1}[\ind_i])$. Let random variable sequences $\pub_{\geq i}$ and $\pub_{\geq i}'$ denote the releases by $\Unlrn$ in the scenarios that the $(i-1)^{\text{th}}$ database was $\D_{i-1}$ and $\D_{i-1}'$ respectively. The divergence between these two random variable sequences reflect the capacity of any adversary to infer the record deleted by $\up_i$. Since, we have identical databases after $\up_i$, i.e. $\D_{j-1} \circ \up_j = \D_{j-1}' \circ \up_j$ for all $j \geq i$, note that both $\pub_j$ and $\pub_j'$ are independent Laplace distributions with a shift of exactly $\frac{1}{j}$ units. Therefore,
	\begin{align*}
		\max_{\Out \subset \Pub^*} \log \frac{\prob{}{\pub_{\geq i} \in \Out}}{\prob{}{\pub_{\geq i}' \in \Out}} &= \sum_{j=i}^{\infty} \max_{\Out_j \subset \R} \log \frac{\prob{}{\pub_j \in \Out_j}}{\prob{}{\pub_j' \in \Out_j}} 
	    =  \sum_{j=i}^{\infty} \log e^{\eps / j} = \infty. 
	    %&= \lim_{k \rightarrow \infty} \eps (\sum_{i=1}^{k-1} \frac{1}{i} - \sum_{i=1}^{j-1} \frac{1}{i}) \\
	    % \geq \lim_{k \rightarrow \infty} \eps (\log(\frac{k}{j}) - 1) \geq \infty.
	\end{align*}
	\vspace{-0.4cm}
\end{proof}

\subsection{Unsoundness and Incompleteness of Offline Unlearning Definitions}
\label{app:two_stage_unlearning}

In this subsection, we show that our criticisms on soundness and completeness of unlearning notions under adaptive requests in Section~\ref{sec:deletion_vs_unlearn} also apply to the following unlearning definition variants of \citet{guo2019certified, sekhari2021remember}.

\begin{definition}[$(\eps, \del)$-certified removal~\citep{guo2019certified}]
	\label{def:guo}
	A removal mechanism $\Unlrn$ performs \emph{$(\eps, \del)$-certified removal} for learning algorithm $\Lrn$ if for all databases $\D \subset \X$ and deletion subset $S \subset \D$,
\begin{equation}
	\Unlrn(\D, S, \Lrn(\D)) \stackrel{\eps,\del}{\approx} \Lrn(\D\setminus S).
\end{equation}
\end{definition}
\begin{definition}[$(\eps, \del)$-unlearning~\citep{sekhari2021remember}]
	\label{def:sekhari}
	For all $\D \subset \X$ of size $\n$ and deletion subset $S \subset \D$ such that $|S| \leq m$, a learning algorithm $\Lrn$ and an unlearning algorithm $\Unlrn$ is \emph{$(\eps, \del)$-unlearning} if
\begin{equation}
	\Unlrn(T(\D), S, \Lrn(\D)) \stackrel{\eps,\del}{\approx} \Unlrn(T(\D \setminus S), \varnothing, \Lrn(\D \setminus S)),
\end{equation}
where $\varnothing$ denotes the empty set and $T(\D)$ denotes the data statistics available to $\Unlrn$ about $\D$.
\end{definition}
%

% \begin{definition}[Data Deletion Operation~\citet{ginart2019making}]
% 	% \label{def:ginart}
% 	Fix any dataset $\D \subset \X$ and learning algorithm $\Lrn$. Operation $\Unlrn$ is a \emph{deletion operation} for $\Lrn$ if $\Unlrn(\D, S, \Unlrn(\D)) \stackrel{0,0}{\approx} \Lrn(\D \setminus S)$ for any set $S \subset \D$ selected \emph{independently} of $\Lrn(\D)$.
% \end{definition}

{\bf Unsoundness.} Unlike Definition~\ref{def:ginart}, Definitions~\ref{def:guo} and~\ref{def:sekhari} make no assumptions about dependence between the deletion request $S$ and the learned model $\Lrn(\D)$. So, request $S$ can depend on $\Lrn(\D)$. This dependence is common in the real world; for example, a user deletes her information if she doesn't like what model $\Lrn(\D)$ reveals about her. We recall the example we provide in Section~\ref{sec:deletion_vs_unlearn} to show that Definitions~\ref{def:guo} and~\ref{def:sekhari}, are unsound under adaptivity.
 
For the universe of records $\X = \{-2, -1, 1, 2\}$, consider the following learning and unlearning algorithms:
\begin{equation}
	\Lrn(\D) = \sum_{\x \in \D} \x, \quad \text{and} \quad \Unlrn(\D, S, \Lrn(\D)) = \sum_{\x \in \D \setminus S} \x.
\end{equation}
Note that for any $\D \subset \X$ and any $S \subset \D$, the above algorithm pair $(\Lrn, \Unlrn)$ satisfies Definitions~\ref{def:guo},~\ref{def:sekhari} and~\ref{def:ginart} for $\eps = \del = 0$ and $T(\D) = \D$. Suppose the adversary is aware that the following dependence holds between the learned model $\Lrn(\D)$ and deletion request $S$:
\begin{equation}
	S = \begin{cases} 
		\{\x < 0: \forall \x \in \X\}&\text{if} \ \Lrn(\D) < 0,\\
		\{\x > 0: \forall \x \in \X\}&\text{otherwise}.
	\end{cases}
\end{equation}
Consider two neighbouring databases $\D_{-1} = \{-2, -1, 2\}$ and $\D_{1} = \{-2, 1, 2\}$. Knowing the above dependence, an adversary can determine whether $\D = \D_{-1}$ or $\D = \D_{1}$ by looking only at $\Unlrn(\D, S, \Lrn(\D))$. This is because if $\D = \D_{-1}$, then the observation after unlearning is $2$, and if $\D = \D_{1}$, the observation after unlearning is $-2$. So, even though $(\Lrn, \Unlrn)$ satisfies the guarantees of \citet{guo2019certified} and \citet{sekhari2021remember}, it blatantly reveals the identity ($-1$ or $1$) of a deleted record to an adversary observing only the post-deletion release.

Note that \citet{ginart2019making}'s Definition~\ref{def:ginart} assumes that the requests $S$ is selected independently of the learned model $\Lrn(\D)$. So, our construction does not apply, keeping the possibility that their definition is sound. We remark, however, that algorithms satisfying their definitions cannot be trusted in settings where we expect some dependence between deletion requests and the learned models.

{\bf Incompleteness.} Definitions~\ref{def:guo} and~\ref{def:sekhari} are also incomplete. Consider an unlearning algorithm $\Unlrn$ that outputs a fixed output $\x_1 \in \X$ if the deletion request $S = \varnothing$ and outputs another fixed output $\x_2 \in \X$ if the deletion request $S \neq \varnothing$. It is easy to see that $\Unlrn$ is a valid deletion algorithm as its output does not depend on the input database $\D$ or the learned model $\Lrn(\D)$. However, note that $\Unlrn$ does not satisfy the unlearning Definition~\ref{def:sekhari}, for any learning algorithm $\Lrn$. And, for a learning algorithm $\Lrn(\D) = \sum_{\x \in \D} \x$, one can also verify that the pair $(\Lrn, \Unlrn)$ does not satisfy Definitions~\ref{def:guo} either.

\section{Proofs for Section~\ref{sec:deletion}}
\label{app:proofs_deletion}

\begin{reptheorem}{thm:soundness}[Data-deletion Definition~\ref{dfn:deletion} is sound]
	If the algorithm pair $(\Lrn, \Unlrn)$ satisfies $(\q, \eps)$-data-deletion guarantee under all $\pubs$-adaptive $\reps$-requesters, then even with the power of designing an $\pubs$-adaptive $\reps$-requester $\updreq$ that interacts with the curator before deletion of a target record at any step $i\geq1$, any adversary observing only the post-deletion releases $(\hat\Thet_i, \hat\Thet_{i+1}, \cdots)$ has its membership inference advantage for inferring a deleted target bounded as 
\begin{equation}
	\label{eqn:adv_sound_app}
	\text{Adv}(\text{MI}) \leq \min\left\{\sqrt{2\eps},\frac{\q e^{\eps(\q-1)/\q}}{\q - 1}[2(\q - 1)]^{1/\q} - 1\right\}.
\end{equation}
\end{reptheorem}
\begin{proof}
	For an arbitrary step $i\geq1$, suppose one of the replacement operations in the edit request $\up_i \in \U^\reps$ replaces a record at index `$\ind$' from the database $\D_{i-1}$ with `$\y$'. In the worst case, this record $\D_{i-1}[\ind]$ might have been there from the start, i.e. $\D_0[\ind] = \D_0[\ind]$, and influenced all the decisions of the adaptive requester $\updreq$ in the edit steps $1, \cdots, i-1$. To prove soundness, we need to show that if $(\Lrn, \Unlrn)$ satisfies $(\q, \eps)$-data-deletion, then even in this worst-case scenario, no adaptive adversary can design a membership inference test $\text{MI}(\hat\Thet_i, \hat\Thet_{i+1}, \cdots) \in \{0, 1\}$ that can distinguish with high probability the null hypothesis $H_0 = \{\D_0[\ind] = \x\}$ from the alternate hypothesis $H_1 = \{\D_0[\ind] = \x'\}$ for any $\x, \x' \in \X$. That is, the advantage of any test $\text{MI}$, defined as
\begin{equation}
	\text{Adv}(\text{MI}) \stackrel{\text{def}}{=} \prob{}{\text{MI}(\hat\Thet_i, \hat\Thet_{i+1}, \cdots) = 1 |H_0} - \prob{}{\text{MI}(\hat\Thet_i, \hat\Thet_{i+1}, \cdots) = 1 |H_1},
\end{equation}
must be small. Since after processing edit request $\up_i$, the databases $\D_i, \D_{i+1}, \cdots$ no longer contain the deleted record $\D_{i-1}[\ind]$, the data-processing inequality implies that future models $\hat\Thet_{i+1}, \hat\Thet_{i+2}, \cdots$ cannot have more information about $\D_{i-1}[\ind]$ that what is present in $\hat\Thet_i$. Therefore, any test $\text{MI}(\hat\Thet_i, \hat\Thet_{i+1}, \cdots)$ has a smaller advantage than the optimal test $\text{MI}^*(\hat\Thet_i) \in \{0,1\}$ that only uses $\hat\Thet_i$.

Also, since $(\Lrn, \Unlrn)$ satisfy $(\q, \eps)$-data-deletion for any $\pubs$-adaptive $\reps$-requester $\updreq$, we know from Definition~\ref{dfn:deletion} that there exists a mapping $\pI^\updreq_i$ such that for all $\D_0 \in \X^\n$, the model $\hat\Thet_i$ generated by the interaction between $(\Lrn, \Unlrn, \updreq)$ on $\D_0$ after $i$th edit satisfies the inequality $\Ren{\q}{\hat\Thet_i}{\pI^\updreq_i(\D_0\circ\langle \ind, \y \rangle)} \leq \eps$. As the database $\D_0 \circ \langle \ind, \y \rangle$ is identical under both hypothesis $H_0$ and $H_1$, we have $\Ren{\q}{\hat\Thet_i|H_b}{\bar\Thet} \leq \eps$ for $b \in \{0, 1\}$, where $\bar\Thet = \pI^\updreq_i(\D_0\circ\langle \ind, \y \rangle)$. From R\'enyi divergence to $(\eps, \del)$-indistinguishability conversion described in Remark~\ref{rem:renyi_onesided}, we get
\begin{align}
	\prob{}{\text{MI}^*(\hat\Thet_i) = 1|H_0} &\leq e^{\eps'(\del)} \prob{}{\text{MI}^*(\bar\Thet) = 1} + \del,\ \text{and} \\
	\prob{}{\text{MI}^*(\hat\Thet_i) = 0|H_1} &\leq e^{\eps'(\del)} \prob{}{\text{MI}^*(\bar\Thet) = 0} + \del,
\end{align}
where $\eps'(\del) = \eps + \frac{\log 1/\del}{\q - 1}$ for any $0 < \del < 1$. On adding the two inequalities, we get:
\begin{align*}
	\text{Adv}(\text{MI}) \leq \text{Adv}(\text{MI}^*) &= \prob{}{\text{MI}^*(\hat\Thet_i) = 1|H_0} - \prob{}{\text{MI}^*(\hat\Thet_i) = 1|H_1} \\
							   &\leq \min_{\del} e^{\eps'(\del)} - 1 + 2\del \\
							   &= \frac{\q e^{\eps(\q-1)/\q}}{\q - 1}[2(\q - 1)]^{1/\q} - 1
\end{align*}
Alternatively, from monotonicity of R\'enyi divergence w.r.t. order $\q$ and the fact that R\'enyi divergence converges to KL divergence as $\q \rightarrow 1$, we have from $\Ren{\q}{\hat\Thet_i|H_b}{\bar\Thet} \leq \eps$ for $b \in \{0, 1\}$ that
\begin{align*}
	&\KL{\hat\Thet_i|H_b}{\bar\Thet} \leq \Ren{\q}{\hat\Thet_i|H_b}{\bar\Thet} \leq \eps \\
	\implies&\TV{\hat\Thet_i|H_b;\bar\Thet} \leq \sqrt{\frac{\eps}{2}} \tag{From Pinkser inequality},
\end{align*}
for $b \in \{0, 1\}$. So, from triangle inequality on total variation distance, we have
\begin{equation}
	\TV{\hat\Thet_i|H_0;\hat\Thet_i|H_1} \leq \TV{\hat\Thet_i|H_0;\bar\Thet} + \TV{\hat\Thet_i|H_0;\bar\Thet} \leq \sqrt{2\eps}.
\end{equation}
So, advantage of any membership inference attack $\text{MI}$ must have an advantage satisfying
\begin{equation}
	\text{Adv}(\text{MI}) = \prob{}{\text{MI}(\hat\Thet_i) = 1|H_0} - \prob{}{\text{MI}(\hat\Thet_i) = 1|H_1} \leq \sqrt{2\eps}.
\end{equation}
\end{proof}

\begin{reptheorem}{thm:dp_necessary}[Privacy of remaining records is necessary for adaptive deletion]
	Let $\Test: \C \rightarrow \{0, 1\}$ be a membership inference test for $\Lrn$ to distinguish between neighbouring databases $\D, \D' \in \X^\n$. Similarly, let $\Testb: \C \rightarrow \{0, 1\}$ be a membership inference test for $\Unlrn$ to distinguish between $\bar\D, \bar\D' \in \X^\n$ that are neighbouring after applying edit $\bar\up \in \U^1$. If $\text{Adv}(\Test) > \del$ and $\text{Adv}(\Testb) > \del$, then the pair $(\Lrn, \Unlrn)$ cannot satisfy $(\q, \eps)$-data-deletion under $1$-adaptive $1$-requester for any
	\begin{equation}
		\label{eqn:dp_necessary_app}
		\eps < \max\left\{\frac{\del^4}{2}, \log(\q - 1)  + \frac{\q}{\q-1}\log\left(\frac{1+\del^2}{\q2^{1/\q}} \right)\right\}.
	\end{equation}
\end{reptheorem}
\begin{proof}
By assumption, we know that there exists tests $\Test, \Testb: \C \rightarrow \{0,1\}$ such that
\begin{equation}
	\text{Adv}(\Test) \stackrel{\mathrm{def}}{=} \prob{}{\Test(\Lrn(\D)) = 1} - \prob{}{\Test(\Lrn(\D')) = 1} > \del,
\end{equation}
and for all $\thet \in \C$,
\begin{equation}
	\text{Adv}(\Testb) \stackrel{\text{def}}{=} \prob{}{\Testb(\Unlrn(\bar\D,\bar\up,\thet)) = 1} - \prob{}{\Testb(\Unlrn(\bar\D',\bar\up,\thet)) = 1} > \del.
\end{equation}
Define $\Out' = \{\thet \in \C| \Test(\thet) = 1\}$ and $\Outb' = \{\thet \in \C| \Testb(\thet) = 1\}$. We have that the total variation distance between $\Lrn(\D)$ and $\Lrn(\D')$ is lower bounded as
\begin{align}
	\TV{\Lrn(\D);\Lrn(\D')} &= \sup_{\Out \subset \C} |\prob{}{\Lrn(\D) \in \Out} - \prob{}{\Lrn(\D') \in \Out}| \\
				&> \prob{}{\Lrn(\D) \in \Out'} - \prob{}{\Lrn(\D') \in \Out'} \\
				&= \prob{}{\Test(\Lrn(\D)) = 1} - \prob{}{\Test(\Lrn(\D')) = 1} > \del.
\end{align}
Similarly, we also have that for all $\thet \in \C$, the total variation distance between $\Unlrn(\bar\D, \bar\up, \thet)$ and $\Unlrn(\bar\D', \bar\up, \thet)$ is lower bounded as
\begin{align}
	\TV{\Unlrn(\bar\D, \bar\up, \thet);\Unlrn(\bar\D',\bar\up,\thet)} &= \sup_{\Out \subset \C} |\prob{}{\Unlrn(\bar\D, \bar\up,\thet) \in \Out} -  \prob{}{\Unlrn(\bar\D', \bar\up,\thet) \in \Out}| \\
									  &> \prob{}{\Unlrn(\bar\D, \bar\up,\thet) \in \Outb'} - \prob{}{\Unlrn(\bar\D', \bar\up,\thet) \in \Outb'} \\
									  &= \prob{}{\Testb(\Unlrn(\bar\D,\bar\up,\thet)) = 1} - \prob{}{\Testb(\Unlrn(\bar\D',\bar\up,\thet)) = 1} > \del.
\end{align}
%
% If $\Lrn$ is not $(0,\del)$-DP with respect to replacement of a single record, then there exists a pair of neighbouring databases $\D, \D'$ such that
% %
% \begin{equation}
% 	\TV{\Lrn(\D);\Lrn(\D')} = \sup_{\Out \in \C} |\prob{}{\Lrn(\D) \in \Out} - \prob{}{\Lrn(\D') \in \Out}| > \del.
% \end{equation}
% %
% Similarly, if $\Unlrn$ is not $(0, \del)$-DP with respect to replacement of a single record that is not being deleted, then there exists a pair of databases $\bar\D, \bar\D'$ and an edit request $\bar\up \in \U^1$ such that $\bar\D \circ \bar\up$ and $\bar\D' \circ \bar\up$ are neighbouring and for all $\thet \in \C$,
% %
% \begin{equation}
% 	\TV{\Unlrn(\bar\D, \bar\up, \thet);\Unlrn(\bar\D',\bar\up,\thet)}= \sup_{\Out \in \C} |\prob{}{\Unlrn(\bar\D, \bar\up,\thet) \in \Out} -  \prob{}{\Unlrn(\bar\D', \bar\up,\thet) \in \Out}| > \del.
% \end{equation}
% %
% Since $\mathbf{TV}$ distance is bounded from below in both cases, there exists tests $\Test, \Testb: \C \rightarrow \{0,1\}$ such that
% %
% \begin{equation}
% 	\text{Adv}(\Test;\D,\D') \stackrel{\text{def}}{=} \prob{}{\Test(\Lrn(\D)) = 1} - \prob{}{\Test(\Lrn(\D')) = 1} > \del,
% \end{equation}
% %
% and for all $\thet \in \C$,
% %
% \begin{equation}
% 	\text{Adv}(\Testb;\bar\D, \bar\D', \bar\up) \stackrel{\text{def}}{=} \prob{}{\Testb(\Unlrn(\bar\D,\bar\up,\thet)) = 1} - \prob{}{\Testb(\Unlrn(\bar\D',\bar\up,\thet)) = 1} > \del.
% \end{equation}
% %

Assume W.L.O.G. that $\bar\up$ replaces at index $\n$ and the edited databases $\bar\D \circ \up, \bar\D' \circ \up$ differs only at index $1$. Also assume that $\D, \D'$ differs at index $\n$. 

Recall from Definition~\ref{dfn:deletion} that satisfying $(\q, \eps)$-data-deletion under $1$-adaptive $1$-requesters requires existence of a map $\pI^\updreq_\n: \X^\n \rightarrow \C$ for each $\updreq$ such that for all $\D_0 \in \X^\n$,
\begin{equation}
	\label{eqn:desired_dd_bnd}
	\Ren{\q}{\Unlrn(\D_{\n-1}, \up_\n, \hat\Thet_{\n-1})}{\pI^\updreq_\n(\D_0 \circ \up_\n)} \leq \eps, 
\end{equation}

To prove the theorem statement, we show that for a starting database $\D_0 \in \{\D, \D'\}$ and an edit request $\up_\n = \bar\up$ that deletes the differing record in choices of $\D_0$ at edit step $\n$, there exists a $1$-adaptive $1$-requester $\updreq$ that sends adaptive edit requests $\up_1, \cdots, \up_{\n-1}$ in the first $\n-1$ steps such that no map $\pI^\updreq_\n$ exists that satisfies \eqref{eqn:desired_dd_bnd} for both choices of $\D_0$ when $\eps$ follows inequality~\eqref{eqn:dp_necessary_app}.

Consider the following construction of $1$-adaptive $1$-requester $\updreq$ that only observes the first model $\hat\Thet_0 = \Lrn(\D_0)$ and generates the edit requests $(\up_1, \cdots, \up_{\n-1})$ as follows:
\begin{equation}
	\updreq(\hat\Thet_0; \up_1, \up_2, \cdots, \up_{i-1}) = \begin{cases}
		\langle i, \bar\D[i]\rangle &\text{if} \ \Test(\hat\Thet_0) = 1, \\
		\langle i, \bar\D'[i]\rangle &\text{otherwise}.
	\end{cases}
\end{equation}
This requester $\updreq$ transforms any initial database $\D_0$ to $\D_{\n-1} = \bar\D$ if the outcome $\Test(\hat\Thet_0) = 1$, otherwise to $\D_{\n-1} = \bar\D'$. Consider an adversary that does not observe the interaction transcript $(\hat\Thet_{<n}; \up_{<\n})$, but is interested in identifying whether $\D_0$ was $\D$ or $\D'$. The adversary gets to observe only the output $\hat\Thet_\n = \Unlrn(\D_{\n-1}, \up_\n, \hat\Thet_{\n-1})$ generated after processing the edit request $\up_\n = \bar\up$. On this observation, the adversary runs the membership inference test $\text{MI}(\hat\Thet_\n) = \Testb(\hat\Thet_\n)$. The membership inference advantage of $\text{MI}$ is
\begin{align*}
	\text{Adv}(\text{MI};\D, \D') &\stackrel{\text{def}}{=} \prob{}{\text{MI}(\hat\Thet_\n) = 1| \D_0 = \D} - \prob{}{\text{MI}(\hat\Thet_\n) = 1| \D_0 = \D'} \\
				      &=\sum_{b\in \{0,1\}} \prob{}{\Testb(\hat\Thet_\n) = 1| \Test(\hat\Thet_0) = b}\times \prob{}{\Test(\hat\Thet_0) = b| \D_0 = \D} \\
				      &\quad - \sum_{b\in \{0,1\}} \prob{}{\Testb(\hat\Thet_\n) = 1| \Test(\hat\Thet_0) = b}\times \prob{}{\Test(\hat\Thet_0) = b| \D_0 = \D'} \\
				      &= \left(\prob{}{\Testb(\hat\Thet_\n) = 1| \D_{\n-1} = \bar\D} - \prob{}{\Testb(\hat\Thet_\n) = 1| \D_{\n-1} = \bar\D'}\right) \text{Adv}(\Test;\D,\D') \\
				      &= \text{Adv}(\Testb; \bar\D, \bar\D', \bar\up) \times \text{Adv}(\Test;\D,\D') > \del^2.
\end{align*}
So, from the contrapositive of our soundness Theorem~\ref{thm:soundness}, we have that $(\Lrn, \Unlrn)$ cannot be an $(\eps, \q)$-data-deletion algorithm for $\eps$ and $\q$ satisfying
\begin{align}
	&\del^2 > \min\left\{\sqrt{2\eps},\frac{\q e^{\eps(\q-1)/\q}}{\q-1}[2(\q - 1)]^{1/\q} - 1\right\} \\
	\iff& \eps < \max\left\{\frac{\del^4}{2}, \log(\q - 1)  + \frac{\q}{\q-1}\log\left(\frac{1+\del^2}{\q2^{1/\q}}\right)\right\}.
\end{align}
\end{proof}

\begin{reptheorem}{thm:reduction}[From adaptive to non-adaptive deletion]
	If an algorithm pair $(\Lrn, \Unlrn)$ satisfies $(\q, \epsdd)$-data-deletion under all non-adaptive $\reps$-requesters and is also $(\q, \epsdp)$-R\'enyi DP with respect to records not being deleted, then it also satisfies $(\q, \epsdd + \pubs \epsdp)$-data-deletion under all $\pubs$-adaptive $\reps$-requesters.
\end{reptheorem}
\begin{proof}
	To prove this theorem, we need to show that for any $\pubs$-adaptive $\reps$-requester $\updreq$, there exists a construction for a map $\pI^\updreq_i: \X^\n \rightarrow \C$ such that for all $\D_0 \in \X^\n$, the sequence of model $(\hat\Thet_i)_{i\geq0}$ generated by the interaction between $(\updreq, \Lrn, \Unlrn)$ on $\D_0$ satisfies the following inequality for all $i\geq1$:
	\begin{equation}
		\label{eqn:const_prof}
		\Ren{\q}{\Unlrn(\D_{i-1}, \up_i, \hat\Thet_{i-1})}{\pI^\updreq_i(\D_0 \circ \langle \ind, \y \rangle)} \leq \epsdd + \pubs \epsdp, \quad \text{for all}\ \up_i \in \U^\reps\ \text{and}\ \langle \ind, \y \rangle \in \up_i.
	\end{equation}
	Fix a database $\D_0 \in \X^\n$ and an edit request $\up_i \in \U^\reps$. Let $\D_0' \in \X^\n$ be a neighbouring database defined to be $\D'_0 = \D_0 \circ \langle \ind, \y\rangle$ for an arbitrary replacement operation $\langle \ind, \y \rangle \in \up_i$. Given any $\pubs$-adaptive $\reps$-requester $\updreq$, let $(\hat\Thet_i)_{i \geq 0}$ and $(\Up_i)_{i \geq 1}$ be the sequence of released model and edit request random variables generated on $\updreq$'s interaction with $(\Lrn, \Unlrn)$ with initial database as $\D_0$. Similarly, let $(\hat\Thet_i')_{i \geq 0}$ and $(\Up_i')_{i \geq 1}$  be the corresponding sequences generated due to the interaction among $(\updreq, \Lrn, \Unlrn)$ on $\D_0'$.

	Since $(\Lrn, \Unlrn)$ is assumed to satisfy $(\q, \epsdd)$-data-deletion guarantee under non-adaptive $\reps$-requesters, recall from Remark~\ref{rem:non_adap_independence} that there exists a mapping $\pI: \X^\n \rightarrow \C$ such that for any fixed edit sequence $\up_{\leq i} \stackrel{\text{def}}{=} (\up_1, \up_2, \cdots, \up_i)$, 
	\begin{align}
		&\Ren{\q}{\hat\Thet_i|_{\Up_{\leq i} = \up_{\leq i}}}{\pI(\D_0 \circ \up_{\leq i})} \leq \epsdd \\
		\implies&\Ren{\q}{\Unlrn(\D_0\circ\Up_{<i}, \up_i, \hat\Thet_i)|_{\Up_{<i} = \up_{<i}}}{\pI(\D_0 \circ \Up_{<i}' \circ \up_i)|_{\Up_{<i} = \up_{<i}}} \leq \epsdd. \label{eqn:non_adap_dp_adap_1}
	\end{align}
	Note that since the replacement operation $\langle \ind, \y \rangle$ is part of the edit request $\up_i$, we have $\D_0 \circ \Up_{< i}' \circ \up_i = \D_0' \circ \Up_{< i}' \circ \up_i$. Moreover, since the sequence $\Up_{<i}'$ of edit requests is generated by the interaction of $(\updreq, \Lrn, \Unlrn)$ on $\D_0' = \D_0 \circ \langle \ind, \up \rangle$ and the $i$th edit request $\up_i$ is fixed beforehand, we can define a valid construction of a map $\pI^\updreq_i:\X^\n \rightarrow \C$ as per Definition~\ref{dfn:deletion} as follows:
	\begin{equation}
		\pI^\updreq_i(\D_0 \circ \langle \ind, \y \rangle) = \pI(\D_0' \circ \Up_{<i}' \circ \up_i).
	\end{equation}
	For brevity, let $\hat\Thet_\up = \Unlrn(\D_0 \circ \Up_{<i}, \up_i, \hat\Thet_{i-1})$, and $\hat\Thet_\up' = \pI^\updreq_i(\D_0 \circ \langle \ind, \y \rangle)$. For this construction, we prove the requisite bound in \eqref{eqn:const_prof} as follows. 
	\begin{align*}
		\Ren{\q}{\hat\Thet_\up}{\hat\Thet_\up'} &\leq \Ren{\q}{(\hat\Thet_\up, \Up_{< i})}{(\hat\Thet_\up', \Up_{< i}')} \tag{Data processing inequality~\cite[Theorem 1]{van2014renyi}} \\
						 &= \frac{1}{\q - 1} \log \int_{\thet} \sum_{\up_{< i}} \frac{J(\thet, \up_{<i})^\q}{J'(\thet, \up_{<i})^{\q - 1}} \dif{\thet} \tag{$J$ \& $J'$ are joint PDFs of $(\hat\Thet_\up, \Up_{< i})$ \& $(\hat\Thet_\up', \Up_{< i}')$} \\
						 &= \frac{1}{\q - 1} \log \sum_{\up_{< i}} \frac{\prob{}{\Up_{< i} = \up_{< i}}^\q}{\prob{}{\Up_{< i}' = \up_{< i}}^{\q-1}}  \left\{\int_{\thet} \frac{p_{\hat\Thet_\up|\Up_{< i} = \up_{< i}}(\thet)^\q}{p_{\hat\Thet_\up'|\Up_{< i}' = \up_{< i}}(\thet)^{\q-1}} \dif{\thet} \right\} \\
						 &\leq \frac{1}{\q - 1} \log \sum_{\up_{< i}} \frac{\prob{}{\Up_{< i} = \up_{< i}}^\q}{\prob{}{\Up_{< i}' = \up_{< i}}^{\q-1}} \exp((\q - 1)\epsdd) \tag{From \eqref{eqn:non_adap_dp_adap_1}}\\
						 &= \epsdd + \Ren{\q}{\Up_{< i}}{\Up_{< i}'} \\
						 &\leq \epsdd + \Ren{\q}{\left(\hat\Thet_{s^1}, \cdots, \hat\Thet_{s^\pubs}\right)}{\left(\hat\Thet_{s^1}', \cdots, \hat\Thet_{s^\pubs}'\right)} \tag{If $\updreq$ sees outputs at steps $s^1,\cdots,s^\pubs$} \\
						 &\leq  \epsdd + \pubs \epsdp. \tag{Via R\'enyi composition}
	\end{align*}

\end{proof}

\subsection{Our Reduction Theorem~\ref{thm:reduction} versus \citet{gupta2021adaptive}'s Reduction}
\label{subsec:gupta_compare}

Adaptive unlearning guarantee in \citep[Definition 2.3]{gupta2021adaptive} is designed to ensure that no adaptive requester $\updreq$ can force the output distribution of the unlearning algorithm $\Unlrn(\D_{i-1}, \up_i, \hat\Theta_{i-1})$ to diverge substantially from that of retraining algorithm $\Lrn(\D_i)$ with high probability. Such an attack is possible in unlearning algorithms that rely on some persistent states that are only randomized once during initialization. For example, \citet{bourtoule2021machine}'s SISA unlearning algorithm randomly partitions the initial database $\D_0$ during setup and uses the same partitioning for processing edit requests, deleting records from respective shards on request. \citet{gupta2021adaptive} show that an adaptive update requester $\updreq$ can interactively send deletion requests $\up_1, \cdots, \up_i$ to SISA so that after some time, the partitioning of remaining records in $\D_i = \D_0 \circ \up_1 \cdots \up_i$ follows a pattern that is unlikely to occur on repartitioning of $\D_i$ if we execute $\Lrn(\D_i)$. 

They provide a general reduction~\citep[Theorem 3.1]{gupta2021adaptive} from adaptive to non-adaptive unlearning guarantee under differential privacy. Their reduction relies on DP with regards to a change in the description of learning/unlearning algorithm's internal randomness and not with regards to the standard replacement of records. DP with respect to internal description of randomness means that an adversary observing an unlearned model remains uncertain about persistent states like database partitioning in SISA during setup. So from a triangle inequality type argument, \citet{gupta2021adaptive} show that with DP with respect to learning/unlearning algorithms' coins along with a non-adaptive unlearning guarantee implies an adaptive unlearning guarantee.

Our work shows that satisfying adaptive unlearning definition of \citet{gupta2021adaptive} still does not guarantee data deletion. In Theorem~\ref{thm:adaptive_violation}, we demonstrate that there exists an algorithm pair $(\Lrn, \Unlrn)$ satisfying adaptive unlearning Definition~\ref{dfn:unlearning} (a strictly stronger version of \citep[Definition 2.3]{gupta2021adaptive}), but still causes blatant non-privacy of deleted records in post-deletion release. The vulnerability we identify occurs because an adaptive requester can learn the identity of any target record before it is deleted and re-encode it back in the curator's database by sending edit requests. Because of this, an adversary (who knows how the adaptive requester works but does not have access to the requester's interaction transcript) can extract the identity of the target record from the model released after processing the deletion request. In our work, we argue that a reliable (and necessary) way to prevent this attack is to make sure that no adaptive requester ever learns the identity of a target record from the pre-deletion model releases it has access to. Consequently, our reduction in Theorem~\ref{thm:reduction} from adaptive to non-adaptive requests relies on differential privacy with respect to the standard replacement of records instead.

% !TEX root = ../../main.tex

\section{Calculus Refresher}
\label{ssec:appendix_cal}

Given a twice continuously differentiable function $\Loss:\C \rightarrow \R$, where $\C$ is a closed subset of $\R^\dime$, its gradient
$\graD{\Loss}: \C \rightarrow \R^\dime$ is the vector of partial derivatives
\begin{equation}
    \label{eqn:grad_dfn}
    \graD{\Loss(\thet)} = \left(\doh{\Loss(\thet)}{\thet_1},\cdots,
    \doh{\Loss(\thet)}{\thet_2}\right).
\end{equation}
Its Hessian $\hesS{\Loss}:\C \rightarrow \R^{\dime \times \dime}$ is the matrix of second partial
derivatives
\begin{equation}
    \label{eqn:hess_dfn}
    \hesS{\Loss(\thet)} = \left(\doh{^2\Loss(\thet)}{\thet_i\thet_j}\right)_{1 \leq i,j \leq \dime}.
\end{equation}
Its Laplacian $\lapL{\Loss}:\C \rightarrow \R$ is the trace of its Hessian $\hesS{\Loss}$, i.e.,
\begin{equation}
    \label{eqn:lapl_defn}
    \lapL{\Loss(\thet)} = \tra{\hesS{\Loss(\thet)}}.
\end{equation}

Given a differentiable vector field $\ve = \left(\ve_1, \cdots, \ve_\dime\right) :\C \rightarrow \R^\dime$, its 
divergence $\divR{\ve}: \C \rightarrow \R$ is
\begin{equation}
    \label{eqn:divr_defn}
    \divR{\ve}(\thet) = \sum_{i=1}^\dime \doh{\ve_i(\thet)}{\thet_i}.
\end{equation}

Some identities that we would rely on:
\begin{enumerate}
    \item Divergence of gradient is the Laplacian, i.e.,
        \begin{equation}
            \label{eqn:divr_lapl_eq}
            \divR{\graD{\Loss}}(\thet) = \sum_{i=1}^\dime \doh{^2\Loss(\thet)}{\thet_i^2}
            = \lapL{\Loss(\thet)}.
        \end{equation}
    \item For any function $f: \C \rightarrow \R$ and a vector field $\ve: \C \rightarrow
        \R^\dime$ with sufficiently fast decay at the border of $\C$,
        \begin{equation}
            \label{eqn:divr_dotp_eq}
            \int_{\C} \dotP{\ve(\thet)}{\graD{f(\thet)}\dif{\thet}}
            = - \int_{\C} f(\thet)(\divR{\ve})(\thet)\dif{\thet}.
        \end{equation}
    \item For any two functions $f, g: \C \rightarrow \R$, out of which at least for one the
        gradient decays sufficiently fast at the border of $\C$, the following also holds.
        \begin{align}
            \label{eqn:lapl_dotp_eq}
            \int_{\C} f(\thet)\lapL{g(\thet)} \dif{\thet} 
            = - \int_{\C} \dotP{\graD{f(\thet)}}{\graD{g(\thet)}} \dif{\thet} = \int_{\C} g(\thet)\lapL{f(\thet)} \dif{\thet}.
        \end{align}
    \item Based on Young's inequality, for two vector fields $\ve_1, \ve_2: \C \rightarrow \R^d$,
        and any $a, b \in \R$ such that $ab =1$, the following inequality holds.
        \begin{equation}
            \label{eqn:young_ineq}
            \dotP{\ve_1}{\ve_2} (\thet) \leq \frac{1}{2a}\norm{\ve_1(\thet)}^2 
            + \frac{1}{2b}\norm{\ve_2(\thet)}^2.
        \end{equation}
\end{enumerate}
Wherever it is clear, we would drop $(\thet)$ for brevity. For example, we would represent
$\divR{\ve}(\thet)$ as only $\divR{\ve}$.

% !TEX root = ../main.tex

\section{Loss Function Properties}
\label{sec:loss_properties}

In this section, we provide the formal definition of various properties that we assume in the paper. Let $\loss(\thet;\x): \domain \times \X \rightarrow \R$ be a loss function on $\domain$ for any record $\x \in \X$.
\begin{definition}[Lipschitzness]
	A function $\loss(\thet;\x)$ is said to be \emph{$\lip$ Lipschitz continuous} if for all $\thet, \thet' \in \domain$ and any $\x \in \X$,
	\begin{equation}
		|\loss(\thet;\x) - \loss(\thet';\x)| \leq \lip \norm{\thet - \thet'}.
	\end{equation}
	If $\loss(\thet;\x)$ is differentiable, then it is $\lip$-Lipschitz if and only if ${\graD{\loss(\thet;\x)} \leq \lip}$ for all $\thet\in \domain$.
\end{definition}
\begin{definition}[Boundedness]
	A function $\loss(\thet;\x)$ is said to be \emph{$\B$-bounded} if for all $\x \in \X$, its output takes values in 
	range $[-\B, \B]$.
\end{definition}

\begin{definition}[Convexity]
	A continuous differential function $\loss(\thet;\x)$ is said to be \emph{convex} if for all 
	$\thet, \thet' \in \domain$ and $\x \in \X$,
	\begin{equation}
		\loss(\thet';\x) \geq \loss(\thet;\x) + \dotP{\graD{\loss(\thet;\x)}}{\thet' - \thet},
	\end{equation}
	and is said to be $\cvx$-strongly convex if
	\begin{equation}
		\loss(\thet';\x) \geq \loss(\thet;\x) + \dotP{\graD{\loss(\thet;\x)}}{\thet' - \thet} + \frac{\cvx}{2} \norm{\thet' - \thet}^2.
	\end{equation}
\end{definition}
\begin{theorem}[{\citep[Theorem 2.1.4]{nesterov2003introductory}}]
	A twice continuously differentiable function $\loss(\thet;\x)$ is \emph{convex} if and only if for all $\thet \in \domain$ and $\x \in \X$, its hessian matrix $\hesS{\loss(\thet;\x)}$ is positive semidefinite, i.e., $\hesS{\loss(\thet;\x)} \succcurlyeq 0$ and is $\cvx$-strongly convex if its hessian matrix satisfies $\hesS{\loss(\thet;\x)} \succcurlyeq \cvx \Id$.
\end{theorem}

\begin{definition}[Smoothness]
	A continuously differentiable function $\loss(\thet;\x)$ is said to be \emph{$\smh$-Smooth} if for all $\thet, \thet' \in \domain$ and $\x \in \X$,
	\begin{equation}
		\norm{\graD{\loss(\thet;\x)} - \graD{\loss(\thet';\x)}} \leq \smh \norm{\thet - \thet'}.
	\end{equation}
\end{definition}
\begin{theorem}[{\citep[Theorem 2.1.6]{nesterov2003introductory}}]
	A twice continuously differentiable convex function $\loss(\thet;\x)$ is $\smh$-smooth if and only if for all $\thet \in \domain$ and $\x \in \X$,
	\begin{equation}
		\hesS{\loss(\thet;\x)} \preccurlyeq \smh \Id.
	\end{equation}
\end{theorem}
%
% \begin{definition}[Subgaussian gradients]
% 	Let $1 \leq \bsize \leq \n$. We say gradients of a loss function $\loss(\thet;\x)$ are subgaussian  on $\D \in \X^\n$ 
% 	with variance $\noise^2$ if for all $\thet \in \domain$, the error ${\Omega(\thet; \D) = 
% 	\Loss_{\batch}(\thet) - \Loss_\D(\thet)}$ due to Poisson subsampling 
% 	$\batch \sim \Sam_{\bsize/\n}(\D)$ satisfies following the tail condition for some constant $\h \geq 0$:
% 	%
% 	\begin{equation}
% 		\expec{}{\exp\left(\norm{\Omega(\thet;\D)}^2 / \h^2 \right)} \leq \exp(1).
% 	\end{equation}
% 	%
% \end{definition}

\subsection{Effect of Gradient Clipping}
\label{ssec:clipping}

First order optimization methods on a continuously differentiable loss function $\loss(\thet;\x)$ over a database $\D \in \X^\n$ with gradient clipping ${\Clip_\lip(\ve) = \ve / \max \left(1, \frac{\norm{\ve}}{\lip}\right)}$ is equivalent to optimizing 
\begin{equation}
	\label{eqn:surrogate_objective_appendix}
	\rLoss_{\D}(\thet) =  \frac{1}{|\D|} \sum_{\x \in \D} \rloss(\thet;\x) + \reg(\thet),
\end{equation}
where $\rloss(\thet;\x)$ is a surrogate loss function that satisfies $\graD{\rloss(\thet;\x)} = \Clip_\lip(\graD{\loss(\thet;\x)})$. This surrogate loss function inherits convexity, boundedness, and smoothness properties of $\loss(\thet;\x)$, as shown below. 

\begin{lemma}[Gradient clipping retains convexity]
	\label{lem:clipping_retains_convexity}
	If $\loss(\thet;\x)$ is a twice continuously differentiable convex function for every $\x \in \domain$, then surrogate loss 
	$\rloss(\thet;\x)$ resulting from gradient clipping is also convex for every $\x \in \domain$.
\end{lemma}
\begin{proof}
	Note that the clip operation ${\Clip_\lip(\ve)}$ is a closed-form solution of the orthogonal projection onto a closed
	ball of radius $\lip$ and centered around origin, i.e.
	\begin{equation}
		{\Clip_\lip(\ve)} = \underset{\norm{\ve'} \leq \lip}{\arg\min}\ \norm{\ve - \ve'}.
	\end{equation}
	By properties of orthogonal projections on closed convex sets, for every $\ve, \ve' \in \domain$,
	\begin{equation}
		\dotP{\ve' - \Clip_\lip(\ve)}{\ve - \Clip_\lip(\ve)} \leq 0 \quad \text{if and only if}	\ \norm{\ve'} \leq \lip.
	\end{equation}
	Therefore, for any $\thet \in \domain$, and $\x \in \X$, we have
	\begin{align}
		&\dotP{\graD{\rloss(\thet + h\hat\ve;\x)} - \graD{\rloss(\thet;\x)}}{\graD{\loss(\thet;\x)} - \graD{\rloss(\thet;\x)}} \leq 0, \\
		&\dotP{\graD{\rloss(\thet;\x)} - \graD{\rloss(\thet + h\hat\ve;\x)}}{\graD{\loss(\thet + h\hat\ve;\x)} - \graD{\rloss(\thet + h\hat\ve;\x)}} \leq 0,
	\end{align}
	for all unit vectors $\hat\ve \in \domain$ and magnitude $h > 0$. For the directional derivative of vector field $\graD{\rloss(\thet;\x)}$ along $\hat\ve$, defined as $\graD{_{\hat\ve} \graD{\rloss(\thet;\x)}} = \lim_{h \rightarrow 0^+} \frac{\graD{\rloss(\thet + h\hat\ve;\x)} - \graD{\rloss(\thet;\x)}}{h}$, the above two inequalities imply
	\begin{equation}
		\dotP{\graD{_{\hat\ve} \graD{\rloss(\thet;\x)}}}{\graD{\loss(\thet;\x)} - \graD{\rloss(\thet;\x)}} = 0,
	\end{equation}
	for all $\hat\ve$. Therefore, when $\graD{\rloss(\thet;\x)} \neq \graD{\loss(\thet;\x)}$, we must have $\hesS{\rloss(\thet;\x)} = 0$. And, when $\graD{\loss(\thet;\x)} = \graD{\rloss(\thet;\x)}$, gradients aren't clipped, 
	which implies the rate of change of $\loss(\thet;\x)$ along any direction $\hat\ve$ is % $\hesS{\rloss(\thet;\x)} = \hesS{\loss(\thet;\x)} \succcurlyeq 0$.
	\begin{align*}
		\graD{_{\hat\ve} \cdot \graD{\rloss(\thet;\x)}} &= \lim_{h \rightarrow 0^+} \dotP{\frac{\graD{\rloss(\thet + h\hat\ve;\x)} - \graD{\loss(\thet;\x)}}{h}}{\hat\ve} \\
		&= \begin{cases} \hat\ve^\top\hesS{\loss(\thet;\x)}\hat\ve &\text{if}\ \exists h > 0 \ \text{s.t.}\ \graD{\rloss(\thet + h\hat\ve;\x)} = \graD{\loss(\thet + h\hat\ve;\x)} \\ 0 &\text{otherwise}\end{cases} \geq 0.
	\end{align*}
\end{proof}

\begin{lemma}[Gradient clipping retains boundedness]
	\label{lem:clipping_retains_boundedness}
	If $\loss(\thet;\x)$ is a continuously differentiable and $\B$-bounded function for every $\x \in \X$, then 
	the surrogate loss $\rloss(\thet;\x)$ resulting from gradient clipping is also $\B$-bounded.
\end{lemma}
\begin{proof}
	Since $\loss(\thet;\x)$ is continuously differentiable, its $\B$-boundedness implies path integral of 
	$\graD{\loss(\thet;\x)}$ along any curve between 
	$\thet, \thet' \in \domain$ is less than $2\B$. Since $\Clip_\lip(\cdot)$ operation clips the
	gradient magnitude, the path integral of $\graD{\rloss(\thet;\x)}$ is also less than $2\B$. That is, the maximum
	and minimum values that $\rloss(\thet;\x)$ takes differ no more than $2\B$. By adjusting the constant of 
	path integral, we can always ensure $\rloss(\thet;\x)$ takes values in range $[-\B, \B]$ without affecting
	first order optimization algorithms.
\end{proof}

\begin{lemma}[Gradient clipping retains smoothness]
	\label{lem:clipping_retains_smoothness}
	If $\loss(\thet;\x)$ is a continuously differentiable and $\smh$-smooth function for every $\x \in \domain$, then 
	surrogate loss $\rloss(\thet;\x)$ resulting from gradient clipping is also $\smh$-smooth for every $\x \in \domain$.
\end{lemma}
\begin{proof}
	Note that the gradient clipping operation is equivalent to 
	an orthogonal projection operation into a ball of radius $\lip$, i.e. ${\Clip_\lip(\ve) = {\arg\min}_{\ve'} 
	\{\norm{\ve' - \ve} : \ve \in \domain, \norm{\ve'} \leq \lip \}}$.
	Since orthogonal projection onto a closed convex set is a $1$-Lipschitz operation, for any $\thet, \thet' \in \domain$,
	\begin{equation}
		\norm{\graD{\rloss(\thet;\x)} - \graD{\rloss(\thet';\x)}} \leq \norm{\graD{\loss(\thet;\x)} - \graD{\loss(\thet';\x)}} \leq \smh \norm{\thet - \thet'}.
	\end{equation}
\end{proof}
Additionally, the surrogate loss $\rloss(\thet;\x)$ is twice differentiable almost everywhere if $\loss(\thet;\x)$ is smooth, which follows from the following Rademacher's Theorem.
\begin{theorem}[Rademacher's Theorem~\citep{nekvinda1988simple}]
	If $f: \R^\n \rightarrow \R^\n$ is Lipschitz continuous, then $f$ is differentiable almost everywhere in $\R^\n$.
\end{theorem}
All our results in Section~\ref{sec:noisggd} rely on the above four properties on losses and therefore apply with gradient clipping instead of the Lipschitzness assumption.

\section{Additional Preliminaries and Proofs for Section~\ref{sec:noisggd}}

\subsection{Langevin Diffusion and Markov Semigroups}
\label{sec:langevin_diffusion}

Langevin diffusion process on $\domain$ with noise variance $\noise^2$ under the influence of a potential ${\Loss: \domain \rightarrow \R}$ is characterized by the Stochastic Differential Equation (SDE)
\begin{equation}
	\label{eqn:langevin_sde}
	\dif{\Thet_t} = - \graD{\Loss}(\Thet_t) \dif{t} + \sqrt{2\noise^2}\dif{\Z_t},
\end{equation}
where $\dif{\Z_t} = \Z_{t+\dif{t}} - \Z_t \sim \sqrt{\dif{t}}\Gaus{0}{\Id}$ is the $\dime$-dimensional Weiner process.

We present some preliminaries on the diffusion theory used in our analysis. Let $\p_t(\thet_0, \thet_t)$ denote the 
probability density function describing the distribution of $\Thet_t$, on starting from ${\Thet_0 = \thet_0}$ at time $t=0$. 
For SDE \eqref{eqn:langevin_sde}, the associated Markov semigroup $\mathbf P$, is defined as a family of operators $(P_t)_{t\geq 0}$,
such that an operator $P_t$ sends any real-valued measurable function $f: \domain \rightarrow \R$ to 
\begin{equation}
	P_t f(\thet_0) = \expec{}{f(\Thet_t) | \Thet_0 = \thet_0} = \int f(\thet_t) \p_t(\thet_0, \thet_t) \dif{\thet_t}.
\end{equation}
The infinitesimal generator $\Gen \overset{\text{def}}{=} \lim_{s\rightarrow0} \frac{1}{s} \left[P_{t+s} - P_s\right]$ for this diffusion semigroup is 
\begin{equation}
	\label{eqn:generator_langevin}
	\Gen f = \noise^2 \lapL{f} - \dotP{\graD{\Loss}}{\graD{f}}.
\end{equation}
This generator $\Gen$, when applied on a function $f(\thet_t)$, gives the infinitesimal change in the value of a function $f$ when $\thet_t$ 
undergoes diffusion as per \eqref{eqn:langevin_sde} for $\dif{t}$ time. That is,
\begin{equation}
	\label{eqn:generator_on_markov_operator}
	\partial{_t P_t f(\thet_0)} = \int \partial{_t} \p_t(\thet_0, \thet_t) f(\thet_t) \dif{\thet_t} = \int \p_t(\thet_0, \thet_t) \Gen f(\thet_t) \dif{\thet_t}.
\end{equation}

The dual operator of $\Gen$ is the Fokker-Planck operator $\Gen^*$, which is defined as the adjoint of generator $\Gen$, in the 
sense that
\begin{equation}
% \int \p_t(\thet_0, \thet_t) \Gen f(\thet_t) \dif{\thet_t} = \int f(\thet_t) \Gen^* \p_t(\thet_0, \thet_t) \dif{\thet_t}.
\int f \Gen^* g \dif{\thet} = \int g \Gen f \dif{\thet},
\end{equation}
for all real-valued measurable functions $f, g: \domain \rightarrow \R$. Note from \eqref{eqn:generator_on_markov_operator} 
that, this operator provides an alternative way to represent the rate of change of function $f$ at time $t$:
\begin{equation}
	\partial{_t P_t f(\thet_0)} = \int f(\thet_t) \Gen^* \p_t(\thet_0, \thet_t) \dif{\thet_t}.
\end{equation}
To put it simply, Fokker-Planck operator gives the infinitesimal change in the distribution of $\Thet_t$ with respect to time. For 
the Langevin diffusion SDE~\eqref{eqn:langevin_sde}, the Fokker-Planck operator is the following:
\begin{equation}
	\label{eqn:fokker_planck_langevin}
	\partial{_t} \p_t(\thet) = \Gen^* \p_t(\thet) = \divR{\p_t(\thet) \graD{\Loss(\thet)}} + \noise^2 \lapL{\p_t(\thet)}.
\end{equation}

From this Fokker-Planck equation, one can verify that the stationary or invariant distribution $\pI$ of Langevin diffusion, which is the solution
of $\partial{_t} \p_t = 0$, follows the Gibbs distribution
\begin{equation}
	\label{eqn:gibbs}
	\pI(\thet) \propto e^{-\Loss(\thet)/\noise^2}.
\end{equation}
Since $\pI$ is the stationary distribution, note that for any measurable function $f: \domain \rightarrow \R$,
\begin{equation}
	\label{eqn:generator_zero_expectation}
	\expec{\pI}{\Gen f} = \int f \Gen^* \pI \dif{\thet} = 0.
\end{equation}

\subsection{Isoperimetric Inequalities and Their Properties}
\label{ssec:isoperi}

Convergence properties of various diffusion semigroups have been extensively analyzed in literature under certain isoperimetric
assumptions on the stationary distribution $\pI$~\citep{bakry2014analysis}. One such property of interest is the 
\emph{logarithmic Sobolev ($\LS$) inequality}~\citep{gross1975logarithmic}, which we define next. 

The \emph{carr\'e du champ} operator $\carre$ of a diffusion semigroup with invariant measure $\mU$ is defined using its infinitesimal 
generator $\Gen$ as
\begin{equation}
	\carre(f,g) = \frac{1}{2} \left[\Gen(fg) - f\Gen g - g \Gen f \right],
\end{equation}
for every $f, g \in \mathbb{L}^2(\mU)$. Carr\'e du champ operator represent fundamental properties of a Markov semigroup that affect
its convergence behaviour. One can verify that Langevin diffusion semigroup's carr\'e du champ operator (on differentiable $f, g$) is
\begin{equation}
	\label{eqn:carre_langevin}
	\carre(f,g) = \noise^2 \dotP{\graD{f}}{\graD{g}}.
\end{equation}
We use shorthand notation $\carre(f) = \carre(f, f) = \noise^2 \norm{\graD{f}}$.

\begin{definition}[Logarithmic Sobolev Inequality (see~{\citet[p. 24]{bakry2014analysis}})]
	\label{dfn:def_lsi}
	A distribution with probability density $\pI$ is said to satisfy a \emph{logarithmic Sobolev inequality} ($\LS(\lsi)$) 
	(with respect to $\carre$ in \eqref{eqn:carre_langevin}) if for all functions $f \in \mathbb{L}^2(\mU)$ with continuous
	derivatives $\graD{f}$, 
	\begin{equation}
		\label{eqn:def_lsi}
		\Ent_\pI(f^2) \leq \frac{1}{2\lsi} \int \frac{\carre(f^2)}{f^2} \pI \dif{\thet} = \frac{2\noise^2}{\lsi} \int \norm{\graD{f}}^2 \pI \dif{\thet},
	\end{equation}
	where entropy $\Ent_\pI$ is defined as
	\begin{equation}
		\Ent_\pI(f^2) = \expec{\pI}{f^2 \log f^2} - \expec{\pI}{f^2} \log \expec{\pI}{f^2}.
	\end{equation}
\end{definition}

Logarithmic Sobolev inequality is a very non-restrictive assumption and is satisfied by a large class of distributions.
The following well-known result show that Gaussians satisfy $\LS$ inequality.
\begin{lemma}[LS inequality of Gaussian distributions (see~{\citet[p. 258]{bakry2014analysis}})]
	\label{lem:lsi_gaussian}
	Let $\rhO$ be a Gaussian distribution on $\domain$ with covariance $\noise^2/\cvx$ (i.e., the Gibbs distribution 
	\eqref{eqn:gibbs} with $\Loss(\cdot)$ being the $\ltwo$ regularizer $\reg(\thet) = \frac{\cvx}{2} \norm{\thet}^2$). Then $\rhO$ satisfies $\LS(\cvx)$ tightly
	(with respect to $\carre$ in \eqref{eqn:carre_langevin}), i.e.
	\begin{equation}
		\Ent_\rhO(f^2) = \frac{2\noise^2}{\cvx} \int \norm{\graD{f}}^2 \rhO \dif{\thet}.
	\end{equation}
	Additionally, if $\mU$ is a distribution on $\domain$ that satisfy $\LS(\lsi)$, then the convolution 
	$\mU \conv \rhO$, defined as the distribution of $\Thet + \Z$ where $\Thet \sim \mU$ and $\Z \sim \pI$, satisfies $\LS$ inequality with constant $\left(\frac{1}{\lsi} + \frac{1}{\cvx} \right)^{-1}$.
\end{lemma}
\citet{bobkov2007isoperimetric} show that like Gaussians, all strongly log concave distributions (or more generally, 
log-concave distributions with finite second order moments) satisfy $\LS$ inequality (e.g. Gibbs distribution $\pI$ with any strongly 
convex $\Loss$). $\LS$ inequality is also satisfied under non-log-concavity too. For example, $\LS$ inequality is stable under 
Lipschitz maps, although such maps can destroy log-concavity. 
\begin{lemma}[LS inequality under Lipschitz maps (see~{\citet{ledoux2001concentration}})]
	\label{lem:lsi_lipschitz}
	If $\pI$ is a distribution on $\domain$ that satisfies $\LS(\lsi)$, then for any $\lip$-Lipschitz map 
	$\T: \domain \rightarrow \domain$, the pushforward distribution $\T_{\#\pI}$, representing the distribution of $T(\Thet)$ when 
	$\Thet \sim \pI$, satisfies $\LS(\lsi/\lip^2)$.
\end{lemma}
%
% That is, for any $\lip$-lipschitz map 
% $T: \domain \rightarrow \domain$, $\LS(\lsi)$ of $\pI$ implies $\LS(\lsi\lip)$ of the pushforward distribution 
% $T(\Thet)$ where $\Thet \sim \pI$. 
$\LS$ inequality is also stable under bounded perturbations to the distribution, as shown in the 
following lemma by~\citet{holley1986logarithmic}.
\begin{lemma}[$\LS$ inequality under bounded perturbations (see~\citet{holley1986logarithmic})]
	\label{lem:lsi_perturbation}
	If $\pI$ is the probability density of a distribution that satisfies $\LS(\lsi)$, then any probability distribution with
	density $\pI'$ such that $\frac{1}{\sqrt\B} \leq \frac{\pI(\thet)}{\pI'(\thet)} \leq \sqrt\B$ everywhere in $\domain$ for some 
	constant $\B > 1$ satisfies $\LS(\lsi/\B)$.
\end{lemma} 

Logarithmic Sobolev inequality is of interest to us due to its equivalence to the following inequalities on Kullback-Leibler 
and R\'enyi divergence.
\begin{lemma}[$\LS$ inequality in terms of KL divergence~{\citep{vempala2019rapid}}]
	\label{lem:kl_fis_lsi}
	The distribution $\pI$ satisfies $\LS(\lsi)$ inequality (with respect to $\carre$ in \eqref{eqn:carre_langevin}) if and 
	only if for all distributions $\mU$ on $\domain$ such that $\frac{\mU}{\pI} \in \mathbb{L}^2(\pI)$ with continuous 
	derivatives $\graD{\frac{\mU}{\pI}}$,
	\begin{equation}
		\label{eqn:kl_fis_lsi}
		\KL{\mU}{\pI} \leq \frac{\noise^2}{2\lsi} \Fis{\mU}{\pI}.
	\end{equation}
\end{lemma}
\begin{proof}
	Set $f^2$ in \eqref{eqn:def_lsi} to $\frac{\mU}{\pI}$ to obtain \eqref{eqn:kl_fis_lsi}. Alternatively, set
	$\mU = \frac{f^2\pI}{\expec{\pI}{f^2}}$ in \eqref{eqn:kl_fis_lsi} to obtain \eqref{eqn:def_lsi}.
\end{proof}

\begin{lemma}[Wasserstein distance bound under $\LS$ inequality~{\citep[Theorem 1]{otto2000generalization}}]
	\label{lem:ren_wasser_kl}
	If distribution $\pI$ satisfies $\LS(\lsi)$ inequality (with respect to $\carre$ in \eqref{eqn:carre_langevin}) 
	then for all distributions $\mU$ on $\domain$, 
	\begin{equation}
		\label{eqn:ren_wasser_kl}
		\Was{\mU}{\pI}^2 \leq \frac{2\noise^2}{\lsi} \KL{\mU}{\pI}.
	\end{equation}
\end{lemma}

\begin{lemma}[$\LS$ inequality in terms of Rényi Divergence~\citep{vempala2019rapid}]
	\label{lem:ren_info_lsi}
	The distribution $\pI$ satisfies $\LS(\lsi)$ inequality (with respect to $\carre$ in \eqref{eqn:carre_langevin}) if and 
	only if for all distributions $\mU$ on $\domain$ such that $\frac{\mU}{\pI} \in \mathbb{L}^2(\pI)$ with continuous 
	derivatives $\graD{\frac{\mU}{\pI}}$, and any $\q > 1$,
	\begin{equation}
	    \label{eqn:ren_info_lsi}
	    \Ren{\q}{\mU}{\pI} + \q(\q - 1) \partial{_\q \Ren{\q}{\mU}{\pI}}
	    \leq \frac{\q^2\noise^2}{2\lsi} \frac{\Gren{\q}{\mU}{\pI}}{\Eren{\q}{\mU}{\pI}}.
	\end{equation}
\end{lemma}

\begin{proof}
	For brevity, let the functions ${R(\q)=\Ren{\q}{\mU}{\pI}}$, ${E(\q)=\Eren{\q}{\mU}{\pI}}$, 
	and ${I(\q)=\Gren{\q}{\mU}{\pI}}$. Let function ${f^2(\thet)=\left(\frac{\mU(\thet)}{\pI(\theta)}\right)^\q}$. Then,
	\begin{align*}
		\expec{\pI}{f^2} = \expec{\pI}{\left(\frac{\mU}{\pI}\right)^{\q}} = E(\q),
		\tag{From \eqref{eqn:renyi_dfn}}
	\end{align*}
	and,
	\begin{align*}
		\expec{\pI}{f^2 \log f^2} = \expec{\pI}{\left(\frac{\mU}{\pI}\right)^{\q}
		                          \log \left(\frac{\mU}{\pI}\right)^{\q}} 
		                          = \q \partial_\q \expec{\pI}{\int_{\q}\left(\frac{\mU}{\pI}\right)^{\q}
		                          \log\left(\frac{\mU}{\pI}\right)\dif{\q}}
		                          = \q \partial_\q \expec{\pI}{\left(\frac{\mU}{\pI}\right)^{\q}}
		                          = \q \partial_\q E(\q).
		                          \tag{From Lebniz rule and \eqref{eqn:renyi_dfn}}
	\end{align*}
 	Moreover,
 	\begin{align*}
 		\expec{\pI}{\norm{\graD{f}}^2} = \expec{\pI}{\norm{\graD{\left(\frac{\mU}{\pI}\right)^\frac{\q}{2}}}^2}
 		= \frac{\q^2}{4} I(\q)
 		\tag{From \eqref{eqn:renyi_info}}
 	\end{align*}
 	On substituting \eqref{eqn:def_lsi} with the above equalities, we get:
 	\begin{align*}
		&\Ent_{\pI}(f^2) \leq \frac{2\noise^2}{\lsi} \expec{\pI}{\norm{\graD{f}}^2} \\
 	    \iff&\q\partial_\q E(\q) - E(\q) \log E(\q) \leq \frac{\q^2\noise^2}{2\lsi} I(\q) \\
 	    \iff&\q\partial_\q \log E(\q) - \log E(\q) \leq \frac{\q^2\noise^2}{2\lsi} \frac{I(\q)}{E(\q)} \\
 	    \iff&\q \partial_\q \left((\q-1) R(\q)\right) - (\q - 1)R(\q) \leq \frac{\q^2\noise^2}{2\lsi} \frac{I(\q)}{E(\q)}
 	        \tag{From \eqref{eqn:renyi_dfn}} \\
	    \iff& R(\q) + \q(\q-1) \partial_\q R(\q) \leq \frac{\q^2\noise^2}{2\lsi} \frac{I(\q)}{E(\q)}
 	\end{align*}
\end{proof}

% !TEX root = ../main.tex

\subsection{(R\'enyi) Differential Privacy Guarantees on Noisy-GD}
\label{app:dp_guarantees}

In this section, we recap the differential privacy bounds in literature for Noisy-GD Algorithm~\ref{alg:noisygd}.
\begin{theorem}[R\'enyi DP guarantee for Noisy-GD Algorithm~\ref{alg:noisygd}]
	\label{thm:abadi}
	If $\loss(\thet;\x)$ is $\lip$-Lipschitz, then Noisy-GD satisfies $(\q, \eps)$-R\'enyi DP with $\eps = \frac{\q\lip^2}{\noise^2\n^2} \cdot \step \K$.
\end{theorem}
\begin{proof}
	The $L_2$ sensitivity of gradient $\nabla\Loss_{\D}(\thet)\stackrel{\text{def}}{=}\frac{1}{\n}\sum_{\x \in \D} \nabla\loss(\thet;\x)+\nabla\reg(\thet)$ computed in step 2 of Algorithm~\ref{alg:noisygd} for neighboring databases in $\X^\n$ that differ in a single record is $\frac{2\lip}{\n}$ since $\loss(\thet;\x)$ is $\lip$-Lipschitz.

	Conditioned on observing the intermediate model $\Thet_{\step k}=\thet_k$ at step $k$, the next model $\Thet_{\step(k+1)}$ after the noisy gradient update is a Gaussian mechanism with noise variance $2\noise^2/\step$. So, for neighboring databases $\D, \D' \in \X^\n$, we have from the R\'enyi DP bound of Gaussian mechanisms proposed by~\citet[Proposition 7]{mironov2017renyi} that
\begin{equation}
	\Ren{\q}{\Thet_{\step(k+1)}\mid_{\Thet_{\step k}=\thet_k}}{\Thet_{\step(k+1)}'\mid_{\Thet_{\step k}'=\thet_k}}\leq\frac{\step \q\lip^2}{\n^2\noise^2},
\end{equation}
where $(\Thet_{\step k})_{0\leq k\leq \K}$ and $(\Thet_{\step k}')_{0\leq k\leq \K}$ are intermediate parameters in Algorithm~\ref{alg:noisygd} when run on databases $\D$ and $\D'$ respectively. Finally, from R\'enyi composition~\citet[Proposition 1]{mironov2017renyi}, we have
\begin{align*}
	\Ren{\q}{\Thet_{\step \K}}{\Thet_{\step \K}'} &\leq \Ren{\q}{(\Thet_0,\Thet_\step,\cdots,\Thet_{\step \K})}{(\Thet_0',\Thet_\step',\cdots,\Thet_{\step \K}')} \\
						      &\leq \sum_{k=0}^{\K-1} \Ren{\q}{\Thet_{\step(k+1)}\mid_{\Thet_{\step k}=\thet_k}}{\Thet_{\step(k+1)}'\mid_{\Thet_{\step k}'=\thet_k}} \\
						      &\leq \frac{\q\lip^2}{\n^2\noise^2}\cdot\eta \K.
\end{align*}
\end{proof}
\begin{remark}
	Different papers discussing Noisy-GD variants adopt different notational conventions for the total noise added to the gradients. The noise variance in our Algorithm~\ref{alg:noisygd} is $2\step\noise^2$; but is $\frac{\step^2\noise^2\lip^2}{\n^2}$ in the full-batch setting of DP-SGD by~\citet{abadi2016deep}. To translate the bound in Theorem~\ref{thm:abadi}, one can simply rescale $\noise$ across different conventions to have the same noise variance, i.e., $2\step\noise^2=\frac{\step^2\hat\noise^2\lip^2}{\n^2}$.
\end{remark}

	Our Theorem~\ref{thm:abadi} is somewhat identical to \citet{abadi2016deep}'s $(\eps,\del)$-DP bound. To verify this, note from R\'enyi divergence to $(\eps, \del)$-indistinguishability conversion discussed in Remark~\ref{rem:renyi_onesided} that $(1+\frac{2}{\eps}\log\frac{1}{\del},\frac{\eps}{2})$-R\'enyi DP implies $(\eps,\del)$-DP. So, setting the bound in Theorem~\ref{thm:abadi} to be smaller than $\frac{\eps}{2}$ and substituting $\q=1+\frac{2}{\eps}\log\frac{1}{\del}$, we get
	$$
	\left(\frac{\eps+2\log\frac{1}{\del}}{\eps}\right)\frac{\lip^2}{\n^2\noise^2}\cdot\eta \K\leq\frac{\eps}{2}\iff\frac{\sqrt{\K(\eps+2\log\frac{1}{\del}})}{\eps}\leq\hat\noise. 
	$$
	For $\eps\leq2\log\frac{1}{\del}$, we get the same noise bound as in \citet[Theorem 1]{abadi2016deep} for their (full-batch) DP-SGD algorithm.

	Next, we recap the tighter R\'enyi DP guarantee of \citet{chourasia2021differential} under stronger assumptions on the loss function.
\begin{theorem}[R\'enyi DP guarantee for Noisy-GD Algorithm~\ref{alg:noisygd}~\citep{chourasia2021differential}]
	\label{thm:chourasia}
	If $\loss(\thet;\x)$ is convex, $\lip$-Lipschitz, and $\smh$-smooth and $\reg(\thet)$ is the $\ltwo$ regularizer with constant $\cvx$, then Noisy-GD with learning rate $\step < \frac{1}{\smh + \cvx}$ satisfies $(\q, \eps)$-R\'enyi DP with ${\eps = \frac{4\q\lip^2}{\cvx\noise^2\n^2}\left(1 - e^{-\cvx\step\K/2}\right)}$.
\end{theorem}

\subsection{Proofs for Subsection~\ref{sec:deletion_convex}}
\label{app:deletion_convex}

In this appendix, we provide a proof of our Theorem~\ref{thm:unlearning_accuracy_convex} which applies to convex losses $\loss(\thet;\x)$ under $\ltwo$ regularizer $\reg(\thet)$. Let $\D_0 \in \X^\n$ be any arbitrary database, and $\updreq$ be any non-adaptive $\reps$-requester. 

Our first goal in this section is to prove $(\q, \epsdd)$-data-deletion guarantees on our proposed algorithm pair $(\Lrn_\Nsgd, \Unlrn_\Nsgd)$ (in Definition~\ref{dfn:noisy_gd_unlrn}) under $\updreq$. That is, if $(\hat\Thet_i)_{i \geq 0}$ is the sequence of models produced by the interaction between $(\Lrn_\Nsgd, \Unlrn_\Nsgd, \updreq)$ on $\D_0$, we need to show that their exists a mapping $\pI^\updreq_i$ such that for all $i \geq 1$ and any $\up_i \in \U^\reps$, 
\begin{equation}
	\Ren{\q}{\Unlrn(\D_{i-1}, \up_i, \hat\Thet_{i-1})}{\pI^\updreq_i(\D_0 \circ \langle \ind, \y \rangle)} \leq \epsdd \quad\text{for all}\ \langle \ind, \y \rangle \in \up_i.
\end{equation}

For an arbitrary replacement operation $\langle \ind, \y \rangle$ in $\up_i$, we define a map $\pI^\updreq_i(\D_0 \circ \langle \ind, \y \rangle) = \hat\Thet_i'$, where the model sequence $(\hat\Thet_i')_{i\geq0}$ is produced by the interaction of between algorithms $(\Lrn_\Nsgd, \Unlrn_\Nsgd, \updreq)$ on initial database $\D_0 \circ \langle \ind, \y \rangle$. Since non-adaptive requester $\updreq$ is equivalent to fixing the edit sequence $(\up_i)_{i\geq1}$ a-priori, note that showing the data-deletion guarantee reduces to proving the following DP-like bound
\begin{equation}
	\label{eqn:dp_bound_eqiv}
	\Ren{\q}{\Unlrn(\D_{i-1}, \up_i, \hat\Thet_{i-1})}{\Unlrn(\D_{i-1}', \up_i, \hat\Thet_{i-1}')} \leq \epsdd,
\end{equation}
for for all $\up_{\leq i}$ and for all neighbouring databases $\D_0, \D_0'$ s.t. $\D_0' = \D_0 \circ \langle \ind, \y \rangle$ with $\langle \ind, \y \rangle \in \up_i$. 

% Without loss in generality, suppose the record that $\up = \langle \ind_\up, \y_\up \rangle$ deletes in database $\D_{i-1}$ existed in the start database $\D_0$ (i.e. $\ind_\up \not\in \up_{j}$ for all $1 \leq j \leq i-1$). For a neighbouring database $\D_0' \in \X^\n$, such that $\D_0 \circ \up = \D_0' \circ \up$, suppose $(\hat\Thet_i')_{i \geq 0}$ is the sequence of models produced by the interaction between $(\Lrn_\Nsgd, \Unlrn_\Nsgd, \updreq)$ on $\D_0'$. Note that, since $\D_0'$ never contained the record that $\up$ is erasing in $\D_{i-1}$, and since the non-adaptive requester $\updreq$ is oblivious to the sequence of model releases (which is equivalent to saying that the edit sequence $(\up_i)_{i \geq 1}$ is fixed before the interaction occurs), $\hat\Thet_{i-1}'$ is independent of the deleted record $\D_{i-1}[\ind_u]$. As such, we define $\pI^\up_i$ to be the distribution of $\Unlrn(\D_{i-1}', \up, \hat\Thet_{i-1}')$. So, showing the data-deletion guarantee reduces to proving the following DP-like bound
% %
% \begin{equation}
% 	\label{eqn:dp_bound_eqiv}
% 	\Ren{\q}{\Unlrn(\D_{i-1}, \up, \hat\Thet_{i-1})}{\Unlrn(\D_{i-1}', \up, \hat\Thet_{i-1}')} \leq \epsdd,
% \end{equation}
% %
% for all neighbouring $\D_0, \D_0'$ s.t. $\D_0 \circ \up = \D_0' \circ \up$, and for all $\up_{< i}$ that does not edit index $\ind_u$\footnote{\footnotesize If a request in $\up_j$ for $j < i$ edits the record at index $\ind_u$, we can construct a similar neighbouring database $\D_j'$ after the edit for our analysis.}. 

Note from our Definition~\ref{dfn:noisy_gd_unlrn} that the sequence of models $(\hat\Thet_0, \cdots, \hat\Thet_{i})$ can be seen as being generated from a continuous run of $\Nsgd$, where:
\begin{enumerate}
	\item for iterations $0 \leq k < \K_\Lrn$, the loss function is $\Loss_{\D_0}$, 
	\item for the iterations $\K_\Lrn + (j - 1)\K_\Unlrn \leq k < \K_\Lrn + j\K_\Unlrn$ on any $1 \leq j \leq i-1$, the loss function is $\Loss_{\D_j}$, and
	\item for the iterations $\K_\Lrn + (i - 1)\K_\Unlrn \leq k < \K_\Lrn + i \K_\Unlrn$, the loss function is $\Loss_{\D_{i-1} \circ \up_i}$.
\end{enumerate}
Let $(\Thet_{\step k})_{0 \leq k \leq \K_\Lrn + i \K_\Unlrn}$ be the sequence representing the intermediate parameters of this extended $\Nsgd$ run. Similarly, let $(\Thet_{\step k}')_{k \geq 0}$ be the parameter sequence corresponding to the extended run on the neighbouring database $\D_0'$. Since $\langle \ind, \y \rangle \in \up_i$, note from the construction that $\D_{i-1}' \circ \up_i = \D_{i-1} \circ \up_i$, meaning that the loss functions while processing request $\up_i$ is identical for the two processes, i.e. $\Loss_{\D_{i-1} \circ \up_i} = \Loss_{\D_{i-1}' \circ \up_i}$. For brevity, we refer to the database seen in iteration $k$ of the two respective extended runs as $\D(k)$ and $\D'(k)$ respectively. 
% Note that both $\Thet_0$ and $\Thet_0'$ are sampled from same the weight initialization distribution $\rhO$, which we define to be $\Gaus{0}{\frac{2\noise^2}{\cvx}\Id}$. 
In short, these two discrete processes induced by $\Nsgd$ follow the following update rule for any ${0 \leq k < \K_\Lrn + i \K_\Unlrn}$:
\begin{align}
	\label{eqn:gradient_update}
	\begin{cases}
		\Thet_{\step(k+1)} = \Thet_{\step k} - \step \graD{\Loss_{\D(k)}(\Thet_{\step k})} + \sqrt{2\step\noise^2}\Z_k \\
		\Thet_{\step(k+1)}' = \Thet_{\step k}' - \step \graD{\Loss_{\D'(k)}(\Thet_{\step k}')} + \sqrt{2\step\noise^2}\Z_k',
	\end{cases}
	\ \text{where}\ \Z_k, \Z_k' \sim \Gaus{0}{\Id},
\end{align}
and $\Thet_0$ and $\Thet_0'$ are sampled from same the weight initialization distribution $\rhO$.
To prove the bound in \eqref{eqn:dp_bound_eqiv}, we follow the approach proposed in \citet{chourasia2021differential} of interpolating the two discrete stochastic process of $\Nsgd$ with two piecewise-continuous tracing diffusions $\Thet_t$ and $\Thet_t'$ in the duration ${\step k < t \leq \step(k + 1)}$, defined as follows. 
\begin{align}
	\label{eqn:tracing_diffusion_deletion}
	\begin{cases}
		\Thet_t = \T_k(\Thet_{\step k}) - \frac{(t - \step k)}{2} \graD{\left(\Loss_{\D(k)}(\Thet_{\step k}) - \Loss_{\D'(k)}(\Thet_{\step k})\right)} + \sqrt{2\noise^2} (\Z_t - \Z_{\step k}), \\
		\Thet_t' = \T_k(\Thet_{\step k}') + \frac{(t - \step k)}{2} \graD{\left(\Loss_{\D(k)}(\Thet_{\step k}') - \Loss_{\D'(k)}(\Thet_{\step k}')\right)} + \sqrt{2\noise^2} (\Z_t' - \Z_{\step k}'),
	\end{cases}
\end{align}
where $\Z_t, \Z_t'$ are two independent Weiner processes, and $\T_k$ is a map on $\domain$ defined as
\begin{equation}
	\label{eqn:T}
	\T_k = \Id - \frac{\step}{2} \graD{\left(\Loss_{\D(k)} + \Loss_{\D'(k)} \right)}.
\end{equation}
Note that equation \eqref{eqn:tracing_diffusion_deletion} is identical to \eqref{eqn:gradient_update} when $t = \step(k+1)$, and
can be expressed by the following stochastic differential equations (SDEs):
\begin{equation}
	\label{eqn:sde_convex}
	\begin{cases}
		\dif{\Thet_t} = - \grad_k(\Thet_{\step k})\dif{t} + \sqrt{2\noise^2}  \dif{\Z_t} \\
		\dif{\Thet_t'} = + \grad_k(\Thet_{\step k}')\dif{t} + \sqrt{2\noise^2} \dif{\Z_t'},
	\end{cases}
	\text{where}\ \grad_k(\Thet) = \frac{1}{2\n}\graD{\left[\loss(\Thet;\D(k)[\ind]) - \loss(\Thet;\D'(k)[\ind])\right]},
\end{equation}
and initial condition $\lim_{t\rightarrow \step k^+}\Thet_t = \T_k(\Thet_{\step k})$, $\lim_{t\rightarrow \step k^+}\Thet_t' = \T_k(\Thet_{\step k}')$. These two SDEs can be equivalently described by the following pair of Fokker-Planck equations.
\begin{lemma}[Fokker-Planck equation for SDE~{\eqref{eqn:sde_convex}}]
	\label{lem:fokker_planck_deletion}
	Fokker-Planck equation for SDE in \eqref{eqn:sde_convex} at time $\step k < t \leq \step (k+1)$, is
	\begin{equation}
		\label{eqn:fokker_planck_deletion}
		\begin{cases}
		\partial_t \mU_{t}(\thet) &= \divR{\mU_{t}(\thet)\expec{}{\grad_k(\Thet_{\step k}) \middle\vert \Thet_t = \thet}} + \noise^2\lapL{\mU_{t}(\thet)}, \\
		\partial_t \mU_{t}'(\thet) &= \divR{\mU_{t}'(\thet)\expec{}{-\grad_k(\Thet_{\step k}') \middle\vert \Thet_t' = \thet}} + \noise^2\lapL{\mU_{t}'(\thet)},
		\end{cases}
	\end{equation}
	where $\mU_{t}$ and $\mU_{t}'$ are the densities of $\Thet_t$ and $\Thet_t'$ respectively.
\end{lemma}
\begin{proof}
	Conditioned on observing parameter $\Thet_{\step k} = \thet_{\step k}$, the process $(\Thet_t)_{\step k < t \leq \step (k+1)}$ is a Langevin diffusion along a constant Vector field (i.e. on conditioning, we get a Langevin SDE~\eqref{eqn:langevin_sde} with $\graD{\Loss(\thet)} = \grad_k(\thet_{\step k})$ for all $\thet \in \domain$). Therefore as per~\eqref{eqn:fokker_planck_langevin}, the conditional probability density $\mU_{t| \step k}(\cdot | \thet_{\step k})$ of $\Thet_t$ given $\Thet_{\step k}$ follows the following Fokker-Planck equation:
	\begin{equation}
		\partial_t \mU_{t|\step k}(\cdot | \thet_{\step k}) = \divR{\mU_{t|\step k}(\cdot | \thet_{\step k})\grad_k(\thet_{\step k})} + \noise^2 \lapL{\mU_{t|\step k}(\cdot | \thet_{\step k})}
	\end{equation}
	Taking expectation over $\mU_{\step k}$ which is the distribution of $\Thet_{\step k}$,
	\begin{align*}
		\partial_t \mU_{t}(\cdot) &= \int \mU_{\step k}(\thet_{\step k}) 
		\left\{\divR{\mU_{t|\step k}(\cdot | \thet_{\step k}) \grad_k(\thet_{\step k})}  + \noise^2 \lapL{\mU_{t|\step k}(\cdot| \thet_{\step k})} \right\} \dif{\thet_{\step k}} \\
					  &= \divR{ \int \grad_k(\thet_{\step k}) \mU_{t, \step k}(\cdot, \thet_{\step k}) \dif{\thet_{\step k}} } + \noise^2 \lapL{\mU_{t}(\cdot)} \\
					  &= \divR{ \mU_{t}(\cdot)\left\{\int \grad_k(\thet_{\step k}) \mU_{\step k| t}(\thet_{\step k}| \cdot) \dif{\thet_{\step k}} \right\}} + \noise^2 \lapL{\mU_{t}(\cdot)} \\
					  &= \divR{ \mU_{t}(\cdot) \expec{}{\grad_k(\Thet_{\step k}) | \Thet_t = \cdot}} + \noise^2 \lapL{\mU_{t}(\cdot)}. \\
	\end{align*}
	where $\mU_{\step k,| t}$ is the conditional density of $\Thet_{\step k}$ given $\Thet_t$. Proof for second Fokker-Planck equation is similar.
\end{proof}
%

%%%%%%%%%%%%%%%%%%%%%%%%%%%%%%%%%%%%%
% Input the figure
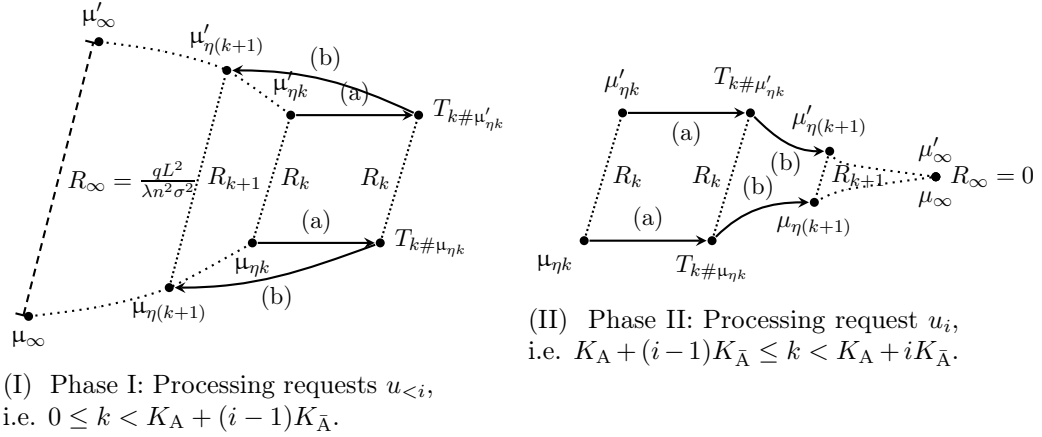
\begin{figure}
	\centering
	\begin{subfigure}{.37\textwidth}
		\begin{tikzpicture}[thick,scale=0.85, every node/.style={scale=0.85}]
			\node[fill,circle,minimum size=4pt, inner sep=0pt, label=below:{$\mU_{\step k}$}] (0) at (0,0) {};
			\node[fill,circle,minimum size=4pt, inner sep=0pt, label=above:{$\mU_{\step k}'$}] (2) at (0.6,2) {};
			\node[fill,circle,minimum size=4pt, inner sep=0pt, label=right:{$\T_{k\#\mU_{\step k}}$}] (1) at (2,0) {};
			\node[fill,circle,minimum size=4pt, inner sep=0pt, label=right:{$\T_{k\#\mU_{\step k}'}$}] (3) at (2.6,2) {};
			\node[fill,circle,minimum size=4pt, inner sep=0pt, label=above:{$\mU_{\step(k+1)}'$}] (4) at (-0.4,2.7) {};
			\node[fill,circle,minimum size=0pt, inner sep=0pt] (5) at (-1.4,3) {};
			% \node[fill,circle,minimum size=4pt, inner sep=0pt, label=above:{$\mU_{\step(k+3)'}$}] (6) at (-2.4,3.15) {};
			\node[fill,circle,minimum size=4pt, inner sep=0pt, label=above:{$\mU_{\infty}'$}] (7) at (-2.4,3.15) {};
			\node[fill,circle,minimum size=4pt, inner sep=0pt, label=below:{$\mU_{\step(k+1)}$}] (8) at (-1.3,-0.7) {};
			\node[fill,circle,minimum size=0pt, inner sep=0pt] (9) at (-2.4,-1) {};
			% \node[fill,circle,minimum size=4pt, inner sep=0pt, label=below:{$\mU_{\step(k+3)}$}] (10) at (-3.5,-1.15) {};
			\node[fill,circle,minimum size=4pt, inner sep=0pt, label=below:{$\mU_{\infty}$}] (11) at (-3.5,-1.15) {};

			\draw[thick,-stealth] (0) to node[above] {{(a)}} (1) ;
			\draw[thick,-stealth] (2) to node[above] {{(a)}} (3);
			\draw[thick,-stealth] (3) to [out=155,in=0] node[above] {{(b)}} (4);
			\draw[thick,-stealth] (1) to [out=200,in=0] node[below] {{(b)}}  (8);
			\draw[thick,densely dotted] (0) to node[right] {{$R_k$}} (2) {};
			\draw[thick,densely dotted] (1) to node[left] {{$R_k$}} (3) {};
			\draw[thick,densely dotted] (8) to node[right] {$R_{k+1}$} (4) {};
			% \draw[thick,densely dotted] (9) to node[right] {$R_{k+2}$} (5) {};
			% \draw[thick,densely dotted] (10) to node[right] {$R_{k+3}$} (6) {};
			\draw[thick,|-|,densely dashed] (11.west) to node[right] {$R_{\infty} = \frac{\q\lip^2}{\cvx\n^2\noise^2}$} (7.west) {};

			\draw[thick,dotted] plot [smooth] coordinates {(2)  (4)  (5)  (7)};
			\draw[thick,dotted] plot [smooth] coordinates {(0)  (8)  (9)  (11)};
		\end{tikzpicture}
		\caption{\label{subfig:convex_deletion_phase1} Phase I: Processing requests $\up_{< i}$, i.e. $0 \leq k < \K_\Lrn + (i-1) \K_\Unlrn$.}
	\end{subfigure}
	\hspace{1cm}
	\begin{subfigure}{.37\textwidth}
		\begin{tikzpicture}[thick,scale=0.85, every node/.style={scale=0.85}]
			\node[fill,circle,minimum size=4pt, inner sep=0pt, label=250:{$\mU_{\step k}$}] (0) at (0,0) {};
			\node[fill,circle,minimum size=4pt, inner sep=0pt, label=below:{$\T_{k\#\mU_{\step k}}$}] (1) at (2,0) {};
			\node[fill,circle,minimum size=4pt, inner sep=0pt, label=above:{$\mu_{\step k}'$}] (2) at (0.6,2) {};
			\node[fill,circle,minimum size=4pt, inner sep=0pt, label=above:{$\T_{k\#\mu_{\step k}'}$}] (3) at (2.6,2) {};
			\node[fill,circle,minimum size=4pt, inner sep=0pt, label=above:{$\mu_{\step(k+1)}'$}] (4) at (3.84,1.4) {};
			\node[fill,circle,minimum size=4pt, inner sep=0pt, label=below:{$\mu_{\step(k+1)}$}] (5) at (3.6,0.6) {};
			\node[fill,circle,minimum size=4pt, inner sep=0pt, label=above:{$\mu_{\infty}'$}] (6) at (5.5,1) {};
			\node[fill,circle,minimum size=4pt, inner sep=0pt, label=below:{$\mu_{\infty}$}] (7) at (5.5,1) {};
			\node[fill,circle,minimum size=4pt, inner sep=0pt, label=right:{$R_\infty=0$}] (7) at (5.5,1) {};
			\node[fill,circle,minimum size=0pt, inner sep=0pt] (8) at (4.1,1.2) {};
			\node[fill,circle,minimum size=0pt, inner sep=0pt] (9) at (4.1,0.8) {};

			\draw[thick,-stealth] (0) to node[above] {{(a)}} (1) ;
			\draw[thick,-stealth] (2) to node[below] {{(a)}} (3);
			\draw[thick,-stealth] (3) to [out=-45,in=180] node[below] {{(b)}} (4);
 			\draw[thick,-stealth] (1) to [out=45,in=180] node[above] {{(b)}}  (5);
 			\draw[thick,densely dotted] (0) to node[right] {{$R_k$}} (2) {};
			\draw[thick,densely dotted] (1) to node[left] {{$R_k$}} (3) {};
			\draw[thick,densely dotted] (4) to node[right] {$R_{k+1}$} (5) {};

			\draw[thick,dotted] plot [smooth] coordinates {(4) (8) (6)};
			\draw[thick,dotted] plot [smooth] coordinates {(5) (9) (7)};
		\end{tikzpicture}
		\caption{\label{subfig:convex_deletion_phase2} Phase II: Processing request $\up_i$, i.e. $\K_\Lrn + (i-1)\K_\Unlrn \leq k < \K_\Lrn + i \K_\Unlrn$.}
	\end{subfigure}
	\caption{\label{fig:convex_proof} Diagram illustrating the technical overview of Theorem~\ref{thm:data_deletion_convex}. Here $\mU_{\step k}$ and $\mU_{\step k'}$ represent the $k$th iteration parameter distribution of $\Thet_{\step k}$ and $\Thet_{\step k}'$ respectively. We interpolate the two discrete processes in two steps: (a) an identical transformation $\T_k$ (as defined in \eqref{eqn:T}, and (b) a diffusion process. If divergence before descent step is $R_k = \Ren{\q}{\mU_{\step k}}{\mU_{\step k}'}$, the stochastic mapping $\T_k$ in (a) doesn't increase the divergence, while the diffusion (b) either increases it upto an asymptotic constant in phase I or decreases it exponentially to $0$ in phase II.}

\end{figure}

We provide an overview of how we bound equation~\eqref{eqn:dp_bound_eqiv} in Figure~\ref{fig:convex_proof}. Basically, our analysis has two phases; in phase (I) we provide a bound on $\Ren{\q}{\hat\Thet_{i-1}}{\hat\Thet_{i-1}'}$ that holds for any choice of number of iterations $\K_\Lrn$ and $\K_\Unlrn$, and in phase (II) we prove an exponential contraction in the divergence $\Ren{\q}{\Unlrn(\D_{i-1}, \up_i, \hat\Thet_{i-1}}{\Unlrn(\D_{i-1}', \up_i, \hat\Thet_{i-1}')}$ with number of iterations $\K_\Unlrn$. 

We first introduce a few lemmas that will be used in both phases. The first set of following lemmas show that the transformation $\Thet_{\step k}, \Thet_{\step k}' \rightarrow \T_k(\Thet_{\step k}), \T_k(\Thet_{\step k})$ preserves the R\'enyi divergence. To prove this property, we show that $\T_k$ is a differentiable bijective map in Lemma~\ref{lem:T_is_bijective_and_lipschitz} and apply the following Lemma from \citet{vempala2019rapid}.
\begin{lemma}[{\citet[Lemma 15]{vempala2019rapid}}]
	\label{lem:divergence_preserved_for_bijections}
	If $\T: \domain \rightarrow \domain$ is a differentiable bijective map, then for any random variables $\Thet, \Thet' \in \domain$, and for all $\q > 0$,
	\begin{equation}
		\Ren{\q}{\T(\Thet)}{\T(\Thet')} = \Ren{\q}{\Thet}{\Thet}.
	\end{equation}
\end{lemma}
\begin{lemma}
	\label{lem:T_is_bijective_and_lipschitz}
	If $\loss(\thet; \x)$ is a twice continuously differentiable, convex, and $\smh$-smooth loss function and regularizer is $\reg(\thet) = \frac{\cvx}{2}\norm{\thet}^2$, then the map $\T_k$ defined in~\eqref{eqn:T} is:
	\begin{enumerate}
		\item a differentiable bijection for any $\step < \frac{1}{\cvx + \smh}$, and
		\item $(1-\step\cvx)$-Lipschitz for any $\step \leq \frac{2}{2\cvx + \smh}$.
	\end{enumerate}
\end{lemma}
\begin{proof}
	{\bf Differentiable bijection.} To see that $\T_k$ is injective, assume $\T_k(\thet) = \T_k(\thet')$ for some $\thet, \thet' \in \domain$. Then, by $(\smh + \cvx)$-smoothness of $\Loss \stackrel{\text{def}}{=} (\Loss_{\D(k)} + \Loss_{\D'(k)})/2$,
		\begin{align*}
			\norm{\thet - \thet'} &= \norm{\T_k(\thet) + \step \graD{\Loss(\thet)} - \T_k(\thet') -  \step\graD{\Loss(\thet') }} \\
					      &= \step \norm{\graD{\Loss(\thet)} - \graD{\Loss(\thet')}} \\
					      & \leq \step(\cvx + \smh) \norm{\thet - \thet'}.
		\end{align*}
		Since $\step < 1/(\cvx + \smh)$, we must have $\norm{\thet - \thet'} = 0$. For showing $\T_k$ is surjective, consider the proximal mapping 
		\begin{equation}
			\prox_\Loss(\thet) = \underset{\thet' \in \domain}{\arg\min} \frac{\norm{\thet' - \thet}^2}{2} - \step \Loss(\thet').
		\end{equation}
		Note that $\prox_\Loss(\cdot)$ is strongly convex for $\step < \frac{1}{\cvx + \smh}$. Therefore, from KKT conditions, we have $\thet = \prox_\Loss(\thet) - \step \graD{\Loss(\prox_\Loss(\thet))} = \T_k(\prox_\Loss(\thet))$. Differentiability of $\T_k$ follows from the twice continuously differentiable assumption on $\loss(\thet;\x)$.

		{\bf Lipschitzness.} Let $\Loss \stackrel{\text{def}}{=} (\Loss_{\D(k)} + \Loss_{\D'(k)})/2$. For any $\thet, \thet' \in \domain$,
		\begin{align*}
			\norm{\T_k(\thet) - \T_k(\thet')}^2 &= \norm{\thet - \step\graD{\Loss(\thet)} - \thet' + \step\graD{\Loss(\thet')}}^2 \\
							&= \norm{\thet - \thet'}^2 + \step^2 \norm{\graD{\Loss(\thet)} - \graD{\Loss(\thet')}}^2 - 2 \step \dotP{\thet - \thet'}{\graD{\Loss(\thet)} - \graD{\Loss(\thet'}}.
		\end{align*}
		We define a function $g(\thet) = \Loss(\thet) - \frac{\cvx}{2}\norm{\thet}^2$, which is convex and $\smh$-smooth. By
		co-coercivity property of convex and $\smh$-smooth functions, we have
		\begin{align*}
			&\dotP{\thet - \thet'}{\graD{g(\thet)} - \graD{g(\thet')}} \geq \frac{1}{\smh} \norm{\graD{g(\thet)} - \graD{g(\thet')}}^2 \\
			\implies& \dotP{\thet - \thet'}{\graD{\Loss(\thet)} - \graD{\Loss(\thet')}} - \cvx \norm{\thet - \thet'}^2 \geq \frac{1}{\smh} \bigg(\norm{\graD{\Loss(\thet)} - \graD{\Loss(\thet')}}^2 + \cvx^2\norm{\thet - \thet'}^2 \\
				&\qquad\qquad\qquad\qquad\qquad\qquad\qquad\qquad\qquad\qquad - 2\cvx \dotP{\thet - \thet'}{\graD{\Loss(\thet)} - \graD{\Loss(\thet')}} \bigg) \\
			\implies& \dotP{\thet - \thet'}{\graD{\Loss(\thet)} - \graD{\Loss(\thet')}} \geq \frac{1}{2\cvx + \smh} \norm{\graD{\Loss(\thet)} - \graD{\Loss(\thet')}}^2 + \frac{\cvx(\cvx+\smh)}{2\cvx+\smh} \norm{\thet - \thet'}^2.
		\end{align*}
		Substituting this in the above inequality, and noting that $\step \leq \frac{2}{2\cvx+\smh}$, we get
		\begin{align*}
			\norm{\T_k(\thet) - \T_k(\thet')}^2 &\leq \left(1 - \frac{2\step\cvx(\cvx+\smh)}{2\cvx+\smh}\right) \norm{\thet - \thet'}^2 + \left(\step^2 - \frac{2\step}{\smh + 2\cvx}\right) \norm{\graD{\Loss(\thet)} - \graD{\Loss(\thet')}}^2 \\
							&\leq \left(1 - \frac{2\step\cvx(\cvx+\smh)}{2\cvx+\smh}\right) \norm{\thet - \thet'}^2 + \left(\step^2\cvx^2 - \frac{2\step\cvx^2}{\smh + 2\cvx}\right) \norm{\thet - \thet'}^2 \\
							&= (1 - \step\cvx)^2 \norm{\thet - \thet'}^2.
		\end{align*}

\end{proof}

The second set of lemmas presented below describe how $\Ren{\q}{\Thet_t}{\Thet_t}$ evolves with time in both phases I and II. Central to our analysis is the following lemma which bounds the rate of change of R\'enyi divergence for any pair of diffusion process characterized by their Fokker-Planck equations.
\begin{lemma}[Rate of change of R\'enyi divergence~\citep{chourasia2021differential}]
	\label{lem:rate_of_divergence}
	Let $\V_t, \V_t': \domain \rightarrow \domain$ be two time dependent vector field such that
	$\max_{\thet \in \domain} \norm{\V_t(\thet) - \V_t'(\thet)} \leq \lip$ for all $\thet \in \domain$ and $t \geq 0$.
	For a diffusion process $(\Thet_t)_{t\geq0}$ and $(\Thet_t')_{t\geq0}$ defined by the Fokker-Planck equations
	\begin{equation}
		\begin{cases}
			\partial_t \mU_t(\thet) = \divR{\mU_t(\thet) \V_t(\thet)} + \noise^2 \lapL{\mU_t(\thet)} \quad \text{and}\\
			\partial_t \mU_t'(\thet) = \divR{\mU_t'(\thet) \V_t'(\thet)} + \noise^2 \lapL{\mU_t'(\thet)},
		\end{cases}
	\end{equation}
	respectively, where $\mU_t$ and $\mU_t$ are the densities of $\Thet_t$ and $\Thet_t'$, the rate of change of R\'enyi divergence between the two at any $t \geq 0$ is upper bounded as
	\begin{equation}
		\label{eqn:rate_of_divergence}
		\partial_t \Ren{\q}{\mU_t}{\mU_t'} \leq \frac{\q \lip^2}{2\noise^2} - \frac{\q\noise^2}{2} \frac{\Gren{\q}{\mU_t}{\mU_t'}}{\Eren{\q}{\mU_t}{\mU_t'}}.
	\end{equation}
\end{lemma}
We will apply the above lemma to the Fokker-Planck equation~\eqref{eqn:fokker_planck_deletion} of our pair of tracing diffusion SDE~\eqref{eqn:tracing_diffusion_deletion} and solve the resulting differential inequality to prove the bound in~\eqref{eqn:dp_bound_eqiv}. To assist our proof, we rely on the following lemma showing that our two tracing diffusion satisfy the $\LS$ inequality described in Definition~\ref{dfn:def_lsi}, which enables the use the inequality~\eqref{eqn:ren_info_lsi} in Lemma~\ref{lem:ren_info_lsi}.
\begin{lemma}
	\label{lem:lsi_tracing}
	If loss $\loss(\thet;\x)$ is convex and $\smh$-smooth, regularizer is $\reg(\thet) = \frac{\cvx}{2}\norm{\thet}^2$, and learning rate $\step \leq \frac{2}{2\cvx + \smh}$, then the tracing diffusion $(\Thet_t)_{0 \leq t \leq \step(\K_\Lrn + i \K_\Unlrn)}$ and $(\Thet_t')_{0 \leq t \leq \step(\K_\Lrn + i \K_\Unlrn)}$ defined in \eqref{eqn:tracing_diffusion_deletion} with $\Thet_0, \Thet_0' \sim \rhO = \Gaus{0}{\frac{\noise^2}{\cvx(1-\step\cvx/2)}\Id}$ satisfy $\LS$ inequality with constant $\cvx(1 - \step\cvx/2)$.
\end{lemma}
\begin{proof}
	For any iteration $0 \leq k < \K_\Lrn + i \K_\Unlrn$ in the extended run of Noisy-GD, and any $0 \leq s \leq \step$, let's define two functions $\Loss_s, \Loss_s':\domain \rightarrow \R$ as follows:
	\begin{equation}
		\Loss_s = \frac{1 + s/\step}{2} \Loss_{\D(k)} + \frac{1 - s/\step}{2} \Loss_{\D'(k)}, \quad \text{and} \
		\Loss_s' = \frac{1 - s/\step}{2} \Loss_{\D(k)} + \frac{1 + s/\step}{2} \Loss_{\D'(k)}.
	\end{equation}
	Since $\reg(\cdot)$ is the $\ltwo(\cvx)$ regularizer and $\loss(\thet;\x)$ is convex and $\smh$-smoothness, both $\Loss_s$ and $\Loss_s'$ are $\cvx$-strongly convex and $(\cvx+\smh)$-smooth for all $0 \leq s \leq \step$ and any $k$. We define maps $\T_s, \T_s': \domain \rightarrow \domain$ as
	\begin{equation}
		\T_s(\thet) = \thet - \step \graD{\Loss_s(\thet)}, \quad \text{and} \ \T_s'(\thet) = \thet - \graD{\Loss_s'(\thet)}.
	\end{equation}
	From a similar argument as in Lemma~\ref{lem:T_is_bijective_and_lipschitz}, both $\T_s$ and $\T_s'$ are $(1-\step\cvx)$-Lipschitz for learning rate $\step \leq \frac{2}{2\cvx+\smh}$.

	Note that the densities of $\Thet_t$ and $\Thet_t'$ of the tracing diffusion for $t = \step k + s$ can be respectively expressed as
	\begin{equation}
		\label{eqn:map_plus_noise}
		\mU_t = {\T_s}_{\#}(\mU_{\step k}) \conv \Gaus{0}{2s\noise^2 \Id}, \quad \text{and} \
		\mU_t' = {\T_s'}_{\#}(\mU_{\step k}') \conv \Gaus{0}{2s\noise^2 \Id},
	\end{equation}
	where $\mU_{\step k}$ and $\mU_{\step k}'$ represent the distributions of $\Thet_{\step k}$ and $\Thet_{\step k}'$. We prove the lemma via induction.

	{\bf Base step:} Since $\Thet_0, \Thet_0'$ are both Gaussian distributed with variance ${\frac{\noise^2}{\cvx(1-\step\cvx/2)}}$, from Lemma~\ref{lem:lsi_gaussian} they satisfy $\LS$ inequality with constant $\cvx(1-\step\cvx/2)$.

	{\bf Induction step:} Suppose $\mU_{\step k}$ and $\mU_{\step k}'$ satisfy $\LS$ inequality with 
	constant $\cvx(1-\step\cvx/2)$. Since equation \eqref{eqn:map_plus_noise} shows that $\mU_t, \mU_t'$ are 
	both Gaussian convolution on a pushforward distribution of $\mU_{\step k}, \mU_{\step k}'$ respectively over 
	a Lipschitz function, from Lemma~\ref{lem:lsi_gaussian} and Lemma~\ref{lem:lsi_lipschitz}, both $\mU_t, \mU_t'$ satisfy
	$\LS$ inequality with constant
	\begin{equation}
		\left(\frac{(1-\step\cvx)^2}{\cvx(1 - \step\cvx/2)} + 2s \right)^{-1} \geq \cvx(1-\step\cvx/2) \times \underbrace{[(1-\step\cvx)^2 + \cvx\step(2 - \step\cvx)]^{-1}}_{=1},
	\end{equation}
	for all $\step k  \leq t \leq \step (k+1)$.
\end{proof}
We are now ready to prove the data-deletion bound in~\eqref{eqn:dp_bound_eqiv}.
\begin{theorem}[Data-Deletion guarantee on $(\Lrn_\Nsgd, \Unlrn_\Nsgd)$ under convexity]
	\label{thm:data_deletion_convex}
	Let the weight initialization distribution be $\rhO = \Gaus{0}{\frac{\noise^2}{\cvx(1 - \step\cvx/2)}}$, the loss function $\loss(\thet;\x)$ be convex, $\smh$-smooth, and $\lip$-Lipschitz, the regularizer be $\reg(\thet) = \frac{\cvx}{2} \norm{\thet}^2$, and learning rate be $\step < \frac{1}{\cvx + \smh}$. Then Algorithm pair $(\Lrn, \Unlrn)$ satisfies a $(\q, \epsdd)$-data-deletion guarantee under all non-adaptive $\reps$-requesters for any noise variance $\noise^2 > 0$ and $\K_\Lrn \geq 0$ if
	\begin{equation} \K_\Unlrn \geq \frac{2}{\step\cvx}\log\left(\frac{4\q\lip^2}{\cvx\epsdd\noise^2\n^2}\right).
	\end{equation}
\end{theorem}

\begin{proof}
	Following the preceding discussion, to prove this theorem, it suffices to show that the inequality~\eqref{eqn:dp_bound_eqiv} holds under the stated conditions. Consider the Fokker-Planck equation described in Lemma~\ref{lem:fokker_planck_deletion} for the pair of tracing diffusions SDEs in~\eqref{eqn:sde_convex}: at any time $t$ in duration $\step k < t \leq \step(k+1)$ for any iteration $0 \leq k < \K_\Lrn + i\K_\Unlrn$, 
	\begin{equation}
		\begin{cases}
		\partial_t \mU_{t}(\thet) &= \divR{\mU_{t}(\thet)\expec{}{\grad_k(\Thet_{\step k}) \middle\vert \Thet_t = \thet}} + \noise^2\lapL{\mU_{t}(\thet)}, \\
		\partial_t \mU_{t}'(\thet) &= \divR{\mU_{t}'(\thet)\expec{}{-\grad_k(\Thet_{\step k}') \middle\vert \Thet_t' = \thet}} + \noise^2\lapL{\mU_{t}'(\thet)},
		\end{cases}
	\end{equation}
	where $\mU_{t}$ and $\mU_{t}'$ are the distribution of $\Thet_t$ and $\Thet_t'$. Since $\loss(\thet;\x)$ is $\lip$-Lipschitz and for any ${\K_\Lrn + (i-1)\K_\Unlrn \leq k < \K_\Lrn + i \K_\Unlrn}$ we have ${\D(k)[\ind] = \D'(k)[\ind]}$, note from the definition of $\grad_k(\thet)$ in~\eqref{eqn:sde_convex} that
	\begin{equation}
		\label{eqn:phase_sensitivity}
		\norm{\expec{}{\grad_k(\Thet_{\step k}) \middle\vert \Thet_t = \thet} - \expec{}{-\grad_k(\Thet_{\step k}') \middle\vert \Thet_t' = \thet}} \leq \begin{cases}\frac{2\lip}{\n} &\quad \text{if}\ k < \K_\Lrn + (i-1) \K_\Unlrn \\ 0&\quad \text{otherwise} \end{cases}.
	\end{equation}
	Therefore, applying Lemma~\ref{lem:rate_of_divergence} to the above pair of Fokker-Planck equations gives that for any $t$ in duration $\step k < t \leq \step (k+1)$,
	\begin{equation}
		\label{eqn:phase_pdi}
		\partial_t \Ren{\q}{\mU_{t}}{\mU_{t}'} \leq \frac{2\q \lip^2}{\noise^2\n^2} \indic{t \leq \step(\K_\Lrn + (i-1)\K_\Unlrn)}  - \frac{\q\noise^2}{2} \frac{\Gren{\q}{\mU_{t}}{\mU_{t}'}}{\Eren{\q}{\mU_{t}}{\mU_{t}'}}.
	\end{equation}
	Equation~\eqref{eqn:phase_pdi} suggests a phase change in the dynamics at iteration $k = \K_\Lrn + (i-1) \K_\Unlrn$. In phase I, the divergence bound increases with time due to the effect of the differing record in database pairs $(\D_j, \D_j')_{0 \leq j \leq i-1}$. In phase II however, the update request $\up_i$ makes $\D_{i-1} \circ \up_i = \D_{i-1}' \circ \up_i$, and so doing gradient descent rapidly shrinks the divergence bound. This phase change is illustrated in the Figure~\ref{fig:convex_proof}.

	For brevity, we denote $R(\q,t) = \Ren{\q}{\mU_{t}}{\mU_{t}'}$. Since $\step < \frac{1}{\cvx + \smh} < \frac{2}{2\cvx + \smh}$, from Lemma~\ref{lem:lsi_tracing}, the distribution $\mU_{t}'$ satisfies $\LS$ inequality with constant $\cvx(1-\cvx\step/2)$. So, we can apply Lemma~\ref{lem:ren_info_lsi} to simplify the above partial differential inequality as follows.
	\begin{equation}
		\label{eqn:unlearning_pdi}
		\partial_t R(\q,t) + \cvx(1 - \cvx\step/2) \left(\frac{R(\q,t)}{\q} + (\q - 1)\partial_\q R(\q,t) \right) \leq \frac{2\q \lip^2}{\noise^2\n^2} \indic{t \leq \step(\K_\Lrn + (i-1)\K_\Unlrn)}.
	\end{equation}
	For brevity, let constant $c_1 = \cvx(1 - \cvx\step/2)$ and constant $c_2 = \frac{2\lip^2}{\noise^2\n^2}$. We define $u(\q,t) = \frac{R(\q,t)}{\q}$. Then, 
	\begin{align*}
		&\partial_t R(\q,t) + c_1 \left(\frac{R(\q,t)}{\q} + (\q - 1) \partial_\q R(\q,t) \right) \leq c_2 \q \times \indic{t \leq \step(\K_\Lrn + (i-1) \K_\Unlrn)} \\
		\implies& \partial_t u(q,t) + c_1 u(\q,t) + c_1(\q - 1) \partial_q u(\q,t) \leq c_2 \times \indic{t \leq \step(\K_\Lrn + (i-1) \K_\Unlrn)}.
	\end{align*}
	For some constant $\bar\q > 1$, let $\q(s) = (\bar \q - 1) \exp\left[c_1\left\{s - \step(\K_\Lrn + i \K_\Unlrn)\right\}\right] + 1$ and $t(s) = s$. Note that $\der{\q(s)}{s} = c_1 (\q(s) - 1)$ and $\der{t(s)}{s} = 1$. Therefore, for any $\step k < s \leq \step (k+1)$, the differential inequality followed along the path $u(s) = u(\q(s),t(s))$ is 
	\begin{align}
		&\der{u(s)}{s} + c_1 u(s) \leq c_2 \times \indic{t \leq \step(\K_\Lrn + (i-1) \K_\Unlrn)} \\
		\implies &\der{}{s} \{ e^{c_1 s} u(s) \} \leq c_2 \times \indic{t \leq \step(\K_\Lrn + (i-1) \K_\Unlrn)}.
	\end{align}
	Since the map $\T_k(\cdot)$ in \eqref{eqn:T} is a differentiable bijection for $\step < \frac{1}{\cvx + \smh}$ as per Lemma~\ref{lem:T_is_bijective_and_lipschitz}, note that Lemma~\ref{lem:divergence_preserved_for_bijections} implies that ${\lim_{s \rightarrow \step k^+} u(s) = u(\step k)}$. Therefore, we can directly integrate in the duration $0 \leq t \leq \step (\K_\Lrn + i\K_\Unlrn)$ to get
	\begin{align*}
		& \left[e^{c_1 s} u(s)\right]_{0}^{\step(\K_\Lrn + i\K_\Unlrn)} \leq \int_0^{\step(\K_\Lrn + (i-1)\K_\Unlrn)} c_2 e^{c_1 s}\dif{s} \\
		\implies&e^{c_1 \step(\K_\Lrn + i\K_\Unlrn)} u(\step(\K_p + i\K_u)) - u(0) \leq \frac{c_2}{c_1}[e^{c_1 \step(\K_\Lrn + (i-1)\K_\Unlrn)} - 1] \\
		\implies& u(\step(\K_\Lrn + i\K_\Unlrn)) \leq \frac{c_2}{c_1}e^{-c_1 \step \K_\Unlrn}. \tag{Since $u(0) = R(\q(0),0) / \q(0) = 0$.}
	\end{align*}
	%\right
	Noting that $\q(0) \geq 1$, on reverting the substitution, we get
	\begin{align*}
		\Ren{\bar\q}{\mU_{\step (\K_\Lrn + i\K_\Unlrn)}}{\mU_{\step (\K_\Lrn + i\K_\Unlrn)}'} &\leq \frac{2\bar\q\lip^2}{\cvx\noise^2\n^2(1 - \step\cvx/2)} \exp\left(-\step\cvx\K_\Unlrn(1-\step\cvx/2)\right) \\
												      &\leq \frac{4\bar\q\lip^2}{\cvx\noise^2\n^2} \exp\left(-\frac{\step\cvx\K_u}{2}\right) \tag{Since $\step < \frac{1}{\cvx + \smh}$}
	\end{align*}
	Recall from our construction that $\mU_{\step (\K_\Lrn + i\K_\Unlrn)}$ and $\mU_{\step (\K_\Lrn + i\K_\Unlrn)}'$ are the distributions of outputs $\Unlrn(\D_{i-1}, \up_i, \hat\Thet_{i-1})$ and $\Unlrn(\D_{i-1}', \up_i, \hat\Thet_{i-1}')$ respectively. Therefore, choosing $\K_\Unlrn$ as specified in the theorem statement concludes the proof.
\end{proof}

Our next goal in this section is to provide utility guarantees for the algorithm pair $(\Lrn_\Nsgd, \Unlrn_\Nsgd)$ in form of excess empirical risk bounds. For that, we introduce some additional auxiliary results first. The following Lemma~\ref{lem:excess_risk_after_edit} shows that excess empirical risks does not increase too much on replacing $\reps$ records in a database, and Lemma~\ref{lem:excess_risk_convex} provides a convergence guarantee on the excess empirical risk of Noisy-GD algorithm under convexity.
\begin{lemma}
	\label{lem:excess_risk_after_edit}
	Suppose the loss function $\loss(\thet;\x)$ is convex, $\lip$-Lipschitz, and $\smh$-smooth, and the regularizer is $\reg(\thet) = \frac{\cvx}{2}\norm{\thet}^2$. Then, the excess empirical risk of any randomly distributed parameter $\Thet$ for any database 
	$\D \in \X^\n$ after applying any edit request $\up \in \U^\reps$ that modifies no more than $\reps$ records is bounded as
	\begin{equation}
		\err(\Thet;\D\circ\up) \leq \left(1 + \frac{\smh}{\cvx}\right)\left[ 2\ \err(\Thet; \D) + \frac{16\reps^2\lip^2}{\cvx\n^2} \right].
	\end{equation}
\end{lemma}

\begin{proof}
	Let $\thet^*_\D$ and $\thet^*_{\D\circ\up}$ be the minimizers of objectives $\Loss_\D(\cdot)$ and $\Loss_{\D\circ\up}(\cdot)$ as defined in~\eqref{eqn:objective}.
	From $\cvx$-strong convexity of the $\Loss_\D$,
	\begin{equation}
		 \Loss_{\D}(\thet^*_{\D \circ \up}) - \Loss_{\D}(\thet^*_\D) \geq \frac{\cvx}{2} \norm{\thet^*_{\D\circ\up} - \thet^*_{\D}}^2.
	\end{equation}
	From optimality of $\thet^*_{\D\circ\up}$ and $\lip$-Lipschitzness of $\loss(\thet;\x)$, we have
	\begin{align*}
		\Loss_{\D}(\thet^*_{\D\circ\up}) &= \Loss_{\D\circ\up}(\thet^*_{\D\circ\up}) + \frac{1}{\n} \left( \sum_{\x \in \D} \loss(\thet^*_{\D\circ\up};\x) - \sum_{\x \in \D\circ\up} \loss(\thet^*_{\D\circ\up};\x) \right) \\
						  &\leq \Loss_{\D\circ\up}(\thet^*_{\D}) + \frac{1}{\n} \left( \sum_{\x \in \D} \loss(\thet^*_{\D\circ\up};\x) - \sum_{\x \in \D\circ\up} \loss(\thet^*_{\D\circ\up};\x) \right)  \\
						  &= \Loss_{\D}(\thet^*_{\D}) + \frac{1}{\n} \sum_{\x \in \D} \left(\loss(\thet^*_{\D\circ\up};\x) - \loss(\thet^*_{\D};\x)\right) + \frac{1}{\n} \sum_{\x \in \D\circ\up} \left(\loss(\thet^*_{\D};\x) - \loss(\thet^*_{\D\circ\up};\x)\right) \\
						  &\leq \Loss_{\D}(\thet^*_{\D}) + \frac{2\reps\lip}{\n} \norm{\thet^*_{\D\circ\up} - \thet^*_{\D}}.
	\end{align*}
	Combining the two inequalities give
	\begin{equation}
		\norm{\thet^*_{\D\circ\up} - \thet^*_{\D}} \leq \frac{4\reps\lip}{\cvx\n}.
	\end{equation}
	Therefore, from $(\cvx + \smh)$-smoothness of $\Loss_{\D\circ\up}$ and $\cvx$-strong convexity of $\Loss_\D$, we have
	\begin{align*}
		\err(\Thet;\D\circ\up) &= \expec{}{\Loss_{\D\circ\up}(\Thet) - \Loss_{\D\circ\up}(\thet^*_{\D\circ\up})} \\
				       &\leq \frac{\cvx + \smh}{2} \expec{}{\norm{\Thet - \thet^*_{\D\circ\up}}^2} \\
				       &\leq (\cvx + \smh) \left[\expec{}{\norm{\Thet - \thet^*_\D}^2} + \norm{\thet^*_{\D} - \thet^*_{\D\circ\up}}^2 \right] \\
				       &\leq \left(1 + \frac{\smh}{\cvx}\right) \left[2 \expec{}{\Loss_\D(\Thet) - \Loss_\D(\thet^*_\D)} + \frac{16\reps^2\lip^2}{\cvx\n^2}\right].
	\end{align*}
\end{proof}

\begin{lemma}[Accuracy of Noisy-GD]
	\label{lem:excess_risk_convex}
	For convex, $\lip$-Lipschitz, and, $\smh$-smooth loss function $\loss(\thet;\x)$ and regularizer $\reg(\thet) = \frac{\cvx}{2} \norm{\thet}^2$, if learning rate $\step < \frac{1}{\cvx+\smh}$, the excess empirical risk of ${\Thet_{\step \K} = \Nsgd(\D, \Thet_0, \K)}$ for any $\D \in \X^\n$ is bounded as
	\begin{equation}
		\err(\Thet_{\step \K};\D) \leq \err(\Thet_0;\D)e^{-\cvx\step\K/2} + \left(1 + \frac{\smh}{\cvx}\right)\dime\noise^2.
	\end{equation}
\end{lemma}

\begin{proof}
	Let $\Thet_{\step k}$ denote the $k$th iteration parameter of Noisy-GD run. Recall that $k+1$th noisy gradient update step is
	\begin{equation}
		\Thet_{\step(k+1)} = \Thet_{\step k} - \step \graD{\Loss_\D(\Thet_{\step k})} + \sqrt{2\step\noise^2}\Z_k.
	\end{equation}
	From $(\smh+\cvx)$-smoothness of $\Loss_{\D}$, we have
	\begin{align*}
		\Loss_{\D}(\Thet_{\step(k+1)}) &\leq \Loss_\D(\Thet_{\step k}) + \dotP{\graD{\Loss_\D(\Thet_{\step k})}}{\Thet_{\step(k+1)} - \Thet_{\step k}} + \frac{\smh + \cvx}{2} \norm{\Thet_{\step(k+1)} - \Thet_{\step k}}^2 \\
						&= \Loss_\D(\Thet_{\step k}) - \step \norm{\graD{\Loss_\D(\Thet_{\step k})}}^2 + \sqrt{2\step\noise^2}\dotP{\graD{\Loss_\D(\Thet_{\step k})}}{\Z_k} \\
						&\quad  + \frac{\step^2(\smh + \cvx)}{2}\norm{\graD{\Loss_\D(\Thet_{\step k})}}^2 + \step\noise^2 (\smh + \cvx)\norm{\Z_k}^2\\
						&\quad  - \step\sqrt{2\step\noise^2}(\smh + \cvx)\dotP{\graD{\Loss_\D(\Thet_{\step k})}}{\Z_k}
	\end{align*}
	On taking expectation over the joint distribution of $\Thet_{\step k}, \Thet_{\step(k+1)}, \Z_k$, the above simplifies to
	\begin{equation}
		\expec{}{\Loss_{\D}(\Thet_{\step(k+1)})} \leq \expec{}{\Loss_\D(\Thet_{\step k})} - \step \left(1 - \frac{\step(\cvx + \smh)}{2}\right)\expec{}{\norm{\graD{\Loss_\D(\Thet_{\step k})}}^2} + \step\dime\noise^2(\smh + \cvx).
	\end{equation}
	Let $\thet_\D^* = \underset{\thet \in \domain}{\arg\min} \ \Loss_\D(\thet)$. From $\cvx$-strong convexity of $\Loss_\D$, for any $\thet \in \domain$, we have
	\begin{equation}
		\norm{\graD{\Loss_\D(\thet)}}^2 \geq 2\cvx (\Loss_\D(\thet) - \Loss_\D(\thet_\D^*)).
	\end{equation}
	Let $\gamma = \cvx\step(2 - \step(\cvx + \smh))$.
	Plugging this in the above inequality, and subtracting $\Loss_\D(\thet_\D^*)$ on both sides, for
	$\step < \frac{1}{\cvx + \smh}$, we get
	\begin{align*}
		\expec{}{\Loss_{\D}(\Thet_{\step(k+1)}) - \Loss_\D(\thet_\D^*)} &\leq  (1 - \gamma) \expec{}{\Loss_{\D}(\Thet_{\step k}) - \Loss_\D(\thet_\D^*)} + \step\dime\noise^2(\smh + \cvx) \\ 
									&\leq (1 - \gamma)^{k+1} \expec{}{\Loss_{\D}(\Thet_0) - \Loss_\D(\thet^*)} + \step\dime\noise^2(\smh + \cvx)(1 + \cdots + (1 - \gamma)^{k+1}) \\
									&\leq e^{-\gamma(k+1)/2} \expec{}{\Loss_{\D}(\Thet_0) - \Loss_\D(\thet_\D^*)} + \frac{\step\dime\noise^2(\smh + \cvx)}{\gamma}.
	\end{align*}
	For $\step < \frac{1}{\cvx + \smh}$, note that $\gamma \geq \cvx\step$, and so
	\begin{equation}
		\err(\Thet_{\step \K};\D) \leq \err(\Thet_0;\D)e^{-\cvx\step\K/2} + \left(1 + \frac{\smh}{\cvx}\right)\dime\noise^2.
	\end{equation}
\end{proof}

Finally, we are ready to prove our main Theorem~\ref{thm:unlearning_accuracy_convex} showing that the algorithm pair $(\Lrn_\Nsgd, \Unlrn_\Nsgd)$ solves the data-deletion problem as described in Section~\ref{sec:deletion}. We basically combine the R\'enyi DP guarantee in Theorem~\ref{thm:chourasia}, non-adaptive data-deletion guarantee in Theorem~\ref{thm:data_deletion_convex}, and prove excess empirical risk bound using Lemma~\ref{lem:excess_risk_convex} and Lemma~\ref{lem:excess_risk_after_edit}.

\begin{reptheorem}{thm:unlearning_accuracy_convex}[Utility, privacy, deletion, and computation tradeoffs]
Let constants ${\cvx, \smh, \lip > 0}$, ${\q > 1}$, and ${0 < \epsdd \leq \epsdp}$. Define constant $\kappa = \frac{\cvx + \smh}{\cvx}$. Let the loss function $\loss(\thet;\x)$ be twice differentiable, convex, $\lip$-Lipschitz, and $\smh$-smooth, the regularizer be $\reg(\thet) = \frac{\cvx}{2}\norm{\thet}^2$. If the learning rate be $\step = \frac{1}{2(\cvx + \smh)}$, the gradient noise variance is ${\noise^2 = \frac{4\q\lip^2}{\cvx \epsdp\n^2}}$, and the weight initialization distribution is ${\rhO = \Gaus{0}{\frac{\noise^2}{\cvx(1 - \step\cvx/2)\Id}}}$, then 
% the number of iterations in $\Lrn_\Nsgd$ is ${\K_\Lrn \geq 2\kappa\log \left( \frac{\epsdp\n^2}{4\q\dime} \right)}$, and the number of iterations in $\Unlrn_\Nsgd$ is ${\K_\Unlrn \geq 2\kappa \log \max\{5\kappa, \frac{8\reps^2}{\q\dime}, \left(\frac{\epsdp}{\epsdd}\right)^2\}}$, 
%
\begingroup
\renewcommand\labelenumi{\bf(\theenumi.)}
\begin{enumerate}
	\item both $\Lrn_\Nsgd$ and $\Unlrn_\Nsgd$ are $(\q, \epsdp)$-R\'enyi DP for any $\K_\Lrn, \K_\Unlrn \geq 0$,
	\item pair $(\Lrn_\Nsgd, \Unlrn_\Nsgd)$ satisfies $(\q, \epsdd)$-data-deletion all non-adaptive $\reps$-requesters
		\begin{equation}
			\text{if} \quad \K_\Unlrn \geq 4\kappa \log \frac{\epsdp}{\epsdd},
		\end{equation}
	\item and all models in $(\hat\Thet_i)_{i \geq 0}$ produced by $(\Lrn_\Nsgd, \Unlrn_\Nsgd, \updreq)$ on any ${\D_0 \in \X^\n}$, where $\updreq$ is any $\reps$-requester, have an excess empirical risk $\err(\hat\Thet_i; \D_i) = O\left(\frac{\q\dime}{\epsdp\n^2}\right)$
		\begin{equation}
			\text{if} \quad \K_\Lrn \geq 4\kappa\log \left( \frac{\epsdp\n^2}{4\q\dime} \right),\quad \text{and} \quad \K_\Unlrn \geq 4\kappa \log \max\left\{5\kappa, \frac{8\epsdp\reps^2}{\q\dime}\right\}.
		\end{equation}
\end{enumerate}
\endgroup
\end{reptheorem}

\begin{proof}

	{\bf (1.) Privacy.} By Theorem~\ref{thm:chourasia}, the Noisy-GD with $\K$ iterations will be $(\q, \epsdp)$-R\'enyi DP for the stated choice of loss function, regularizer, and learning rate as long as ${\noise^2 \geq \frac{4\q\lip^2}{\cvx\epsdp\n^2} \left(1 - e^{-\cvx \step \K/2}\right)}$. Therefore, if we set ${\noise^2 = \frac{4\q\lip^2}{\cvx\epsdp\n^2}}$, Noisy-GD is $(\q, \epsdp)$-R\'enyi DP for any $\K$. For the same $\noise^2$, both $\Lrn_\Nsgd$ and $\Unlrn_\Nsgd$ are also $(\q, \epsdp)$-R\'enyi DP for any $\K_\Lrn$ and $\K_\Unlrn$ as they run Noisy-GD on respective databases for generating the output.

	{\bf (2.) Deletion.} By Theorem~\ref{thm:data_deletion_convex}, for the stated choice of loss function, regularizer, learning rate, and weight initialization distribution, the algorithm pair $(\Lrn_\Nsgd, \Unlrn_\Nsgd)$ satisfies $(\q, \epsdd)$-data-deletion under all non-adaptive $\reps$-requesters $\updreq$ if $\K_\Unlrn \geq \frac{2}{\step\cvx} \log\left(\frac{4\q\lip^2}{\cvx\epsdd \noise^2 \n^2}\right)$. By plugging in ${\noise^2 = \frac{4\q\lip^2}{\cvx\epsdp\n^2}}$ and $\step = \frac{1}{2(\cvx + \smh)}$, this constraint simplifies to $\K_\Unlrn \geq 4\kappa \log\frac{\epsdp}{\epsdd}$.

	{\bf (3.) Accuracy.} We prove the induction hypothesis that under the conditions stated in the theorem, ${\err(\hat\Thet_i; \D_i) \leq \frac{10\kappa\q\dime\lip^2}{\cvx\epsdp\n^2}}$ for all $i \geq 0$.

	{\it Base case:} The minimizer $\thet^*_{\D_0}$ of $\Loss_{\D_0}$ satisfies
	\begin{equation}
		\label{eqn:thet_ast_bnd}
		\graD{\Loss_{\D_0}(\thet^*_{\D_0})} = \frac{1}{\n} \sum_{\x \in \D_0} \graD{\loss(\thet^*_{\D_0};\x)} - \cvx\thet^*_{\D_0} = 0 \implies \norm{\thet^*_{\D_0}} \leq \frac{\lip}{\cvx}.
	\end{equation}
	As a result, the excess empirical risk of initialization weights $\Thet_0 \sim \rhO = \Gaus{0}{\frac{\noise^2}{\cvx(1 - \step\cvx/2)\Id}}$ on $\Loss_{\D_0}$ is bounded as

	\begin{align*}
		\err(\Thet_0;\D_0) &= \expec{}{\Loss_{\D_0}(\Thet_0) - \Loss_{\D_0}(\thet^*_{\D_0})} \\
				   &\leq \frac{(\cvx + \smh)}{2} \expec{}{\norm{\Thet_0 - \thet^*_{\D_0}}^2} \tag{From $(\cvx+\smh)$-smoothness of $\Loss_{\D_0}$}\\
				   &=\frac{(\cvx + \smh)}{2} \left[\norm{\thet^*_{\D_0}}^2 + \expec{}{\norm{\Thet_0}^2} -2\expec{}{\dotP{\thet^*_{\D_0}}{\Thet_0}} \right] \\
				   &\leq \left(1 + \frac{\smh}{\cvx}\right) \left[ \frac{\lip^2}{2\cvx} + \frac{\noise^2\dime}{2 - \cvx\step}  \right] \tag{From~\eqref{eqn:thet_ast_bnd} and $\expec{}{\norm{\Z}^2} = \dime$ if $\Z \sim \Gaus{0}{\Id}$.} \\
										&\leq \kappa \left[\frac{\lip^2}{2\cvx} + \dime\noise^2 \right].
	\end{align*}
	Since ${\hat\Thet_0 = \Lrn_\Nsgd(\D_0) = \Nsgd(\D_0, \Thet_0, \K_\Lrn)}$, by Lemma~\ref{lem:excess_risk_convex}, running ${\K_\Lrn \geq 2\kappa\log\left(\frac{\epsdp\n^2}{4\q\dime}\right)}$ iterations gives
	\begin{align*}
		\err(\hat\Thet_0;\D_0) &\leq \err(\Thet_0;\D_0)e^{-\cvx\step\K_\Lrn/2} + \kappa\dime\noise^2 \\
				       &\leq \kappa \left[\frac{\lip^2}{2\cvx} + \dime\noise^2 \right]e^{-\cvx\step\K_\Lrn/2} + \kappa\dime\noise^2 \\
				       &\leq \frac{\kappa\lip^2}{2\cvx}e^{-\cvx\step\K_\Lrn/2} + \frac{8\kappa\q\dime\lip^2}{\cvx\epsdp\n^2} \tag{On substituting $\noise^2 = \frac{4\q\lip^2}{\cvx\epsdp\n^2}$}\\
				       &\leq \frac{10\kappa\q\dime\lip^2}{\cvx\epsdp\n^2} \tag{Since ${\K_\Lrn \geq 4\kappa\log\left(\frac{\epsdp\n^2}{4\q\dime}\right)}$}
	\end{align*}

	{\it Induction step:} Assume that ${\err(\hat\Thet_{i-1}; \D_{i-1}) \leq \frac{10\kappa\q\dime\lip^2}{\cvx\epsdp\n^2}}$. Since $\hat\Thet_i = \Unlrn_\Nsgd(\D_{i-1}, \up_i, \hat\Thet_{i-1}) = \Nsgd(\D_i, \hat\Thet_{i-1}, \K_\Unlrn)$, by Lemma~\ref{lem:excess_risk_convex} and Lemma~\ref{lem:excess_risk_after_edit}, running ${\K_\Unlrn \geq 2\kappa\log\max\left\{5\kappa, \frac{8\reps^2}{\q\dime}\right\}}$ iterations gives
	\begin{align*}
		\err(\hat\Thet_i;\D_i) &\leq \kappa\left[2 \err(\hat\Thet_{i-1}; \D_{i-1}) + \frac{16\reps^2\lip^2}{\cvx\n^2}\right]e^{-\cvx\step\K_\Unlrn/2} + \kappa\dime\noise^2 \\
				       &\leq \kappa\left[ \frac{20\kappa\q\dime\lip^2}{\cvx\epsdp\n^2} + \frac{16\reps^2\lip^2}{\cvx\n^2}\right]e^{-\cvx\step\K_\Unlrn/2} + \frac{4\kappa\q\dime\lip^2}{\cvx\epsdp\n^2} \tag{Substituting $\noise^2$} \\ 
				       & \leq \frac{16\kappa\reps^2\lip^2}{\cvx\n^2}e^{-\cvx\step\K_\Unlrn/2} + \frac{8\kappa\q\dime\lip^2}{\cvx\epsdp\n^2} \tag{Since $\K_\Unlrn \geq 4\kappa\log (5\kappa)$} \\
				       & \leq \frac{10\kappa\q\dime\lip^2}{\cvx\epsdp\n^2} \tag{Since $\K_\Unlrn \geq 4\kappa\log \frac{8\epsdp\reps^2}{\q\dime}$}
	\end{align*}
\end{proof}

% !TEX root = ../main.tex

\subsection{Proofs for Subsection~\ref{sec:unlrn_nonconvex}}
\label{sec:noisysgd_convergence}

In this Appendix, we provide a proof of our data-deletion and utility guarantee in Theorem~\ref{thm:deletion_accuracy_nonconvex} which applies to non-convex but bounded losses $\loss(\thet;\x)$ under $\ltwo$ regularizer $\reg(\thet)$. Suppose $\D_0 \in \X^\n$ is an arbitrary database, $\updreq$ is any non-adaptive $\reps$-requester, and $(\hat\Thet_i)_{i\geq0}$ is the model sequence generated by the interaction of $(\Lrn_\Nsgd, \Unlrn_\Nsgd, \updreq)$. Our first goal will be to prove $(\q, \epsdd)$-data deletion guarantee on $(\Lrn_\Nsgd, \Unlrn_\Nsgd)$ and we will later use it for arguing utility as well. Recall from Definition~\ref{dfn:deletion} that to prove $(\q, \epsdd)$-data-deletion, we need to construct a map $\pI^\updreq_i: \X^\n \rightarrow \C$ such that for all $i \geq 1$ and any $\up_i \in \U^\reps$, 
\begin{equation}
	\Ren{\q}{\Unlrn(\D_{i-1}, \up_i, \hat\Thet_{i-1})}{\pI^\updreq_i(\D_0 \circ \langle \ind, \y \rangle)} \leq \epsdd \quad\text{for all}\ \langle \ind, \y \rangle \in \up_i.
\end{equation}
Our construction of $\pI^\updreq_i$ for this proof is completely different from the one described in Appendix~\ref{app:deletion_convex}. As discussed in Remark~\ref{rem:non_adap_independence}, since $\updreq$ is non-adaptive, it suffices to show that there exists a map $\pI:\X^\n \rightarrow \C$ such that for all $i\geq 1$,
\begin{equation}
	\Ren{\q}{\Unlrn(\D_{i-1}, \up_i, \hat\Thet_{i-1})}{\pI(\D_i)} \leq \epsdd,
\end{equation}
for all $\D_0 \in \X^\n$ and all edit sequences $(\up_i)_{i\geq1}$ from $\U^\reps$.

Our mapping of choice for the purpose is the Gibbs distribution with the following density:
\begin{equation}
	\label{eqn:gibbs1}
	\pI(\D)(\thet) \propto e^{-\Loss_{\D}(\thet) / \noise^2}.
\end{equation}
The high-level intuition for this construction is that Noisy-GD can be interpreted as Unadjusted Langevin Algorithm (ULA)~\citep{roberts1996exponential}, which is a discretization of the Langevin diffusion (described in eqn.~\eqref{eqn:langevin_sde}) that eventually converges to this Gibbs distribution (see Appendix~\ref{sec:langevin_diffusion} for a quick refresher). However, showing a convergence for ULA (in indistinguishability notions like R\'enyi divergence) to this Gibbs distribution, especially in form of non-asymptotic bounds on the mixing time and discretization error has been a long-standing open problem. Recent breakthrough results by~\citet{vempala2019rapid} followed by~\citet{chewi2021analysis} resolved this problem with an elegant argument, relying solely on isoperimetric assumptions over~\eqref{eqn:gibbs1} that hold for non-convex losses. Our data-deletion argument leverages this rapid convergence result to basically show that once Noisy-GD reaches near-indistinguishability to its Gibbs mixing distribution, maintaining indistinguishability to subsequent Gibbs distribution corresponding to database edits require much fewer Noisy-GD iterations than fresh retraining (i.e. data deletion is faster than retraining).

We start by presenting \citet{chewi2021analysis}'s convergence argument adapted to our Noisy-GD formulation, with a slightly tighter analysis that results in a $\log(\q)$ improvement in the discretization error over the original. Consider the discrete stochastic process $(\Thet_{\eta k})_{0 \leq k \leq \K}$ induced by parameter update step in Noisy-GD algorithm when run for $\K$ iterations on a database $\D$ with an arbitrary start distribution $\Thet_0 \sim \mU_0$. We interpolate each discrete update from $\Thet_{\step k}$ to $\Thet_{\step(k+1)}$ via a diffusion process $\Thet_t$ defined over time $\step k \leq t \leq \step (k + 1)$ as
\begin{equation}
	\label{eqn:tracing_SDE}
	\Thet_t = \Thet_{\step k} - (t - \step k) \graD{\Loss_\D(\Thet_{\step k})} + \sqrt{2\noise^2}(\Z_t - \Z_{\step k}),
\end{equation}
where $\Z_t$ is a Weiner process. Note that if $\Thet_{\step k}$ models the parameter distribution after the $k^{th}$ update, then $\Thet_{\step (k+1)}$ models the parameter distribution after the $k+1^{th}$ update. On repeating this construction for all $k = 0, \cdots, \K$, we get a \emph{tracing diffusion} $\{\Thet_t\}_{t \geq 0}$ for Noisy-GD (which is different from~\eqref{eqn:tracing_diffusion_deletion}). We denote the distribution of random variable $\Thet_t$ with $\mU_t$. The tracing diffusion during the duration $\step k \leq t \leq \step (k+1)$ is characterized by the following Fokker-Planck equation.
\begin{lemma}[Proposition 14~\citep{chewi2021analysis}]
	For tracing diffusion $\Thet_t$ defined in \eqref{eqn:tracing_SDE}, the equivalent Fokker-Planck equation in the
	interval $\step k \leq t \leq \step (k + 1)$ is
	\begin{equation}
		\label{eqn:fokker_planck_noisysgd}
		\partial_t \mU_t(\thet) = \divR{\left\{ \expec{}{\graD{\rLoss_\D(\Thet_{\step k})} - \graD{\rLoss_{\D}(\Thet_t)} \middle\vert \Thet_t = \thet} + \noise^2 \graD{\log \frac{\mU_t(\thet)}{\pI(\D)(\thet)}} \right\} \mU_t(\thet)},
	\end{equation}
	where $\pI(\D)$ is the Gibbs distribution defined in~\eqref{eqn:gibbs1}.
\end{lemma}
\begin{proof}
	Conditioned on observing parameter $\Thet_{\step k} = \thet_{\step k}$, the process $(\Thet_t)_{\step k \leq t \leq \step (k+1)}$ is a Langevin diffusion along a constant Vector field $\graD{\Loss_\D(\thet_{\step k})}$. Therefore, the conditional probability density $\mU_{t| \step k}(\cdot | \thet_{\step k})$ of $\Thet_t$ given $\thet_{\step k}$ follows the following Fokker-Planck equation.
	\begin{equation}
		\partial_t \mU_{t|\step k}(\cdot | \thet_{\step k}) = \noise^2 \lapL{\mU_{t|\step k}(\cdot | \thet_{\step k})} + \divR{\mU_{t|\step k}(\cdot | \thet_{\step k})\graD{\Loss_\D(\thet_{\step k})}}
	\end{equation}
	Taking expectation over $\Thet_{\step k}$, we have
	\begin{align*}
		\partial_t \mU_t(\cdot) &=  \int  \mU_{\step k}(\thet_{\step k}) \left\{ \noise^2 \lapL{\mU_{t|\step k}}(\cdot|\thet_{\step k}) + \divR{\mU_{t|\step k}(\cdot | \thet_{\step k})\graD{\Loss_\D(\thet_{\step k})}} \right\} \dif{\thet_{\step k}} \\
				  &= \noise^2 \lapL{\mU_t}(\cdot) + \divR{\mU_t(\cdot) \graD{\Loss_\D(\cdot)}} + \divR{ \mU_t(\cdot) \int \left[ \graD{\Loss_\D(\thet_{\step k})} - \graD{\Loss_\D(\cdot)} \right] \mU_{\step k| t}(\thet_{\step k} | \cdot) \dif{\thet_{\step k}}  } \\
				  updreq&= \noise^2 \divR{\mU_t(\cdot) \graD{\log \frac{\mU_t(\cdot)}{\pI({\D})(\cdot)}} } +  \divR{ \expec{}{\graD{\Loss_\D(\Thet_{\step k})} - \graD{\Loss_\D(\cdot)}| \Thet_t = \cdot} \mU_t(\cdot)},
	\end{align*}
	where $\mU_{\step k| t}$ is the conditional density of $\Thet_{\step k}$ given $\Thet_t$. 
	For the last equality, we have used the fact that $\graD{\Loss_{\D}} = - \noise^2 \graD{\log \pI({\D})}$ from \eqref{eqn:gibbs1}.
\end{proof}

%
% By design, the output distribution $\pI_{\D}$ of our learning algorithm is the unbiased Gibbs distribution than Noisy 
% SGD approximates~\cite{}. Several convergence analysis of LMCs in literature focus on measuring the divervence with respect to
% this ideal (but intractable) distribution $\pI_{\D}$ (denoted by $\pI$ for brevity in this section) as a function of number of 
% steps $\K$~\cite{}. We extend \citet{chewi2021analysis}'s elegant rapid convergence argument for LMCs to Noisy SGD.

The following lemma provides a partial differential inequality that bounds the rate of change in R\'enyi divergence $\Ren{\q}{\mU_t}{\pI(\D)}$ using Fokker-Planck equation \eqref{eqn:fokker_planck_noisysgd} of Noisy GD's tracing diffusion.
\begin{lemma}[{\citep[Proposition 15]{chewi2021analysis}}]
	\label{lem:convergence_pde}
	Let $\rhO_t \eqdef \mU_t/ \pI(\D)$ where $\pI(\D)$ is the Gibbs distribution defined in \eqref{eqn:gibbs1} and 
	$\psI_t \eqdef \rhO_t^{\q - 1} / \Eren{\q}{\rhO_t}{\pI(\D)}$. The rate of change
	in $\Ren{\q}{\mU_t}{\pI(\D)}$ along racing diffusion in time $\step k \leq t \leq \step (k+1)$ is bounded as
	\begin{equation}
		\label{eqn:pdi_without_lsi}
		\partial_t \Ren{\q}{\mU_t}{\pI(\D)} \leq -\frac{3 \q\noise^2}{4} \frac{\Gren{\q}{\mU_t}{\pI(\D)}}{\Eren{\q}{\mU_t}{\pI(\D)}} + \frac{\q}{\noise^2} \expec{}{ \psI_t(\Thet_t) \norm{\graD{\Loss_\D(\Thet_{\step k})} - \graD{\rLoss_{\D}(\Thet_t)}}^2}.
	\end{equation}
\end{lemma}

\begin{proof}
	For brevity, let $\Delta_t(\cdot) = \expec{}{\graD{\Loss_\D(\Thet_{\step k})} - \graD{\Loss_\D(\Thet_t)} \middle\vert \Thet_t = \cdot}$ in context of this proof. From Lebinz integral rule, we have
	\begin{align*}
		\partial_t \Ren{\q}{\mU_t}{\pI(\D)} &= \frac{\q}{(\q - 1) \Eren{\q}{\mU_t}{\pI(\D)}} \int \left(\frac{\mU_t}{\pI(\D)} \right)^{\q - 1} \partial_t \mU_t \dif{\thet} \\
						&= \frac{\q}{(\q - 1) \Eren{\q}{\mU_t}{\pI(\D)}} \int \rhO_t^{\q - 1} \divR{\left\{\Delta_t + \noise^2 \graD{\log \rhO_t} \right\}\mU_t} \dif{\thet} \tag{From \eqref{eqn:fokker_planck_noisysgd}}\\
						&= - \frac{\q}{(\q - 1) \Eren{\q}{\mU_t}{\pI(\D)}} \int \dotP{\graD{\left(\rhO_t^{\q - 1}\right)}}{\Delta_t + \noise^2\graD{\log \rhO_t}} \mU_t \dif{\thet} \\
						&= - \frac{\q}{\Eren{\q}{\mU_t}{\pI(\D)}} \int \rhO_t^{\q - 2} \dotP{\graD{\rhO_t}}{\Delta_t + \noise^2 \frac{\graD{\rhO_t}}{\rhO_t}} \mU_t \dif{\thet} \\
						&= - \frac{\q}{\Eren{\q}{\mU_t}{\pI(\D)}} \left\{ \noise^2 \Gren{\q}{\mU_t}{\pI(\D)}  + \frac{2}{\q} \underbrace{\expec{\mU_t}{\rhO_t^{\q/2 - 1} \dotP{\graD{\left(\rhO_t^{\q/2}\right)}}{\Delta_t}}}_{\overset{\text{def}}{=} F_1} \right\} \tag{From \eqref{eqn:renyi_info}} \\
						% &= - \q \noise^2 \frac{\Gren{\q}{\mU_t}{\pI}}{\Eren{\q}{\mU_t}{\pI}} - 2 \frac{\expec{\mU_t}{\rhO_t^{\q/2 - 1} \dotP{\graD{\left(\rhO_t^{\q/2}\right)}}{\Delta_t}}}{\Eren{\q}{\mU_t}{\pI}}
	\end{align*}
	Note that the expectation in $\Delta_t(\cdot)$ is over the conditional distribution $\mU_{\step k | t}$ while the expectation in $F_1$ is over $\mU_t$. Therefore, we can combine the two to get an expectation over the unconditional joint distribution over $\Thet_t$ and $ \Thet_{\step k}$ as follows.
	\begin{align*}
		- F_1 &= \expec{\Thet_t \sim \mU_t}{\rhO_t^{\q/2 - 1}(\Thet_t) \dotP{\graD{\left(\rhO_t^{\q/2}\right)(\Thet_t)}}{\expec{\Thet_{\step k} \sim \mU_{\step k|t}}{\graD{\Loss_{\D}(\Thet_t)} - \graD{\Loss_\D(\Thet_{\step k})}}}} \\
		&= \expec{\mU_{\step k, t}}{\rhO_t^{\q/2 - 1}(\Thet_t) \dotP{\graD{\left(\rhO_t^{\q/2}\right)(\Thet_t)}}{\graD{\Loss_\D(\Thet_t)} - \graD{\Loss_\D(\Thet_{\step k})}}} \\
		&\leq \frac{\noise^2}{2\q} \expec{}{\rhO_t^{-1}(\Thet_t)\norm{\graD{\left(\rhO_t^{\q/2}\right)(\Thet_t)}}^2} 
		+ \frac{\q}{2\noise^2} \expec{}{\rhO_t^{\q - 1}(\Thet_t) \norm{\graD{\rLoss_{\D}(\Thet_t)} - \graD{\rLoss_{B_k}(\Thet_{\step k})}}^2} \\
		&= \frac{\q\noise^2}{8} \Gren{\q}{\rhO_t}{\mU} + \frac{\q}{2\noise^2} \expec{}{\rhO_t^{\q - 1}(\Thet_t) \norm{\graD{\rLoss_{\D}(\Thet_t)} - \graD{\rLoss_{B_k}(\Thet_{\step k})}}^2} \tag{From \eqref{eqn:renyi_info}} \\
	\end{align*}
	Substituting it in the preceding inequality proves the proposition.
\end{proof}
We need to solve the PDI \eqref{eqn:pdi_without_lsi} to get a convergence bound for Noisy-GD. To help in that, we first introduce the change of measure inequalities shown in \citet{chewi2021analysis}.

% The following lemma shows that under the assumption of boundedness of loss function $\loss(\thet;\x)$, 
% the Gibbs distribution $\pI(\D)$ satisfies $\LS$ inequality.
% %
% \begin{corollary}
% \label{cor:lsi_of_rloss}
% If $\loss(\thet;\x)$ is $(\noise^2\log(\B)/ 4)$-bounded for some constant $\B > 1$, then the output distribution of
% the learning algorithm $\Lrn(\cdot)$ on any database $\D \in \X^\n$ that samples from the Gibbs
% distribution $\pI_\D(\thet)$ defined in \ref{eqn:lrn_alg} satisfies $\LS(\cvx/\B)$.
% \end{corollary}
% %
% \begin{proof}
% 	Let $\pI_\D(\thet) = \frac{1}{Z_\pI} \exp\left(-\rLoss_\D(\thet)/\noise^2\right)$ be the density of Gibbs distribution on 
% 	regularized loss $\rLoss_\D$ defined in \eqref{eqn:reg_loss}, and $\pI'(\thet) = \frac{1}{Z_{\pI'}} \exp\left(-\frac{\cvx 
% 	\norm{\thet}^2}{2\noise^2}\right)$ be the a Gaussian measure. From Lemma~\ref{lem:lsi_gaussian}, $\pI'$ satisfies $\LS(\cvx)$.
% 	Note that the density ratio $\frac{\pI(\thet)}{\pI'(\thet)} = \frac{Z_{\pI'}}{Z_\pI} \exp\left( - \frac{\sum_{\x \in \D} 
% 	\loss(\thet;\x)}{\noise^2}\right)$ is bounded as
% 	%
% 	\begin{equation}
% 	\frac{1}{\sqrt{\B}} \leq \frac{\pI(\thet)}{\pI'(\thet)} \leq \sqrt{\B}.
% 	\end{equation}
% 	%
% 	Therefore, from Lemma~\ref{lem:lsi_perturbation}, $\pI$ satisfies $\LS(\cvx/\B)$.
% \end{proof}
% %
% Since the mixing distribution $\pI$ satisfy Logarithmic Sobolev inequality, we can use the inequalities in
% lemma~\ref{lem:kl_fis_lsi} and lemma~\ref{lem:ren_info_lsi} for solving the PDI.

%
\begin{lemma}[Change of measure inequality~\citep{chewi2021analysis}]
	\label{lem:cng_of_measure}
	If $\loss(\thet;\x)$ is $\smh$-smooth, and regularizer is $\reg(\thet) = \frac{\cvx}{2}\norm{\thet}^2$, then for any probability density $\mU$ on $\domain$, 
	\begin{equation}
		\label{eqn:cng_of_measure}
		\expec{\mU}{\norm{\graD{\rLoss_\D}}^2} \leq 4\noise^4 \expec{\pI(\D)}{\norm{\graD{\sqrt{\frac{\mU}{\pI(\D)}}}}^2} + 2 \dime\noise^2(\smh + \cvx),
	\end{equation}
	where $\pI(\D)$ is the Gibbs distribution defined in~\eqref{eqn:gibbs1}.
\end{lemma}
\begin{proof}
Consider the Langevin diffusion \eqref{eqn:langevin_sde} described in Appendix~\ref{sec:langevin_diffusion} over the potential $\Loss_\D$. The Gibbs distribution $\pI(\D)$ is its stationary distribution, and the diffusion's infinitesimal generator $\Gen$ applied on the $\rLoss_\D$ gives
\begin{equation}
\label{eqn:generator_on_Loss}
\Gen \rLoss_\D = \noise^2 \lapL{\rLoss_\D} - \norm{\graD{\rLoss_\D}}^2.
\end{equation}
Therefore,
\begin{align*}
	\expec{\mU}{\norm{\graD{\rLoss_\D}}^2} &= \noise^2 \expec{\mU}{\lapL{\rLoss_\D}} - \expec{\mU}{\Gen\rLoss_\D} \tag{From \eqref{eqn:generator_on_Loss}} \\
					   &\leq \dime\noise^2(\smh + \cvx) - \int \Gen\rLoss_\D \left( \frac{\mU}{\pI(\D)} - 1\right) \pI(\D) \dif{\thet} \tag{From $\smh$-smoothness and \eqref{eqn:generator_zero_expectation}} \\
					   &= \dime\smh\noise^2(\smh + \cvx) + \int \left[ \norm{\graD{\rLoss_\D}}^2 - \noise^2\lapL{\rLoss_\D} \right] \left(\frac{\mU}{\pI(\D)} - 1\right)\pI(\D) \dif{\thet} \\
					   &= \dime\smh\noise^2(\smh + \cvx) + \int \norm{\graD{\rLoss_\D}}^2 (\mU - \pI(\D)) \dif{\thet} \\
&\quad+ \noise^2 \int \dotP{\graD{\rLoss_\D}}{\graD{\left[\left(\frac{\mU}{\pI(\D)} - 1\right)\pI(\D)\right]}} \dif{\thet} \tag{From \eqref{eqn:lapl_dotp_eq}} \\
					   &= \dime\smh\noise^2(\smh + \cvx) + \int \norm{\graD{\rLoss_\D}}^2 (\mU - \pI(\D)) \dif{\thet} + \noise^2 \int \dotP{\graD{\rLoss_\D}}{- \frac{\graD{\rLoss_\D}}{\noise^2}}(\mU - \pI(\D)) \dif{\thet} \\
&\quad+ \noise^2 \int \dotP{\graD{\rLoss_\D}}{\graD{\frac{\mU}{\pI(\D)}}} \pI(\D) \dif{\thet} \tag{Since $\graD{\pI(\D)} = - \frac{\graD{\rLoss_\D}}{\noise^2}\pI(\D) $}  \\
					   &= \dime\smh\noise^2(\smh + \cvx) + 0+ 2\noise^2 \int \dotP{\sqrt{\frac{\mU}{\pI(\D)}} \graD{\rLoss_\D}}{\graD{\sqrt{\frac{\mU}{\pI(\D)}}}} \pI(\D) \dif{\thet} \\
					   &\leq \dime\smh\noise^2(\smh + \cvx) + \frac{1}{2}\expec{\mU}{\norm{\graD{\rLoss_\D}}^2} + 2\noise^4 \expec{\pI(\D)}{\norm{\graD{\sqrt{\frac{\mU}{\pI(\D)}}}}^2} \tag{From \eqref{eqn:young_ineq} with $a = 2\noise^2$} 
\end{align*}
\end{proof}
Another change in measure inequality needed for the proof is the Donsker-Varadhan variational principle.
\begin{lemma}[Donsker-Varadhan Variational principle~\citep{donsker1983asymptotic}]
	\label{lem:donsker_vardhan}
	If $\nU$ and $\nU'$ are two distributions on $\domain$ such that $\nU \ll \nU'$, then for all functions
	$f: \domain \rightarrow \R$, 
	\begin{equation}
		\label{eqn:donsker_vardhan}
		\expec{\Thet \sim \nU}{f(\Thet)} \leq \KL{\nU}{\nU'} + \log \expec{\Thet' \sim \nU'}{\exp(f(\Thet'))}.
	\end{equation}
\end{lemma}
We are now ready to prove the rate of convergence guarantee for Noisy-GD following \citet{chewi2021analysis}'s method, but with a more refined analysis that leads to a improvement of $\log\q$ factor in the discretization error (compared to the original~\citep[Theorem 4]{chewi2021analysis}). 
\begin{theorem}[Convergence of Noisy-GD in R\'enyi divergence]
	\label{thm:convergence}
	Let constants $\smh, \cvx, \noise^2 > 0$ and $\q, \B > 1$. Suppose the loss function $\loss(\thet;\x)$ is $(\noise^2\log(\B)/4)$-bounded and $\smh$-smooth, and regularizer is $\reg(\thet) = \frac{\cvx}{2}\norm{\thet}^2$. If step size is ${\step \leq \frac{\cvx}{64\B\q^2(\smh + \cvx)^2}}$, then for any database $\D \in \X^\n$ and any weight initialization distribution $\mU_0$ for $\Thet_0$, the R\'enyi divergence of distribution $\mU_{\step \K}$ of output model $\Thet_{\step \K} = \Nsgd(\D, \Thet_0, \K)$ with respect to the Gibbs distribution $\pI(\D)$ defined in \eqref{eqn:gibbs1} shrinks as follows:
	\begin{equation}
		\Ren{\q}{\mU_{\step \K}}{\pI(\D)} \leq \q\exp\left( - \frac{\cvx\step\K}{2\B} \right) \Ren{\q}{\mU_0}{\pI(\D)} + \frac{32\dime\step\q\B(\smh + \cvx)^2}{\cvx}.
	\end{equation}
\end{theorem}
\begin{proof}
	From $(\smh + \cvx)$-smoothness of loss $\Loss_\D$ we have that for any $\step k \leq t \leq \step (k + 1)$,
\begin{align*}
	\norm{\graD{\Loss_\D(\Thet_{\step k})} - \graD{\Loss_{\D}(\Thet_t)}}^2 &\leq (\smh+\cvx)^2 \norm{\Thet_{\step k} - \Thet_t}^2\\
 											   &= (\smh+\cvx)^2 \norm{(t-\step k) \graD{\Loss_\D(\Thet_{\step k})} - \sqrt{2(t - \step k)\noise^2}\Z_k}^2 \tag{From \eqref{eqn:tracing_SDE}} \\
											   &\leq 2\step^2(\smh+\cvx)^2\norm{\graD{\Loss_\D(\Thet_{\step k})}}^2 + 4\step\noise^2(\smh+\cvx)^2\norm{\Z_k}^2 \\
											   &\leq 4\step^2(\smh+\cvx)^2\norm{\graD{\Loss_\D(\Thet_{\step k})} - \graD{\Loss_\D(\Thet_t)}}^2 \\
											   &\quad + 4\step^2(\smh+\cvx)^2\norm{\graD{\Loss_\D(\Thet_t)}}^2 + 4\step\noise^2(\smh+\cvx)^2\norm{\Z_k}^2 
\end{align*}
Let $\rhO_t \eqdef \frac{\mU_t}{\pI(\D)}$ and $\psI_t \eqdef \rhO_t^{\q - 1} / \Eren{\q}{\rhO_t}{\pI(\D)}$. 
If $\step \leq \frac{1}{2\sqrt{2}(\smh+\cvx)}$, we rearrange to get the following and use it to get the following bound on the discretization error in~\eqref{eqn:pdi_without_lsi}:
\begin{align*}
	% \norm{\graD{\rLoss_{\batch_k}(\Thet_{\step k})} - \graD{\rLoss_{\D}(\Thet_t)}}^2 
	{\expec{}{ \psI_t(\Thet_t) \norm{\graD{\Loss_{\batch_k}(\Thet_{\step k})} - \graD{\Loss_{\D}(\Thet_t)}}^2}} &\leq 8\step^2(\smh + \cvx)^2 \underbrace{\expec{}{\psI_t(\Thet_t)\norm{\graD{\Loss_{\D}(\Thet_t)}}^2}}_{\overset{\text{def}}{=} F_1} \\
		 &\quad+ 32\step \noise^2(\smh + \cvx)^2 \underbrace{\expec{}{\psI_t(\Thet_t)\norm{\Z_k}^2/4}}_{\overset{\text{def}}{=} F_2}.
\end{align*}

Hence, for solving the PDI~\eqref{eqn:pdi_without_lsi}, we have to bound the three expectations $F_1$ and $F_2$. 
\begin{enumerate}
	 	\item {\bf Bounding $F_1$.} Note that ${\expec{\Thet_t \sim \mU_t}{\psI_t(\Thet_t)} = \int \psI_t(\thet)\mU_t(\thet)\dif{\thet} = \frac{1}{\Eren{\q}{\rhO_t}{\pI(\D)}}\int \frac{\mU_t^\q}{\pI(\D)^{\q - 1}}\dif{\thet} = 1}$. So, $\psi_t\mu_t(\thet) \eqdef \psI_t(\thet) \mU_t(\thet)$ is a probability density function on $\domain$. On applying the measure change Lemma~\ref{lem:cng_of_measure} on it, we get
		\begin{align*}
			F_1 = \expec{\psI_t \mU_t}{\norm{\graD{\rLoss_{\D}}}^2} 
			    &\leq 4\noise^4 \expec{\pI(\D)}{\norm{\graD{\sqrt{\frac{\psI_t \mU_t}{\pI(\D)}}}}^2} + 2\dime\noise^2(\smh + \cvx) \tag{From \eqref{eqn:cng_of_measure}} \\
			    &= 4\noise^4 \expec{\pI(\D)}{\frac{\norm{\graD{\sqrt{\rhO_t^\q}}}^2}{\Eren{\q}{\mU_t}{\pI(\D)}}} + 2\dime\noise^2(\smh + \cvx)  \\
			    &= \noise^4\q^2\frac{\Gren{\q}{\mU_t}{\pI(\D)}}{\Eren{\q}{\mU_t}{\pI(\D)}} + 2\dime \noise^2 (\smh + \cvx). \tag{From \eqref{eqn:renyi_info}} \\
		\end{align*}

	\item {\bf Bounding $F_2$.} Since $\psI_t\mu_t$ is a valid density on $\domain$, the joint density $\psI_t\mU_{t,z}(\thet, z) \eqdef \psI_t(\thet) \mU_{t,z}(\thet, z)$ where $\mU_{t,z}$ is the joint density of $\Thet_t$ and $\Z_k$ is also a valid density. Note that the $F_2$ is an expectation on $\norm{\Z_k}^2$ taken over the joint density $\psI_t\mU_{t,z}$. We can perform a measure change operation using Donsker-Varadhan principle to get
		\begin{align*}
			F_2 = \expec{\psI_t\mU_{t,z}}{\norm{\Z_k}^2/4}
			    \leq \KL{\psI_t\mU_{t,z}}{\mU_{t,z}} + \log \expec{\mU_{z}}{\exp(\norm{\Z_k}^2/4)},
		\end{align*}
		where we simplified the second term using the fact that the marginal $\mU_z$ of $\mU_{t,z}$ is a standard normal Gaussian. The random variable $\norm{\Z_k}^2$ is distributed according to the Chi-squared distribution $\chi^2_\dime$ with $\dime$ degrees of freedom. Since the moment generating function of Chi-squared distribution is ${\mathrm{M}_{\chi^2_\dime}(t) = \expec{X \sim \chi^2_\dime}{\exp(t X)} = (1 - 2t)^{-\dime/2}}$ for $t < \frac{1}{2}$, we can simplify the second term in $F_2$ as
		\begin{equation}
			\log \expec{\mU_z}{\exp(\norm{\Z_k}^2/4)} = \log\mathrm{M}_{\chi^2_\dime}\left(\frac{1}{4}\right) = \frac{\dime \log 2}{2}.
		\end{equation}
		The KL divergence term can be simplified as follows.
		\begin{align*}
			\KL{\psI_t\mU_{t,z}}{\mU_{t,z}} &= \int \int \psI_t\mU_{t,z}(\thet_t, z) \log \psI_t(\thet_t) \dif{\thet_t} \dif{z} \\
							&= \int \psI_t\mU_t  \log \frac{\rhO_t^{\q-1}}{\Eren{\q}{\mU_t}{\pI(\D)}} \dif{\thet_t} \tag{On marginalization of $z$} \\
							&= \frac{\q-1}{\q} \int \mU_t \psI_t \log \left\{ \frac{\rhO_t^{\q}}{\Eren{\q}{\mU_t}{\pI(\D)}} - \log \Eren{\q}{\mU_t}{\pI(\D)}^{1 / (\q - 1)} \right\}\dif{\thet_t} \\
							&= \frac{\q-1}{\q}\left\{ \KL{\mU_t\psI_t}{\pI(\D)} - \Ren{\q}{\mU_t}{\pI(\D)} \right\} \\
							&\leq \KL{\mU_t\psI_t}{\pI(\D)} \tag{Since $\Ren{\q}{\mU_t}{\pI(\D)} > 0$}
		\end{align*}
		Note that under the assumptions of the Theorem, $\pI(\D)$ satisfies log-Sobolev inequality~\eqref{eqn:def_lsi} with constant $\cvx/\B$ (i.e. satisfies $\LS(\cvx/\B)$). To see this, recall from Lemma~\ref{lem:lsi_gaussian} that the Gaussian distribution ${\rhO(\thet) = \Gaus{0}{\frac{\noise^2}{\cvx}\Id}}$ satisfies $\LS(\cvx)$ inequality. Since loss $\loss(\thet;\x)$ is $(\noise^2\log(\B)/4)$-bounded, the density ratio ${\frac{\pI(D)(\thet)}{\rhO(\thet)} \in \left[\frac{1}{\sqrt{\B}}, \sqrt{\B}\right]}$. The claim therefore follows from Lemma~\ref{lem:lsi_perturbation}. Using this inequality, from Lemma~\ref{lem:kl_fis_lsi} we have
		\begin{align*}
			\KL{\mU_t \psI_t}{\pI(\D)} &\leq \frac{\noise^2\B}{2\cvx} \int \mU_t \psI_t \norm{\graD{\log \left(\frac{\mU_t \psI_t}{\pI(\D)}\right)}}^2 \dif{\thet_t} \\
					       &= \frac{\noise^2\B}{2\cvx} \int \frac{\rhO_t^{\q}}{\Eren{\q}{\mU_t}{\pI(\D)}} \norm{\graD{ \log (\rhO_t^\q)}}^2 \pI(\D) \dif{\thet_t} \\
					       &= \frac{2\noise^2\B}{\cvx} \frac{1}{\Eren{\q}{\mU_t}{\pI(\D)}} \int \norm{\graD{(\rhO_t^{\q/2})}}^2\pI(\D) \dif{\thet_t} \\
					       &= \frac{\q^2\noise^2\B}{2\cvx} \frac{\Gren{\q}{\mU_t}{\pI(\D)}}{\Eren{\q}{\mU_t}{\pI(\D)}}
		\end{align*}
\end{enumerate}

		On combining all the two bounds on $F_1$ and $F_2$ and rearranging, we have
		\begin{align*}
			\expec{}{ \psI_t(\Thet_t) \norm{\graD{\Loss_\D(\Thet_{\step k})} - \graD{\Loss_\D(\Thet_t)}}^2} 
		&\leq 8 \step \q^2 \noise^4 (\smh + \cvx)^2 \frac{\Gren{\q}{\mU_t}{\pI(\D)}}{\Eren{\q}{\mU_t}{\pI(\D)}} \left(\step + \frac{2\B}{\cvx} \right) \\
		&\quad+ 16 \step \dime \noise^2 (\smh + \cvx)^2 \left(\step(\smh + \cvx) + \log 2\right) 
		\end{align*}
		Let step size be $\step \leq \min \left\{ \frac{2\B}{\cvx}, \frac{\cvx}{64\B \q^2 (\smh + \cvx)^2} \right\}$. 
		Then, the first term above is bounded as
		\begin{equation}
			8 \step \q^2 \noise^4 (\smh + \cvx)^2 \frac{\Gren{\q}{\mU_t}{\pI(\D)}}{\Eren{\q}{\mU_t}{\pI(\D)}} \left(\step + \frac{2\B}{\cvx} \right) \leq \frac{\noise^4}{2} \frac{\Gren{\q}{\mU_t}{\pI(\D)}}{\Eren{\q}{\mU_t}{\pI(\D)}}.
		\end{equation}
		Let $\step \leq \frac{1}{4(\smh + \cvx)}$.
		Then, in the third term,  $\left(\step(\smh+\cvx) + \log 2\right) \leq 1$.
		Plugging the bound on discretization error back in the PDI~\eqref{eqn:pdi_without_lsi}, we get
		\begin{equation}
			\partial_t \Ren{\q}{\mU_t}{\pI(\D)} \leq -\frac{\q\noise^2}{4} \frac{\Gren{\q}{\mU_t}{\pI(\D)}}{\Eren{\q}{\mU_t}{\pI(\D)}} + 16 \step \dime \q (\smh+\cvx)^2.
		\end{equation}
		Since $\pI(\D)$ satisfies $\LS(\cvx/\B)$ inequality, from Lemma~\ref{lem:ren_info_lsi} this PDI reduces to
		\begin{equation}
			\partial_t \Ren{\q}{\mU_t}{\pI(\D)} +  \frac{\cvx}{2\B} \left( \frac{\Ren{\q}{\mU_t}{\pI(\D)}}{\q} + (\q - 1)\partial_\q \Ren{\q}{\mU_t}{\pI(\D)} \right)   \leq 16\dime \step \q (\smh+\cvx)^2.
		\end{equation}
		Let $c_1 = \frac{\cvx}{2\B}$ and $c_2 = 16\dime\step(\smh+\cvx)^2$. Additionally,
		let $u(\q, t) = \frac{\Ren{\q}{\mU_t}{\pI(\D)}}{\q}$. Then, 
		\begin{align*}
			&\partial_t \Ren{\q}{\mU_t}{\pI(\D)} + c_1 \left(\frac{\Ren{\q}{\mU_t}{\pI(\D)}}{\q} + (\q - 1) \partial_\q \Ren{\q}{\mU_t}{\pI(\D)}  \right) 
			\leq c_2\q \\
			\implies& \frac{\partial_t \Ren{\q}{\mU_t}{\pI(\D)}}{\q} + c_1 \frac{\Ren{\q}{\mU_t}{\pI(\D)}}{\q} + c_1 (\q - 1) \left( \frac{\partial_\q \Ren{\q}{\mU_t}{\pI(\D)}}{\q} - \frac{\Ren{\q}{\mU_t}{\pI(\D)}}{\q^2}  \right) \leq c_2 \\
			\implies& \partial_t u(\q, t) + c_1 u(\q, t) + c_1(\q - 1) \partial_\q u(\q, t) \leq c_2.
		\end{align*}
 		For some constant $\bar\q \geq 1$, let $\q(s) = (\bar\q - 1) \exp({c_1(s - \step \K)}) + 1$, and $t(s) = s$. Note that $\der{\q(s)}{s} = c_1(\q(s) - 1)$ and $\der{t(s)}{s} = 1$. Therefore, for any $0 \leq t \leq \step \K$, the PDI above implies the following differential inequality is followed along the path $u(s) = u(\q(s), t(s))$.
 		\begin{align*}
 			\der{u(s)}{s} + c_1 u(s) \leq c_2 \implies& \der{}{s} \{ e^{c_1 s} u(s)  \} \leq c_2 e^{c_1 s} \\
 			\implies& [e^{c_1 s} u(s)]_{0}^{\step \K}  \leq \int_0^{\step \K} c_2 e^{c_1 s} \dif{s} \\
 			\implies& e^{c_1 \step \K} u(\step \K) - u(0) \leq \frac{c_2(e^{c_1 \step \K} - 1)}{c_1} \\
 			\implies& u(\step \K) \leq e^{-c_1 \step \K} u(0) + \frac{c_2}{c_1}(1 - e^{-c_1 \step \K}).
 		\end{align*}
 		On reversing the parameterization of $\q$ and $t$, we get 
		\begin{align*}
			\Ren{\q(\step\K)}{\mU_{\step \K}}{\pI(\D)} &\leq \frac{\q(\step \K)}{\q(0)} e^{-c_1 \step \K} \Ren{\q(0)}{\mU_{0}}{\pI(\D)} + \frac{c_2}{c_1}\q(\step\K)  \\
							  &\leq \frac{\q(\step\K)}{\q(0)} \exp\left({-\frac{\cvx\step\K}{2\B}}\right) \Ren{\q(0)}{\mU_{0}}{\pI(\D)} +  \frac{32 \dime \step \B(\smh+\cvx)^2}{\cvx} \q(\step \K).
		\end{align*}
		Since $\q(0) > 1$ and $\bar\q = \q(\step \K) > \q(0)$, from monotonicity of R\'enyi divergence in $\q$, we get
		\begin{equation}
			\Ren{\bar\q}{\mU_{\step \K}}{\pI(\D)} \leq \bar\q\exp\left( - \frac{\cvx\step\K}{2\B} \right) \Ren{\bar\q}{\mU_0}{\pI(\D)} + \frac{32\dime\step\bar\q\B(\smh + \cvx)^2}{\cvx}.
		\end{equation}
		Finally, noting that for constants $\B, \q > 1$ and $\smh, \cvx > 0$,
		\begin{equation}
			\step \leq \min \{\frac{1}{2\sqrt{2}(\smh + \cvx)}, \frac{1}{4(\smh + \cvx)}, \frac{2\B}{\cvx}, \frac{\cvx}{64\B\q^2(\smh + \cvx)^2} \} = \frac{\cvx}{64\B\q^2(\smh + \cvx)^2},
		\end{equation}
		completes the proof.
		% %
		% \begin{align*}
		% 	& e^{-c_1 t} \partial_t \{ e^{c_1 t} \Ren{\q}{\mU_t}{\pI}\} = \partial_t \Ren{\q}{\mU_t}{\pI} + c_1 \Ren{\q}{\mU_t}{\pI} \leq  c_2 \\
		% 	\implies& \int_0^{\step \K} e^{c_1 t} \Ren{\q}{\mU_t}{\pI} \dif{t} \leq \int_0^{\step \K} c_2 e^{c_1 t} \dif{t} \\
		% 	\implies& e^{c_1 \step \K} \Ren{\q}{\mU_{\step \K}}{\pI} - \Ren{\q}{\mU_{0}}{\pI} \leq c_2 \frac{(e^{c_1 t} - 1)}{c_1} \\
		% 	\implies& \Ren{\q}{\mU_{\step \K}}{\pI} \leq e^{- c_1 \step\K} \Ren{\q}{\mU_0}{\pI} + \frac{c_2}{c_1} \\
		% 	\implies& \Ren{\q}{\mU_{\step \K}}{\pI} \leq \exp\left( -\frac{\cvx\step\K}{2\q\B} \right) \Ren{\q}{\mU_0}{\pI} +  \frac{32 \B \q^2}{\cvx} \left(\frac{\lip^2}{\noise^2} + 2\dime\step(\smh+\cvx)^2\right).
		% \end{align*}
\end{proof}

We will use Theorem~\ref{thm:convergence} for proving the data-deletion and utility guarantee on the pair $(\Lrn_\Nsgd, \Unlrn_\Nsgd)$. We need the following result that shows that Gibbs distributions enjoy strong indistinguishability on bounded perturbations to its potential function (which is basically why the exponential mechanism satisfies $(\eps, 0)$-DP~\citep{wang2015privacy,dwork2014algorithmic}).
\begin{lemma}[Indistinguishability under bounded perturbations]
	\label{lem:bounded_perturbation_renyi}
	For two potential functions $\Loss, \Loss': \domain \rightarrow \R$ and some constant $\noise^2$, let $\nU \propto e^{-\Loss/\noise^2}$ and $\nU' \propto e^{-\Loss'/\noise^2}$ be the respective Gibbs distributions. If $|\Loss(\thet) - \Loss'(\thet)| \leq c$ for all $\thet \in \domain$, then $\Ren{\q}{\nU}{\nU'} \leq \frac{2c}{\noise^2}$ for all $\q > 1$.
\end{lemma}
\begin{proof}
	The Gibbs distributions $\nU, \nU'$ have a density
	\begin{align*}
		\nU(\thet) = \frac{1}{\Lambda} e^{-\Loss(\thet)/\noise^2}, \quad \text{and} \quad \nU'(\thet) = \frac{1}{\Lambda'} e^{-\Loss'(\thet)/\noise^2},
	\end{align*}
	where $\Lambda, \Lambda'$ are the respective normalization constants. If for all $\thet \in \domain$, the potential difference $\vert \Loss(\thet) - \Loss'(\thet) \vert \leq c$, then 
	\begin{align*}
		\Ren{\q}{\nU}{\nU'} &= \frac{1}{\q - 1} \log\int \frac{\nU^\q}{\nU'^{\q - 1}} \dif{\thet}  \\
				    &= \frac{1}{\q - 1} \log \int \left(\frac{\Lambda'}{\Lambda}\right)^{\q-1} \exp\left(\frac{\q-1}{\noise^2}(\Loss'(\thet) - \Loss(\thet))\right) \times \nU(\thet) \dif{\thet} \\
				    &\leq \frac{1}{\q - 1} \left\{ (\q-1) \log \frac{\Lambda'}{\Lambda} + \log \exp\left(\frac{c(\q - 1)}{\noise^2}\int \nU \dif{\thet} \right) \right\} \\
				    &= \frac{1}{\q - 1} \left\{ (\q-1) \log \frac{\int \exp \left(-\frac{\Loss(\thet)}{\noise^2} + \frac{\Loss(\thet) - \Loss'(\thet)}{\noise^2} \right) \dif{\thet}}{\int \exp \left(-\frac{\Loss(\thet)}{\noise^2} \right) \dif{\thet}} + \frac{c(\q-1)}{\noise^2}\right\} \\
				    &\leq \frac{2c}{\noise^2}.
	\end{align*}
\end{proof}

In Theorem~\ref{thm:deletion_accuracy_nonconvex}, we show that $(\Lrn_\Nsgd, \Unlrn_\Nsgd)$ solves the data-deletion problem described in Section~\ref{sec:deletion} even for non-convex losses. Our proof uses the convergence Theorem~\ref{thm:convergence} and indistinguishability for bounded perturbation Lemma~\ref{lem:bounded_perturbation_renyi} to show that the data-deletion algorithm $\Unlrn_\Nsgd$ can consistently produce models indistinguishable to the corresponding Gibbs distribution~\eqref{eqn:gibbs1} in the online setting at a fraction of computation cost of $\Lrn_\Nsgd$. As discussed in Remark~\ref{rem:non_adap_independence}, such an indistinguishability is sufficient for ensuring data-deletion for non-adaptive requests. As for adaptive requests, the well-known R\'enyi DP guarantee of~\citet{abadi2016deep} combined with our reduction Theorem~\ref{thm:reduction} offers a data-deletion guarantee for $(\Lrn_\Nsgd, \Unlrn_\Nsgd)$ under adaptivity.

Our proof of accuracy for the data-deleted models leverages the fact that Gibbs distribution~\eqref{eqn:gibbs1} is an almost excess risk minimizer as shown in the following Theorem~\ref{thm:gibbs_optimal}. Since our data-deletion guarantee is based on near-indistinguishability to \eqref{eqn:gibbs1}, this property also ensures near-optimal excess risk of data-deleted models.
\begin{theorem}[Near optimality of Gibbs sampling]
	\label{thm:gibbs_optimal}
	If the loss function $\loss(\thet;\x)$ is $\noise^2\log(\B)/4$-bounded and $\smh$-smooth, the regularizer is $\reg(\thet) = \frac{\cvx}{2}\norm{\thet}^2$, then the excess empirical risk for a model $\bar\Thet$ sampled from the Gibbs distribution $\pI(\D) \propto e^{-\Loss_\D/\noise^2}$ is
	\begin{equation}
		\err(\bar\Thet; \D) = \expec{}{\Loss_\D(\bar\Thet) - \Loss_\D(\thet^*_\D)} \leq \frac{\dime\noise^2}{2}\left(\log\frac{\smh+\cvx}{\cvx} + \sqrt{\B} \right).
	\end{equation}
\end{theorem}
\begin{proof}
	We simplify expected loss as
	\begin{equation}
		\expec{}{\Loss_\D(\bar\Thet)} = \int \Loss_\D \pI(\D)\dif{\thet} = \noise^2 (\entropy(\pI(\D)) - \log(\Lambda_\D)),
	\end{equation}
	where 
	\begin{equation}
		\entropy(\pI(\D)) = - \int \pI(\D) \log \pI(\D) \dif{\thet} = - \int \frac{e^{-\Loss_\D/\noise^2}}{\Lambda_\D} \log \frac{e^{-\Loss_\D/\noise^2}}{\Lambda_\D} \dif{\thet}
	\end{equation}
	is the differential entropy of $\pI(\D)$, and $\Lambda_\D = \int e^{-\Loss_\D/\noise^2} \dif{\thet}$ is the normalization constant. Since the potential function $\Loss_\D$ is $(\cvx + \smh)$-smooth, we have
	\begin{align*}
		-\noise^2\log(\Lambda_\D) &= - \noise^2\log \int e^{-\Loss_\D/\noise^2} \dif{\thet} \\
					  &= \Loss_\D(\thet^*_\D) - \noise^2\log \int e^{(\Loss_\D(\thet^*_\D) - \Loss_\D(\thet))/\noise^2} \dif{\thet} \\
				   & \leq \Loss_\D(\thet^*_\D) - \noise^2\log \int e^{-(\smh+\cvx) \norm{\thet - \thet^*_\D}^2/2\noise^2} \dif{\thet} \\
				   &= \Loss_\D(\thet^*_\D) - \frac{\dime\noise^2}{2} \log\left(\frac{2\pi\noise^2}{\cvx+\smh}\right).
	\end{align*}
	Since $\loss(\thet;\x)$ is $\noise^2\log(\B)/4$-bounded, note that for the Gaussian distribution $\rhO \sim \Gaus{0}{\frac{\noise^2}{\cvx}\Id}$, the density ratio lies in $\frac{\pI(\D)(\thet)}{\rhO(\thet)} \in \left[\frac{1}{\sqrt{\B}}, \sqrt{\B}\right]$ for all $\thet \in \domain$. We decompose entropy $\entropy(\pI(\D))$ into cross-entropy and KL divergence to get
	\begin{align*}
		\entropy(\pI(\D)) &= - \int \pI(\D) \log \rhO \dif{\thet} - \KL{\pI(\D)}{\rhO} \\
				  &\leq - \int \pI(\D) \log \left[\left(\frac{\cvx}{2\pi\noise^2}\right)^{\dime/2} e^{-\frac{\cvx\norm{\thet}^2}{2\noise^2}}\right] \dif{\thet} \tag{Since $\KL{\pI(\D)}{\rhO} \geq 0$} \\
				  &= \frac{\dime}{2}\log\frac{2\pi\noise^2}{\cvx} + \frac{\cvx }{2\noise^2} \int \norm{\thet}^2 \pI(\D)(\thet)\dif{\thet} \\
				  &\leq \frac{\dime}{2}\log\frac{2\pi\noise^2}{\cvx} + \frac{\cvx \sqrt{\B}}{2\noise^2} \int \norm{\thet}^2 \rhO(\thet)\dif{\thet} \tag{Since $\frac{\pI(\D)(\thet)}{\rhO(\thet)} \in \left[\frac{1}{\sqrt{\B}}, \sqrt{\B}\right]$} \\
				  &=\frac{\dime}{2}\log\frac{2\pi\noise^2}{\cvx} + \frac{\dime\sqrt{\B}}{2}.
	\end{align*}
	On combining the bounds, we get
	\begin{equation}
		\err(\bar\Thet; \D) = \expec{}{\Loss_\D(\bar\Thet) - \Loss_\D(\thet^*_\D)} \leq \frac{\dime\noise^2}{2}\left(\log\frac{\smh+\cvx}{\cvx} + \sqrt{\B} \right).
	\end{equation}
\end{proof}

\begin{reptheorem}{thm:deletion_accuracy_nonconvex}[Accuracy, privacy, deletion, and computation tradeoffs]
	Let constants $\cvx, \smh, \lip, \noise^2, \step > 0$, $\q, \B > 1$, and $0 < \epsdd \leq \epsdp < \dime$. Let the loss function $\loss(\thet;\x)$ be $\frac{\noise^2\log(\B)}{4}$-bounded, $\lip$-Lipschitz and $\smh$-smooth, the regularizer be $\reg(\thet) = \frac{\cvx}{2}\norm{\thet}^2$, and the weight initialization distribution be $\rhO = \Gaus{0}{\frac{\noise^2}{\cvx}\Id}$. Then, 
	\begingroup
	\renewcommand\labelenumi{\bf(\theenumi.)}
	\begin{enumerate}
		\item both $\Lrn_\Nsgd$ and $\Unlrn_\Nsgd$ are $(\q, \epsdp)$-R\'enyi DP for any $\step \geq 0$ and any $\K_\Lrn, \K_\Unlrn \geq 0$ if
			\begin{equation}
				\label{eqn:ref_noise}
				\noise^2 \geq \frac{\q\lip^2}{\epsdp\n^2} \cdot \step \max\{\K_\Lrn, \K_\Unlrn\},
			\end{equation}
		\item pair $(\Lrn_\Nsgd, \Unlrn_\Nsgd)$ satisfy $(\q, \epsdd)$-data-deletion under all non-adaptive $\reps$-requesters for any $\noise^2 > 0$, if learning rate is $\step \leq \frac{\cvx \epsdd}{64\dime\q\B(\smh + \cvx)^2}$ and number of iterations satisfy
			\begin{equation}
				\label{eqn:ref_iter}
				\K_\Lrn \geq \frac{2\B}{\cvx\step} \log \left(\frac{\q \log(\B)}{\epsdd}\right), \quad 
				\K_\Unlrn \geq \K_\Lrn - \frac{2\B}{\cvx\step} \log \left(\frac{\log(\B)}{2\left(\epsdd + \frac{\reps}{\n} \log(\B)\right)}\right),
			\end{equation}
		\item and all models in sequence $(\hat\Thet_i)_{i\geq0}$ output by $(\Lrn_\Nsgd, \Unlrn_\Nsgd, \updreq)$ on any $\D_0 \in \X^\n$, where $\updreq$ is an $\reps$-requester, satisfy $\err(\hat\Thet_i; \D_i) = \tilde O\left(\frac{\dime\q}{\epsdp\n^2} + \frac{1}{\n}\sqrt{\frac{\q\epsdd}{\epsdp}}\right)$ when inequalities in~\eqref{eqn:ref_iter} and~\eqref{eqn:ref_noise} are equalities.
	\end{enumerate}
	\endgroup
\end{reptheorem}
\begin{proof}

	{\bf (1.) Privacy.} By Theorem~\ref{thm:abadi}, Noisy-GD with $\K$ iterations on an $\lip$-Lipschitz loss function satisfies $(\q, \epsdp)$-R\'enyi DP for any initial weight distribution $\rhO$ and learning rate $\step \geq 0$ if $\noise^2 = \frac{\q\lip^2}{\epsdp\n^2}\cdot \step \K$. Since, both $\Lrn_\Nsgd$ and $\Unlrn_\Nsgd$ run Noisy-GD for $\K_\Lrn$ and $\K_\Unlrn$ iterations respectively, setting the noise variance given in the Theorem statement ensures $(\q, \epsdp)$-R\'enyi DP for both.

	{\bf (2.) Deletion.} For showing data-deletion under non-adaptive requests, recall that it is sufficient to show that there exists a map $\pI:\X^\n \rightarrow \C$ such that for all $i\geq 1$,
	\begin{equation}
		\label{eqn:toshow_deletion}
		\Ren{\q}{\Unlrn(\D_{i-1}, \up_i, \hat\Thet_{i-1})}{\pI(\D_i)} \leq \epsdd,
	\end{equation}
	for all edit sequences $(\up_i)_{i\geq1}$ from $\U^\reps$, where $(\hat\Thet_i)_{i \geq 0}$ is the sequence of models generated by the interaction of $(\Lrn_\Nsgd, \Unlrn_\Nsgd, \updreq)$ on any database $\D_0 \in \X^\n$. For all $i \geq 0$, let $\hat\mU_i$ denote the distribution of $\hat\Thet_i$. We prove~\eqref{eqn:toshow_deletion} via induction.

	\emph{Base step:}
	Note that the initial weight distribution $\rhO = \Gaus{0}{\frac{\noise^2}{\cvx}\Id}$ has a density proportional to ${e^{-\reg(\thet)/\noise^2}}$ and the distribution $\pI(\D_0)$ has a density proportional to ${e^{-\rLoss_{\D_0}(\thet)/\noise^2}}$. Since both of these are Gibbs distributions with their potential difference $|\Loss_{\D_0}(\thet) - \reg(\thet)| \leq \noise^2 \log(\B) / 4$ for all $\thet\in \domain$ due to boundedness assumption on $\loss(\thet;\x)$, we have from Lemma~\ref{lem:bounded_perturbation_renyi} that
	\begin{equation}
		\Ren{\q}{\rhO}{\pI({\D_0})} \leq \frac{2}{\noise^2} \times \frac{\noise^2\log(\B)}{4} = \frac{\log(\B)}{2}.
	\end{equation}
	Under the stated assumptions on loss $\loss(\thet;\x)$ and learning rate $\step$, note that the convergence Theorem~\ref{thm:convergence} holds. Since ${\hat\Thet_0 = \Lrn_\Nsgd(\D_0) = \Nsgd(\D_0, \Thet_0, \K_\Lrn)}$, where $\Thet_0 \sim \rhO$, we have
	\begin{align*}
		\Ren{\q}{\hat\mU_0}{\pI(\D_0)} &\leq \q \exp\left(-\frac{\cvx\step\K_\Lrn}{2\B}\right) \Ren{\q}{\rhO}{\pI(\D_0)} + \frac{32\dime\step\q\B(\smh+\cvx)^2}{\cvx} \\
					       &\leq \q \exp\left(-\frac{\cvx\step\K_\Lrn}{2\B}\right) \left(\frac{\log(\B)}{2}\right) + \frac{\epsdd}{2} \tag{Since $\step \leq \frac{\cvx\epsdd}{64\dime\q\B(\smh+\cvx)^2}$} \\
					       &\leq \epsdd \tag{Since $\K_\Lrn \geq \frac{2\B}{\cvx\step} \log\left(\frac{\q\log(\B)}{\epsdd}\right)$}
	\end{align*}

	\emph{Induction step:}
	Suppose $\Ren{\q}{\hat\mU_{i-1}}{\pI(\D_{i-1})} \leq \epsdd$. Again, from boundedness of $\loss(\thet;\x)$, we have ${|\Loss_{\D_{i-1}}(\thet) - \Loss_{\D_i}(\thet)| \leq \frac{\reps \noise^2 \log \B}{2\n}}$ for all $\thet \in \domain$. Therefore, from Lemma~\ref{lem:bounded_perturbation_renyi} we have for all $\q > 1$ that
	\begin{equation}
		\Ren{\q}{\pI(\D_{i-1})}{\pI(\D_i)} \leq \frac{\reps \log(\B)}{\n}.
	\end{equation}
	So from the weak triangle inequality Theorem~\ref{thm:triangle_inequality} of R\'enyi divergence,
	\begin{equation}
		\label{eqn:deletion_triangle}
		\Ren{\q}{\hat\mU_{i-1}}{\pI(\D_i)} \leq \Ren{\q}{\hat\mU_{i-1}}{\pI(\D_{i-1})} + \Ren{\infty}{\pI(\D_{i-1})}{\pI(\D_i)} \leq \epsdd + \frac{\reps \log(\B)}{\n}.
	\end{equation}
	Note that $\K_\Unlrn \geq \K_\Lrn - \frac{2\B}{\cvx\step} \log \left(\frac{\log(\B)}{2\left(\epsdd + \frac{\reps}{\n} \log(\B)\right)}\right) \geq \frac{2\B}{\cvx\step} \log \left(\frac{2\q\left(\epsdd + \frac{\reps}{\n} \log(\B)\right)}{\epsdd}\right)$. Since $\hat\Thet_i = \Unlrn_\Nsgd(\D_{i-1}, \up_i, \hat\Thet_{i-1}) = \Nsgd(\D_i, \hat\Thet_{i-1}, \K_\Unlrn)$, convergence Theorem~\ref{thm:convergence} gives
	\begin{align*}
		\Ren{\q}{\hat\mU_i}{\pI(\D_i)} &\leq \q \exp\left(-\frac{\cvx\step\K_\Unlrn}{2\B}\right) \Ren{\q}{\hat\mU_{i-1}}{\pI(\D_i)} +  \frac{32\dime\step\q\B(\smh+\cvx)^2}{\cvx} \\
					       &\leq \q \exp\left(-\frac{\cvx\step\K_\Unlrn}{2\B}\right) \left(\epsdd + \frac{\reps\log(\B)}{\n} \right) + \frac{\epsdd}{2} \tag{From \eqref{eqn:deletion_triangle} and constraint $\step \leq \frac{\cvx\epsdd}{64\dime\q\B(\smh+\cvx)^2}$} \\
					       &\leq \epsdd. \tag{Since $\K_\Unlrn \geq \frac{2\B}{\cvx\step} \log \left(\frac{2\q\left(\epsdd + \frac{\reps}{\n} \log(\B)\right)}{\epsdd}\right)$}
	\end{align*}
	Hence, by induction, $\Ren{\q}{\hat\mU_i}{\pI(\D_i)} \leq \epsdd$ holds for all $i\geq0$. 

	{\bf (3.) Accuracy.}
	Let $\thet^*_{\D_i} = \underset{\thet \in \domain}{\arg\min}\ \Loss_{\D_i}(\thet)$, and $\bar\Thet_i \sim \pI(\D_i)$. We decompose the excess empirical risk of Noisy-GD as follows:
	\begin{equation}
		\label{eqn:risk_nonconvex}
		\err(\hat\Thet_i;\D_i) = \expec{}{\Loss_{\D_i}(\hat\Thet_i) - \Loss_{\D_i}(\bar\Thet_i)}  + \expec{}{\Loss_{\D_i}(\bar\Thet_i) - \Loss_{\D_i}(\thet^*_{\D_i})}.
	\end{equation}
	The second term is the suboptimality of Gibbs distribution and by Theorem~\ref{thm:gibbs_optimal}, it is bounded as
	\begin{equation}
		\expec{}{\Loss_{\D_i}(\bar\Thet_i) - \Loss_{\D_i}(\thet^*_{\D_i})} \leq \frac{\dime\noise^2}{2}\left(\log\frac{\smh+\cvx}{\cvx} + \sqrt{\B} \right).
	\end{equation}
	From $(\cvx + \smh)$-smoothness of $\Loss_{\D_i}$, for any coupling $\Pi$ of $\hat\Thet_i$ and $\bar\Thet_i$, the first term satisfies 
	\begin{align*}
		\expec{}{\Loss_{\D_i}(\hat\Thet_i) - \Loss_{\D_i}(\bar\Thet_i)} &\leq \expec{\Pi}{\dotP{\graD{\Loss_{\D_i}(\bar\Thet_i)}}{\hat\Thet_i - \bar\Thet_i} + \frac{\cvx + \smh}{2} \norm{\hat\Thet_i - \Thet_i}^2} \\
										&= \expec{\Pi}{\dotP{\sum_{\x \in \D_i} \graD{\loss(\bar\Thet_i; \x)} + \cvx \bar\Thet_i}{\hat\Thet_i - \bar\Thet_i} + \frac{\cvx + \smh}{2} \norm{\hat\Thet_i - \Thet_i}^2}. \tag{From $\lip$-Lipschitzness of $\loss(\thet;\x)$ and Jensen's inequality} \\
										&\leq \lip \sqrt{\expec{\Pi}{\norm{\hat\Thet_i - \bar\Thet_i}^2}} + \cvx\expec{\Pi}{\dotP{\bar\Thet_i}{\bar\Thet_i - \hat\Thet_i}} + \frac{\cvx+\smh}{2} \expec{\Pi}{\norm{\hat\Thet_i - \bar\Thet_i}^2} \tag{From Young's inequality~\eqref{eqn:young_ineq}} \\
										&\leq \lip \sqrt{\expec{\Pi}{\norm{\hat\Thet_i - \bar\Thet_i}^2}} + \frac{\cvx}{2}\expec{\bar\Thet_i \sim \pI(\D_i)}{\norm{\bar\Thet_i}^2} + \frac{2\cvx+\smh}{2} \expec{\Pi}{\norm{\hat\Thet_i - \bar\Thet_i}^2}. 
	\end{align*}
	Recall that the distribution $\pI(\D)$ satisfies $\LS(\cvx/\B)$ inequality. On choosing the coupling $\Pi$ to be the infimum, we get the following bound on Wasserstein's distance from Lemma~\ref{lem:ren_wasser_kl}.
	\begin{align}
		\inf_\Pi \sqrt{\expec{\hat\Thet_i, \bar\Thet_i \sim \Pi}{\norm{\hat\Thet_i - \bar\Thet_i}^2}} = \Was{\hat\Thet_i}{\bar\Thet_i} \leq \sqrt{\frac{2\B\noise^2}{\cvx} \KL{\mU_i}{\pI(\D_i)}} \leq \sqrt{\frac{2\epsdd\B\noise^2}{\cvx}}.
	\end{align}
	The last inequality above follows from monotonicity of R\'enyi divergence in $\q$ and the fact that ${\lim_{\q \rightarrow 1} \Ren{\q}{\nU}{\nU'} = \KL{\nU}{\nU'}}$. 

	Since $\pI(\D_i)$ is the Gibbs distribution with density proportional to $e^{-\Loss_{\D_i}/\noise^2}$, we have that
	\begin{equation}
		\expec{\bar\Thet_i \sim \pI(\D_i)}{\norm{\bar\Thet_i}^2} = \frac{1}{\Lambda_{\D_i}} \int \norm{\thet}^2 e^{-\Loss_{\D_i}(\thet)/\noise^2} \dif{\thet} \quad \text{where} \ \Lambda_{\D_i} = \int e^{-\Loss_{\D_i}(\thet)/\noise^2} \dif{\thet}.
	\end{equation}
	From $\frac{\noise^2\log\B}{4}$-boundedness of $\loss(\thet;\x)$, note that we have for every $\thet \in \domain$ that
	\begin{equation}
		\left|\Loss_{\D_i}(\thet) - \reg(\thet)\right| \leq \frac{\noise^2\log\B}{4}.
	\end{equation}
	Therefore, 
	\begin{equation}
		\Lambda_{\D_i} = \int e^{-\Loss_{\D_i}(\thet)/\noise^2} \dif{\thet} \geq \frac{1}{\sqrt[4]{\B}}\int e^{-\reg(\thet)/\noise^2} \dif{\thet},
	\end{equation}
	and hence,
	\begin{align*}
		\expec{\bar\Thet_i \sim \pI(\D_i)}{\norm{\bar\Thet_i}^2} &\leq \frac{\sqrt[4]{\B}}{\int e^{-\reg(\thet)/\noise^2} \dif{\thet}} \times \int \norm{\thet}^2 e^{-\Loss_{\D_i}(\thet)/\noise^2} \dif{\thet} \\
									 &\leq \sqrt{\B} \times \frac{\int \norm{\thet}^2 e^{-\reg(\thet)/\noise^2}}{\int e^{-\reg(\thet)/\noise^2}} \\
									 &= \sqrt{\B}\expec{\Z \sim \Gaus{0}{\frac{\noise^2}{\cvx}\Id}}{\norm{\Z}^2} \\
									 &= \frac{\sqrt{\B}\noise^2 \dime}{\cvx}.
	\end{align*}

	Therefore, on combining all the bounds we get
	\begin{equation}
		\err(\hat\Thet;\D) \leq  \lip\noise\sqrt{\frac{2\epsdd\B}{\cvx}} + \frac{\epsdd\B\noise^2(2\cvx+\smh)}{\cvx} + \frac{\dime\noise^2}{2}\left(\log\frac{\smh+\cvx}{\cvx} + 2\sqrt{\B} \right) = O\left(\noise\sqrt{\epsdd} + \dime\noise^2\right).
	\end{equation}
	Note that if the constraints on $\K_\Lrn$ and $\K_\Unlrn$ in~\eqref{eqn:ref_iter} and on $\noise^2$ in~\eqref{eqn:ref_noise} are equalities instead, we have
	\begin{equation}
		\noise^2 = \frac{2\q\B\lip^2}{\cvx\epsdp\n^2}\log\left(\frac{\q\log(\B)}{\epsdd}\right) = \tilde O\left(\frac{\q}{\epsdp\n^2}\right),
	\end{equation}
	where $\tilde O(\cdot)$ hides logarithmic factors. Therefore, the excess empirical risk has an order
	\begin{equation}
		\err(\hat\Thet;\D) = \tilde O\left(\frac{1}{\n}\sqrt{\frac{\q\epsdd}{\epsdp}} + \frac{\dime\q}{\epsdp\n^2}\right).
	\end{equation}
\end{proof}

\end{document}